\documentclass{article}
\RequirePackage{fix-cm}
\DeclareMathSizes{10}{10}{8}{5}
\usepackage[normalem]{ulem}
\usepackage{a4wide}
\usepackage{csquotes}
\usepackage{times}
\usepackage{authblk}
\usepackage{array}
\usepackage{amsthm}
\usepackage{amssymb}
\usepackage{amsfonts}
\usepackage{amsmath}
\usepackage{MnSymbol}
\usepackage{todonotes}
\usepackage{paralist}
\usepackage{thm-restate} 
\usepackage{multirow}
\usepackage{subcaption}
\presetkeys{todonotes}{inline,color=red!30}{} 
\usepackage{framed}
\usepackage{xfrac}  
\usepackage[nottoc,numbib]{tocbibind}
\usepackage{anyfontsize} 
\mathchardef\mhyphen="2D

\newcommand*{\approxident}{%
  \mathrel{\vcenter{\offinterlineskip
  \hbox{$\sim$}\vskip-.35ex\hbox{$\sim$}\vskip-.35ex\hbox{$\sim$}}}}

\usepackage{tikz}
\usetikzlibrary{shapes.geometric,shapes.misc, positioning}
\usetikzlibrary{decorations.pathmorphing}

\tikzset{snake it/.style={decorate, decoration=snake}}

\usepackage{hyperref}

\theoremstyle{plain} 
\newtheorem{theorem}{Theorem}[section]

\theoremstyle{definition}
\newtheorem{definition}[theorem]{Definition}
\declaretheorem[style=definition]{example}
\renewcommand\thmcontinues[1]{Continued}

\newcommand{\myarg}[1] {{\tt {#1}}}

\newcommand{\xa}{\myarg{A}}
\newcommand{\xb}{\myarg{B}}
\newcommand{\xc}{\myarg{C}}
\newcommand{\xd}{\myarg{D}}
\newcommand{\xe}{\myarg{E}}
\newcommand{\xf}{\myarg{F}}
\newcommand{\xg}{\myarg{G}}
\newcommand{\xh}{\myarg{H}}

\newcommand{\tvt}{{\bf t}}
\newcommand{\tvf}{{\bf f}}
\newcommand{\tvu}{{\bf u}}

\newcommand*\cm{$\checkmark$}
\newcommand*\xm{$\times$}

\newcommand{\graph}{{\cal G}}

\newcommand{\lab}{{\cal L}}
\newcommand{\con}{{\cal C}}
\newcommand{\acc}{{\cal AC}}

\newcommand{\formis}{:}
\newcommand{\formeq}{\approxident}
\newcommand{\nformeq}{\not\formeq}

\newcommand{\ineq}{\#}
\newcommand{\operators}{\{=, \neq, \geq, \leq,>,<\}}

\newcommand{\closure}{{\sf Closure}}
\newcommand{\terms}{{\sf Terms}}
\newcommand{\fterms}{{\sf FTerms}}
\newcommand{\fargs}{{\sf FArgs}}

\newcommand{\arcs}{{\sf Arcs}}
\newcommand{\nodes}{{\sf Nodes}}
\newcommand{\sat}{{\sf Sat}}
\newcommand{\dist}{{\sf Dist}}

\newcommand{\parent}{{\sf Parent}}

\newcommand{\eformulae}{{\sf EFormulae}}   
\newcommand{\oformulae}{{\sf OFormulae}}   
\newcommand{\argcomplete}{{\sf AComplete}}
\newcommand{\arop}{{\sf AOp}}

\newcommand{\args}{{\sf Args}}

\definecolor{darkgreen}{rgb}{0.09, 0.45, 0.27}

\newcommand{\prob}{P}

\newcommand{\pprob}{p}

\newcolumntype{P}[1]{>{\centering\arraybackslash}m{#1}}   

\providecommand{\keywords}[1]{\textbf{\textit{Keywords---}} #1}
\begin{document}

\title{Epistemic Graphs for Representing and Reasoning\\ with Positive and Negative Influences of Arguments}
\date{}
\author{Anthony Hunter}
\author{Sylwia Polberg\thanks{Corresponding author. E--mail: sylwia.polberg at gmail.com. University College London, Department of Computer Science, 66-72 Gower Street, London WC1E 6EA, United Kingdom. \\
© 2019. This manuscript version is made available under the CC-BY-NC-ND 4.0 license \url{http://creativecommons.org/licenses/by-nc-nd/4.0/}\\
The published version of this manuscript is available at \url{https://doi.org/10.1016/j.artint.2020.103236}.
} }

\affil{University College London, London, United Kingdom} 
\author{Matthias Thimm}
\affil{University of Koblenz-Landau, Koblenz, Germany}

\maketitle
 
\begin{abstract} 
This paper introduces epistemic graphs as a generalization of 
the epistemic approach to probabilistic argumentation. 
In these graphs, an argument can be believed or disbelieved up to a given degree,
thus providing a more fine--grained alternative to the standard Dung's approaches when it comes 
to determining the status of a given argument. Furthermore, the flexibility of the epistemic approach allows us to both model 
the rationale behind the existing semantics as well as completely deviate from them when required. 
Epistemic graphs can model both attack and support as well as
relations that are neither support nor attack.
The way other arguments influence a given argument is expressed by the epistemic constraints that 
can restrict the belief we have in an argument with a varying degree of specificity.   
The fact that we can specify the rules under which arguments should be evaluated and we can
include constraints between unrelated arguments permits the framework to be more context--sensitive. 
It also allows for better modelling 
of imperfect agents, which can be important in multi--agent applications.  
\end{abstract}

 
\newpage
\tableofcontents

\newpage
\section{Introduction}
\label{sec:introduction}
%
%
%
%
%

In real-world situations, argumentation is pervaded by uncertainty. 
In monological argumentation, we might be uncertain about how much we believe an argument and how 
much this belief should influence the belief in other arguments. 
These issues are compounded when considering dialogical argumentation, 
where each participant might be uncertain about what other agents believe. 
In addition, there are further notions important for
successful argumentation, such as the ability to take contextual information into account, to handle different
perspectives that 
various  agents can have about a given issue, or to model agents that are not perfectly
rational reasoners or about whom we do not possess complete information.

Our aim in this paper is to present a new formalism for argumentation that takes belief into account 
and tackles these challenges. 
Moreover, we want to have a formalism that will allow us to model how different agents reason with arguments. 
A key application we have in mind is an agent modelling another agent while participating in some form of discussion 
or debate. Hence, the modeller wants to understand the degree of belief in arguments by the other agent and reasons for it. 
He or she may then use the model to help choose their next move in the discussion or debate.

In order to make our investigation more focused, we assume that we can harness the following key assumptions 
for any scenario that we want to handle.

\begin{description}

\item[Argument graph] We assume we have a set of arguments, and some relationships between these arguments. 
We will treat the arguments as abstract in this paper, but we can instantiate each argument with a textual description 
(and we will have examples of such instantiations) or a logical specification 
(for example as a deductive argument \cite{BesnardHunter2014}). 
We also assume that the arguments and relationships between them can be represented by a directed graph, 
and so each node denotes an argument, and each arc denotes a relationship between a pair of arguments.

\item[Belief in arguments] Belief in an arguments can be conceptualized in a number of ways. 
In this paper, we focus on belief in an argument as a combination of the degree to which the premises and claims 
are believed to be true, and the degree to which the claim is believed to follow from those premises. 
Furthermore, we assume that belief in an argument can be modulated or influenced by other arguments. 
\end{description}

Additionally, we are interested in applications where we source arguments from the real-world 
(e.g. arguments that arise in dialogues or discussions). This means that we will have arguments  
that are enthymemes (i.e. arguments with some of the premises and/or claim being implicit). 
This in turn means that different people may have a different belief assignment to an argument 
because it is an enthymeme and therefore they can decode it in different ways.


\subsection*{\textbf{Requirements}}
 
In this paper, we will focus on the following requirements for our new formalism. We will briefly delineate them first, 
and then motivate them through examples and a discussion of how meeting the requirements is of use.  
Some of these requirements have been satisfied by existing proposals in the literature but some are entirely new 
(e.g.  modelling context-sensitivity, modelling different perspectives, modelling imperfect agents, and modelling incomplete graphs).

\begin{description}
\item[Modelling fine--grained acceptability] 
Typical semantics for argumentation frameworks focus on judging whether an argument should be accepted or rejected. 
However, in practical applications, there might be uncertainty as to the degree an argument 
is accepted or rejected. 
Various studies, including \cite{Rahwan2011,PolbergHunter17}, show that a two--valued perspective may 
be insufficient for modelling people's beliefs about arguments. 
Since the degree to which an argument is accepted (rejected) can be expressed by the degree to which 
the argument is believed (disbelieved) \cite{PolbergHunter17}, 
we can see this requirement as stating that we should have a many--valued scale for belief in arguments.
Recent interest in ranking-based semantics and the notion of argument strength, also points 
to the need for the fine-grained requirement 
(see \cite{Bonzon16} for an overview). 

\item[Modelling positive and negative relations between arguments] 
The notions of attack and support are clearly important aspects of argumentation, 
even though the formalization of 
the interaction between these two types of 
relationship is open to multiple interpretation and subject to some debate in the research community 
\cite{CayrolLS13,BrewkaPW14,Prakken14,Polberg16,PolbergHunter17,CabrioVillata13,KontarinisToni16,RosenfeldKraus16}. 
 Nevertheless, there are various studies showing the importance of support in real argumentation, 
such as works on argument mining \cite{CabrioVillata13} or dialogical argumentation \cite{PolbergHunter17}. 
Hence, this requirement is that we need to model how the beliefs in arguments can have a positive or negative influence on 
other arguments, and that the belief in an argument needs to take into account those influences. 

\item[Modelling context--sensitivity]
 Consider two argumentation scenarios represented by the same directed graph, but with arguments instantiated with different 
textual descriptions. The belief that an individual has in these arguments may not be the same in these two scenarios.
The way arguments and influences between them are evaluated can be affected by 
the actual content of the arguments and the problem domain in question \cite{Cerutti2014}.  
Two different instantiations can be interpreted differently by a single user depending on his or her knowledge or preferences 
\cite{Zeng2008}. 
Hence, this requirement is that the context (i.e. how arguments are instantiated) can affect the belief in arguments 
and their influence on other arguments. 


\item[Modelling different perspectives]
It is common for different people to perceive the same information in different ways. 
In argumentation, not only can a given graph be evaluated in various ways by different agents,
but also its structure might not be uniformly perceived  \cite{PolbergHunter17}. 
So if we have an argumentation scenario represented by a single directed graph, 
different people might have dissimilar beliefs in the individual arguments 
and in the influence the belief in one argument has on other arguments.
Partly, this divergence of opinions may occur because of an argument being an enthymeme
 (i.e. an argument that only has some of its premises and claim represented explicitly), which every agent 
can decode differently \cite{Black:2012}. 
This disparity may also occur because of differing background knowledge and experience. 
So this requirement is that the participant (i.e. the agent judging the argument graph) can have belief in 
arguments and their influences that is different to other participants.  



\item[Modelling imperfect agents]
People can exhibit a number of imperfections such as errors in their background knowledge, 
errors in the way they analyze certain information, and biases in how they process information in general. 
So when judging a given argumentation scenario 
some people might make inappropriate or irrational judgments.
This irrationality could be seen in terms of not adhering to argumentation semantics as well as 
in terms of reasoning fallacies or undesirable cognitive biases \cite{ogden2012health}.
Since we want our formalism to be useful for real-world applications, 
we need the ability to model the imperfect agents in their assessment of belief in arguments
and in their assessment of the influence between arguments.

\item[Modelling incomplete situations]
An argumentation graph might not contain all of the arguments relevant to a given problem, in particular those 
 that concern the agent(s). For example, a patient might withhold a certain embarrassing or private piece of information 
from the doctor, 
despite the fact that it can affect the diagnosis. 
However, this incomplete knowledge might also be a result of how the graph is obtained or updated. 
In dialogical argumentation, depending on the used protocol, an agent might not always be able to put forward all arguments 
relevant to the discussion.  
As a result, an agent may, for example, disbelieve an argument that is perceived by us as unattacked, even though 
the agent is privately aware of reasons to doubt the argument. Similarly, an agent can believe an argument despite it 
being attacked, simply because the graph does not contain the agent's supporting arguments. 
Such a behaviour would violate the majority 
of the argumentation semantics available in the literature. 
Furthermore, graph incompleteness in combination with fine--grained acceptability means that we might know that an 
agent believes or disbelieves a particular argument, but cannot precisely state to what degree.
We therefore need an approach that is more resilient 
to potential incompleteness of the possessed information. 
\end{description} 

In the following examples, we consider simple scenarios 
where we might use monological argumentation to make sense of a situation, and possibly to make decisions. 
The examples highlight the value of implementing the above requirements. 

\begin{figure}[ht]
\begin{subfigure}[b]{0.49\textwidth}
\begin{center}
\resizebox{\textwidth}{!}{
\begin{tikzpicture}[->,>=latex,thick, arg/.style={draw,text centered,shape = rounded rectangle,fill=darkgreen!10 }]
\node[arg] (a1) [text width=3.7cm] at (-1.2,0) {$\xa$ =  The train will arrive at 2pm because it is timetabled for a 2pm arrival.};
\node[arg] (a2) [text width=3.5cm] at (4,3)  {$\xb$ =   Normally this train service arrives a little bit late.};
\node[arg] (a3) [text width=3.5cm] at (4,0)  {$\xc$ =   The train appears to be travelling slower than normal.};
\node[arg] (a4) [text width=3.5cm] at (4,-3)  {$\xd$ =  The live travel info app lists the train being on time.};
\path	(a2) edge[] node[above left] {$-$} (a1);
\path	(a3) edge[] node[above] {$-$} (a1);
\end{tikzpicture}
}
\end{center}
\caption{\label{fig:train} Jack's graph.}
\end{subfigure}
\begin{subfigure}[b]{0.49\textwidth}
\begin{center}
\resizebox{\textwidth}{!}{
\begin{tikzpicture}[->,>=latex,thick, arg/.style={draw,text centered,shape = rounded rectangle,fill=darkgreen!10 }]
\node[arg] (a1) [text width=3.7cm] at (-1.2,0) {$\xa$ =  The train will arrive at 2pm because it is timetabled for a 2pm arrival.};
\node[arg] (a2) [text width=3.5cm] at (4,3)  {$\xb$ =   Normally this train service arrives a little bit late.};
\node[arg] (a3) [text width=3.5cm] at (4,0)  {$\xc$ =   The train appears to be travelling slower than normal.};
\node[arg] (a4) [text width=3.5cm] at (4,-3)  {$\xd$ =  The live travel info app lists the train being on time.};
\path	(a2) edge[] node[above left] {$-$} (a1);
\path	(a4) edge node[above right] {$+$} (a1);
\end{tikzpicture}
}
\end{center}
\caption{\label{fig:train2} Jill's graph.}
\end{subfigure}
\caption{Labelleds graph concerning the arrival time of a train journey. Edges labelled with $-$ represent 
attack and edges labelled with $+$ represent support.}
\end{figure}
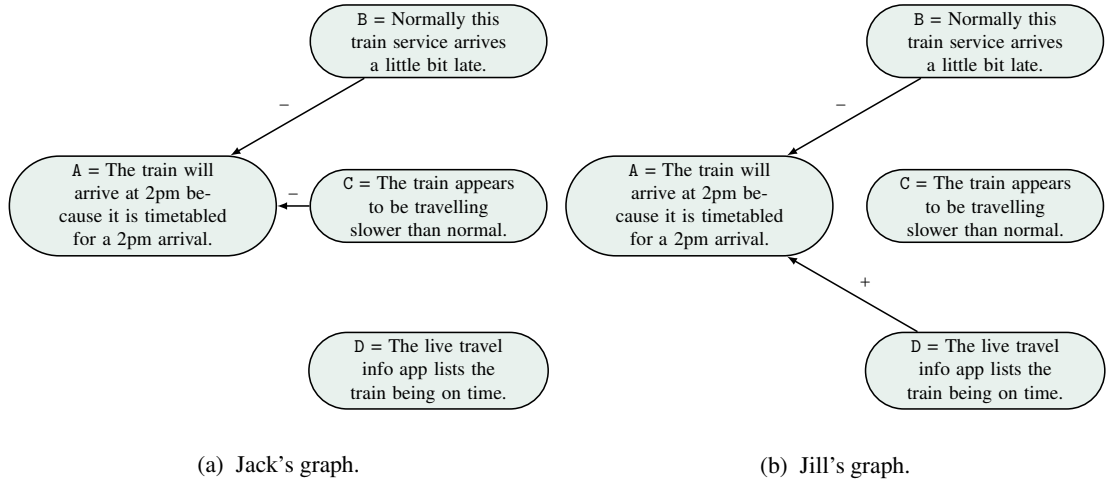

\begin{example}
\label{ex:trains}
Imagine we have two passengers on a train, Jack and Jill, travelling to work.  
Jack is using this particular connection regularly  
and has some experience with the vagaries of the service. Jill, however, uses this connection for the first time, 
and has an important meeting to attend and wants to be on time. 
Let us assume that their knowledge concerning whether the train is going to be late 
is represented in Figures \ref{fig:train} and  \ref{fig:train2}.

Let us first focus on Jack. 
Arguments $\xb$, $\xc$ and $\xd$ are enthymemes, 
and so their claims are not explicit and can be decoded in a number of ways. Since Jack is a regular client we
can assume that the missing claim for $\xb$ (and $\xc$) is 
\enquote{therefore the train will arrive a bit late}. He also has a live travel info app that says that this service
will arrive on time. He has been using this app for a while and does not consider it reliable at all, and because of this experience 
he chooses to decode the claim of $\xd$ as \enquote{the live travel info service predicts the train will be on time}. So, 
Jack does not decode the claim of $\xd$ as \enquote{therefore the train will arrive on time}. 
Hence, he sees arguments $\xb$ and $\xc$ as attacking $\xa$ and disregards the influence of $\xd$. 
Thus, Jack's belief in $\xb$ and $\xc$ suggests that $\xa$ should be disbelieved, and the degree to which 
 he disbelieves or believes $\xa$ should be primarily affected by $\xc$, i.e. his current perception of the service.  
At the same time, he is certain of his 
eyes, i.e. that the info service predicts that the train will be on time, but his belief in that argument does not affect $\xa$. 

Let us now focus on Jill, who is new to the service. She heard from a fellow passenger that this train normally arrives 
a little bit late and chooses to decode the claim of $\xb$ as \enquote{therefore the timetable is inaccurate}. 
This is the first time she has used this particular service, as she had only recently moved from a different town. 
She commuted by train before, but the connection she used from her previous town was a faster one. 
Therefore, she sees argument $\xc$ as a comment on the new line when compared to the line she used before, 
not as a sign of problems happening right now on the train she has boarded. 
Thus, 
her claim for $\xc$ is \enquote{therefore the tracks on this line must be in a worse condition than on the other line} 
and for her, arguments $\xa$ and $\xc$ are not particularly related. Finally, the live travel app she has been using 
has been very reliable in the past and she trusts it. She decodes the claim of $\xd$ as 
\enquote{therefore, the train will be on time}. Therefore, as long as she believes $\xd$ more than she believes 
the complaints of a random stranger on the train (i.e. argument $\xb$), she will believe $\xa$. 
\end{example}

The above example indicates how considering arguments, and beliefs in them, can be a useful part of 
sense-making and decision-making in monological argumentation. How we model the influence of positive 
and negative relations is an important part of this. Furthermore, we can see that there is context-sensitivity, 
in that how we interpret arguments (in particular how we decode enthymemes) and the relationships between 
them can affect this analysis. We can also see that it is reasonable for different agents having different views 
on how to decode a given enthymeme, different views on the influence of one argument on another, and 
different views on how to take multiple relationships into account.  

\begin{figure}[ht]
\begin{center}
\begin{tikzpicture}[->,>=latex,thick, arg/.style={draw,text centered,shape = rounded rectangle,fill=darkgreen!10,font=\normalsize}]
\node[arg] (a1) [text width=4.5cm] at (-1,1) {$\xa$  = Giving up smoking will be good for your health};
\node[arg] (a2) [text width=5cm] at (6,1)  {$\xb$  = My appetite will increase and so I will put on too much weight};
\node[arg] (a3) [text width=5cm] at (6,-1.5)  {$\xc$  = My anxiety will increase and so I will lose too much weight};
\node[arg] (a7) [text width=5cm] at (-1,-1.5)  {$\xd$  = My anxiety will increase and so I will have problems with working};
\node[arg] (a4) [text width=6cm] at (6,3)  {$\xe$ = You can join a healthy eating course to help you manage your weight};
\node[arg] (a5) [text width=5cm] at (-1,-4)  {$\xf$ = You can join a yoga class to help you relax, and thereby manage your anxiety};
\node[arg] (a6) [text width=6.5cm] at (6,-4)  {$\xg$ = You can use online counseling services for anxiety associated with 
smoking cessation, and thereby manage your anxiety};
\path	(a2) edge[bend left]  node[right] {$-$}(a3);
\path	(a3) edge[bend left]  node[left] {$-$} (a2);
\path	(a2) edge  node[above] {$-$}(a1);
\path	(a3) edge node[below] {$-$} (a1);
\path	(a7) edge node[left] {$-$} (a1);
\path	(a4) edge  node[left] {$-$} (a2);
\path	(a5) edge  node[above,pos=0.2] {$-$} (a3);
\path	(a6) edge  node[left] {$-$} (a3);
\path	(a5) edge  node[left] {$-$} (a7);
\path	(a6) edge  node[above,pos=0.2] {$-$} (a7);
\end{tikzpicture}
\end{center}
\caption{\label{fig:smoking} Example of argument graph for persuading someone to give up smoking. 
Edges labelled with $-$ represent attack. The graph contains the arguments known (but not necessarily believed by) 
the artificial agent, and might not contain all arguments of Rachel, Robin or Morgan.}
\end{figure}
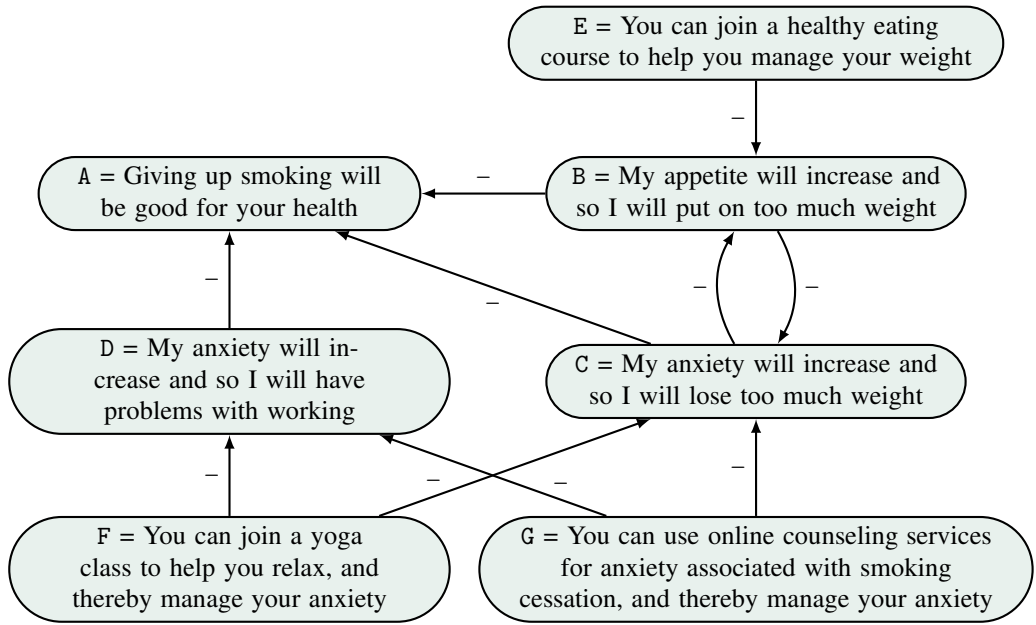
 
\begin{example}
\label{ex:smoking}
Let us now assume that we have an artificial agent attempting to persuade the users Rachel, Robin and Morgan to stop smoking. 
The graph of the artificial agent is represented by Figure \ref{fig:smoking}. The dialogue proceeds in turns and limits the ways 
the participants can respond. The artificial agent can state 
any of the arguments in the graph and the user is allowed to react in two ways. A user (be it Rachel, Robin or Morgan) 
can either select 
his/her counterargument from the list presented by the agent, or 
state how much (s)he agrees or disagrees with an argument presented by the agent. 
The user can end the dialogue at any time, 
the agent ends once there are no arguments to put forward or the user agreed to the desired arguments. 
After the dialogue is finished by any party, the participant is asked whether he or she agrees or disagrees with argument $\xa$. 
If the participant agrees, the dialogue is marked as successful. 

Let us start with Rachel. The agent presents her with argument $\xa$ in order to convince her to stop smoking
and allows her to select from $\xb$, $\xc$ and $\xd$ as her potential arguments. Rachel selects $\xb$ and $\xd$. 
In response to $\xb$, the agent puts forward $\xe$, and Rachel agrees. In response to $\xd$, the agent decides to first put 
forward $\xf$  based on the experience with previous users. Unfortunately, 
Rachel strongly disagrees and ends the discussion. The dialogue is marked as unsuccessful. 
The agent was not aware that Rachel uses a wheelchair and that yoga classes
did not suit her requirements, and the conversation ended before $\xg$ could have been put forward. 

Let us now consider Robin. The agent presents Robin with $\xa$ and again allows $\xb$, $\xc$ and $\xd$ to be selected as 
counterarguments. Robin is afraid of any weight changes associated with smoking cessation and selects both $\xb$ and $\xc$ 
despite the fact that they are conflicting. Consequently, any counterarguments put forward by the artificial agent 
can be seen as at the same time indirectly conflicting with and promoting $\xa$. 
The agent puts forward $\xe$ and $\xf$, to which Robin moderately agrees, and the dialogue ends 
successfully. 

Finally, consider Morgan, who similarly to Rachel selects both $\xb$ and $\xd$. However, in reality, 
Morgan is more afraid of weight gain 
than anxiety affecting his work, but wants to discuss both issues. The agent proposes solutions and Morgan moderately agrees with 
$\xe$, but somewhat disagrees with $\xf$. The agent decides to follow up with $\xg$, with which Morgan strongly disagrees. 
Nevertheless, the dialogue ends successfully due to the fact that Morgan's more pressing issue was addressed. 
\end{example}
  
The above example indicates how beliefs in arguments and relations between them are important in
dialogical argumentation. In particular, the same procedures applied to two agents expressing similar concerns can lead to 
different results based on the beliefs they have in arguments and their private knowledge.
An agent not aware of another agent's arguments or beliefs can put forward unacceptable arguments and
fail to persuade a given party to do or not to do something. One also has to be ready to put forward arguments that, 
possibly due to certain behaviours of the other party that can be deemed not rational, might
work against the agent’s goal.

\begin{example}[Adapted from \cite{Cerutti2014,haepilot}]
\label{ex:context}
The work in \cite{Cerutti2014} has investigated the problem of reinstatement in argumentation 
using an instantiated theory and preferences. We draw attention to two scenarios considered in the study, 
concerning weather forecast and car purchase, where each comes in the basic (without the last sentence) 
and extended (full text) version. 

\begin{displayquote}
The weather forecasting service of the broadcasting company AAA says
that it will rain tomorrow. Meanwhile, the forecast service of the broadcasting
company BBB says that it will be cloudy tomorrow but that it will not rain. It
is also well known that the forecasting service of BBB is more accurate than
the one of AAA. \textit{However, yesterday the trustworthy newspaper CCC published
an article which said that BBB has cut the resources for its weather forecasting
service in the past months, thus making it less reliable than in the past.}
\end{displayquote}
 
\begin{displayquote} 
You are planning to buy a second-hand car, and you go to a dealership
with BBB, a mechanic whom has been recommended you by a friend. The
salesperson AAA shows you a car and says that it needs very little work done
to it. BBB says it will require quite a lot of work, because in the past he had to
fix several issues in a car of the same model. \textit{While you are at the dealership,
your friend calls you to tell you that he knows (beyond a shadow of a doubt)
that BBB made unnecessary repairs to his car last month.}
\end{displayquote}

The formal representation of the base (resp. extended) versions of these scenarios is equivalent (we refer to 
\cite{Cerutti2014,haepilot} for more details). However, 
the findings show that they are not judged in the same way and suggest that the domain dependent knowledge of the 
participants has affected their performance of the tasks.
This shows the importance of modelling context--sensitivity and allowing an agent to evaluate structurally equivalent graphs 
differently.  

\end{example}

\subsection*{\textbf{State of the Art}} 

There are various approaches that attempt to tackle some of the above requirements, as we will discuss in more detail 
in Section \ref{section:Comparison}. However, there does not 
exist one that would be able to deal with all of them at the same time.  
We can find a number of proposals in computational models 
of argument such as the postulates for 
argument weights, strengths or beliefs
\cite{CayrolLS05,CayrolLS05b,LeiteMartins11,Amgoud2013,Amgoud16kr,AmgoudBenNaim16a,AmgoudBenNaim16b,Rago16,Bonzon16,
AmgoudBenNaim17,AmgoudBNDV17,Thimm:2012,Hunter:2013,Hunter:2014,Costa-Pereira:2011}, which offer a more fine--grained 
alternative for Dung's approach. 
Some of these works also permit certain forms of support or positive influences on arguments 
\cite{CayrolLS05, AmgoudBenNaim17,AmgoudBenNaim16a,Rago16,LeiteMartins11}. 
Nevertheless, due to the way the influence aggregation methods are defined, it is difficult for these proposals to meet the 
requirement for modelling context--sensitivity, different perspectives or 
incomplete graphs. Certain flexibility is perhaps possible only with approaches that work with initial scoring assignment such as 
\cite{Rago16, LeiteMartins11, AmgoudBenNaim17, AmgoudBNDV17}, though dealing with imperfect agents still poses difficulties 
in these methods.  
  
In contrast to the above works, there are frameworks that allow us to specify 
the way one argument affects another locally, which promotes dealing with context--sensitivity, different perspectives 
and a wider range of relations between arguments.  
In particular, abstract dialectical frameworks (ADFs) \cite{BrewkaWoltran10,BrewkaESWW13,Polberg16} allow us to 
specify various ways the incoming support and attack can affect a given argument. 
Unfortunately,  this specification has certain restrictions, and  dealing with incomplete scenarios 
and imperfect agents is not ideal. The semantics of ADFs are also primarily two and three--valued, and while 
a recent generalizations allows for considering fine-grained acceptability \cite{Brewka18}, it still suffers from the previously 
mentioned issues.  

\subsection*{\textbf{Our Proposal}} 

 We therefore believe there is a need to investigate  argumentative approaches that would handle both attack and support relations, 
allow for fine--grained argument acceptability,
and permit context--sensitivity, different perspectives, agents' imperfections and incomplete knowledge about agents' graphs. 
%
As a starting point for our research, we take the epistemic approach to probabilistic argumentation, which 
 has already shown to be potentially valuable in modelling agents in persuasion dialogues 
\cite{Hunter15ijcai,Hunter16sum,Hunter16ecai,Hadoux16,Hadoux17}. 
In order to address our requirements, we introduce epistemic graphs a generalization of this approach.
In these graphs, an argument can be believed or disbelieved to a given degree, and 
the way other arguments influence a given argument is expressed by the epistemic constraints that 
can restrict the belief we have in an argument in a flexible way.   


Through the use of degrees of belief, epistemic graphs 
provide a more fine--grained alternative to classical Dung's approaches when it comes 
to determining the status of a given argument. The flexibility of the epistemic approach allows us to both model 
the rationale behind the existing semantics as well as completely deviate from them, 
thus giving us a more appropriate formalism for practical situations including the modelling of imperfect agents. 
Epistemic graphs can model both attack and support as well as
relations that are neither support nor attack, so far analyzed primarily in the context of abstract 
dialectical frameworks \cite{BrewkaESWW13}. The freedom in defining the constraints allows us to easily 
express various interpretations of support at the same time and without the need to transform them, 
which is usually necessary in other types of argumentation frameworks \cite{CayrolLS13,PolbergOren14a,Polberg16}.
The fact that we can specify the rules under which arguments should be evaluated and that we can
include constraints between unrelated arguments allows the framework to be more context--sensitive 
and more accommodating when it comes to dealing with  imperfect agents. Additionally, the ability to leave 
certain relations unspecified lets us deal with cases when the system has insufficient knowledge about the situation. 
%

In this paper, we make the following contributions: 
\begin{inparaenum}[\itshape 1\upshape)]
\item A syntax and semantics for a logical language for constraints that is appropriate for argumentation;
\item A proof theory for reasoning with these constraints so that we can determine whether a set of constraints is consistent, whether 
a set of constraints is minimal, and whether one constraint implies another constraint;
\item A definition for epistemic graphs, study of its properties and an analysis of how it can be used to capture different kinds of argumentation scenarios;
and  
\item  A set of tools for analysing the relationship 
between the graphical structure and the 
constraints contained in the graph and an example of how they can be harnessed in practical applications.  
\end{inparaenum}

In this paper, we do not consider how we can obtain the probability distribution or constraints. However, other works with crowdsourced 
data show how we can obtain belief in arguments and relations between them \cite{Cerutti2014,HunterPolberg17,PolbergHunter17}. 
From this crowdsourced data, we believe that it is entirely feasible to develop machine learning techniques for generating constraints. 
However, we leave the learning of constraints from data to future work. 

\subsection*{\textbf{Outline of the Paper}}

We proceed as follows:
Section \ref{section:Preliminaries} reviews the background that we require;
Section \ref{section:epistemiclanguage} introduces the syntax and semantics for the language we require for specifying epistemic 
graphs; 
Section \ref{section:ReasoningConstraints} presents the proof theory for reasoning with statements in this language;
Section \ref{sec:epistemicgraphs} introduces epistemic graphs and considers how they can be used for analysing different kinds of 
scenarios; 
Section \ref{section:Comparison} compares our work to related state-of-the-art formalisms;
and Section \ref{section:Discussion} discusses our contribution and considers future work.


\section{Preliminaries}
\label{section:Preliminaries} 

In its simplest form, an argument graph is a directed graph in which nodes represent arguments and arcs represent relations. 
In conflict--based graphs, such as 
the ones created by Dung \cite{Dung95}, arcs stand for attacks. In graphs such as those in 
\cite{AmgoudBenNaim16a}, arcs are supports, while in bipolar graphs they can 
be either supports or attacks \cite{CayrolLS13,BoellaGTV10,OrenNorman08,NouiouaRisch11,PolbergOren14a}. 
In some frameworks, such as abstract dialectical 
frameworks, an arc may also represent a dependence relation in case it cannot be strictly classified as neither supporting nor attacking 
\cite{BrewkaWoltran10,BrewkaESWW13,Polberg16}. Argument graphs can be extended in various ways in order to account for 
additional preferences, recursive relations, group relations\footnote{Frameworks with recursive relations are represented as 
generalizations of directed graphs where edges point at other edges. Frameworks with group relations are often represented by B--
graphs, i.e. directed hypergraphs where the head of the edge is a single 
node.} and more. For an overview, we refer the reader to \cite{BrewkaPW14}. 
We will also discuss some of these structures more in Section 
\ref{section:Comparison}. 
For now, we will focus on 
introducing the notation we will use throughout the text.

By an argument graph we will understand a directed graph 
and we will use a labelling function that assigns to every arc a label representing 
its nature -- supporting, attacking, or dependent, where dependency is understood as a relation 
that is neither positive nor negative. Hence, unless stated otherwise, we will assume we are working
with a label set $\Omega = \{+,-,*\}$, which can be adjusted if needed. 
Given that many argumentation graphs allow two arguments to be connected in more ways than one, we allow a single arc to possess more than just one 
label:

\begin{definition}
Let $\graph = (V, A)$, where $A \subseteq V \times V$, be a directed graph. 
A {\bf labelled graph} is a tuple $X = (\graph,\lab)$ where $\lab: A \rightarrow 2^\Omega$ 
is labelling function and $\Omega$ is a set of possible labels. $X$ is \textbf{fully labelled} iff
for every $\alpha \in A$, $\lab(\alpha) \neq \emptyset$. $X$ is \textbf{uni-labelled} iff 
for every $\alpha \in A$, $\lvert \lab(\alpha) \rvert = 1$. 
\end{definition}

Unless stated otherwise, from now on we assume that we are working with fully labelled graphs. 
With $\nodes(\graph)$ we denote the set of nodes $V$ in the graph $\graph$ and with
 $\arcs(\graph)$ we denote the set of arcs $A$ in $\graph$. For a 
graph $\graph$ and a node $\xb \in \nodes(\graph)$, the parents of $\xb$ are 
$\parent(\xb) = \{ \xa \mid (\xa,\xb) \in \arcs(\graph) \}$. With 
$\lab^x(\graph) = \{ \alpha \in \arcs(\graph) \mid x \in \lab(\alpha)\}$ we denote the set of relations labelled with $x$ by $\lab$, 
where $x \in \{+,*,-\}$. In a similar fashion, by $\parent^x(\xb) = \{ \xa \mid (\xa,\xb) \in \arcs(\graph) \land 
x \in \lab((\xa, \xb)) \}$ we will denote the set of parents of an argument $\xb$ s.t. the relation between the two is 
labelled with $x$ by $\lab$. 

On an arc from a parent to the target, a positive label denotes a positive influence, a negative label 
denotes a negative influence, and a star label denotes an influence that is neither strictly positive nor negative. 
If $\lab$ is assigned only the $-$ label to every arc in a graph, then the graph is a conflict--based argument graph, 
and if $\lab$ is assigned $+$ or $-$ (or both) to every 
arc in a graph, then the graph is a bipolar argument graph \cite{NouiouaRisch11,CayrolLS13,PolbergOren14a}. Following 
the analysis in \cite{PolbergHunter17}, a graph making use of all three labels will be referred to as tripolar. 
In Figure \ref{fig:smoking} we 
can see an example of a conflict--based argument graph, Figure \ref{fig:train} shows an example of a bipolar argument graph 
and Figure \ref{fig:tripolar} of a tripolar one. In the last case, we can observe that while $\xe$ and $\xf$ 
are necessary for $\xa$, only one of them can be accepted at a time in order for $\xa$ to be accepted, as having 
both of them would lead to rejecting the argument. This mutual exclusivity requirement for $\xa$ 
is neither an attacking nor a supporting relation, and thus it is classified as a dependency. 
 
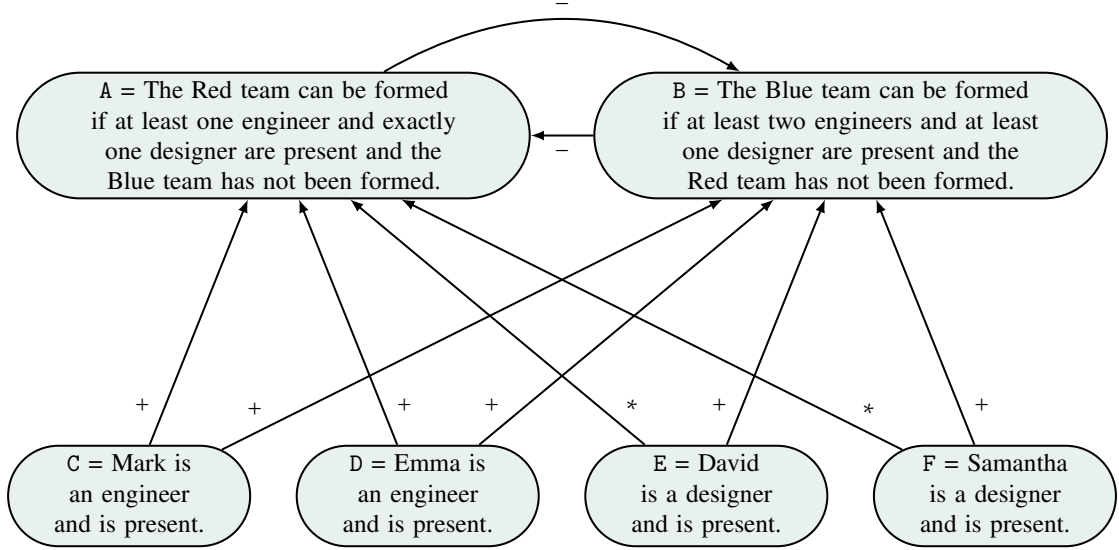
\begin{figure}[ht]
\begin{center}
\resizebox{\textwidth}{!}{
\begin{tikzpicture}[->,>=latex,thick,auto, arg/.style={draw,text centered,shape = rounded rectangle,fill=darkgreen!10}]
\node[arg] (a1) [text width=6cm] at (-2,0) {$\xa$ =  The Red team can be formed if at least one engineer and 
exactly one designer 
are present and the Blue team has not been formed.};
\node[arg] (a2) [text width=6cm] at (6,0)  {$\xb$ =  The Blue team can be formed if at least two engineers and at least one 
designer are present and the Red team has not been formed.};

\node[arg] (a3) [text width=2.5cm] at (-4,-5)  {$\xc$ =   Mark is an engineer and is present.};
\node[arg] (a4) [text width=2.5cm] at (0,-5)  {$\xd$ =   Emma is an engineer and is present.};
\node[arg] (a5) [text width=2.5cm] at (4,-5)  {$\xe$ =  David is a designer and is present.};
\node[arg] (a6) [text width=2.5cm] at (8,-5)  {$\xf$ = Samantha is a designer and is present.}; 

\path	(a3) edge node[pos=0.1] {$+$} (a1)
		(a4) edge node[swap,pos=0.1] {$+$} (a1)
		(a5) edge node[swap,pos=0.1] {$*$} (a1)
		(a6) edge node[swap,pos=0.1] {$*$} (a1) 

		(a3) edge node[pos=0.1]{$+$} (a2)
		(a4) edge node[pos=0.1]{$+$} (a2)
		(a5) edge node[pos=0.1]  {$+$} (a2)
		(a6) edge node[swap,pos=0.1]  {$+$} (a2) 

		(a1) edge [bend left] node {$-$} (a2)
		(a2) edge node {$-$} (a1)
; 
\end{tikzpicture}
}
\end{center}
\caption{\label{fig:tripolar}
A tripolar graph example. Edges labelled with $-$, $+$ and $*$ represent attack, support and dependency respectively. 
Forming of the Red and Blue teams requires particular specialists. Arguments $\xc$, $\xd$, $\xe$ and $\xf$
support the creation of the Blue team. However, only $\xc$ and $\xd$ strictly support the creation of the Red team. 
If we accept $\xe$, then acceptance of $\xf$ leads to the rejection of the Red team, and if we accept $\xf$, 
then the acceptance of $\xe$ leads to the rejection of the Red team. At the same time, one of $\xe$ and $\xf$
has to be accepted. Thus, the relations from $\xe$ and $\xf$ are in some cases attacking, in some supporting, 
and hence they can only be classified as dependent. }
\end{figure}

A given argument graph is evaluated with the use of semantics, which are meant to represent what can be 
considered \enquote{reasonable}. 
The most basic type of semantics -- the extension--based ones -- associate a given graph with sets of arguments, called extensions, 
formed from acceptable arguments. A more refined version, the labeling--based semantics, tell us whether an argument is accepted, 
rejected, or neither \cite{Caminada:2009,Baroni:2011,BrewkaESWW13}. However, when it comes to some 
applications such as user modelling, these two and 
three--valued perspectives can be 
insufficient to express the extent to which the user agrees or disagrees with a given argument \cite{Rahwan2011,PolbergHunter17}. 
Consequently, a variety of weighted, ranking--based and probabilistic approaches have been proposed 
\cite{Amgoud2013,AmgoudBenNaim16b,AmgoudBenNaim16a,Amgoud16kr,Rago16,Bonzon16,
AmgoudBenNaim17,AmgoudBNDV17,Hunter:2013,Hunter15ijcai,Hunter16sum,Hunter16ecai,Hadoux16,Hadoux17,Pu2014,BesnardHunter01,CayrolLS05,CayrolLS05b,LeiteMartins11}.
We will discuss some of these approaches further in Section \ref{section:Comparison} and refer interested readers to the listed papers 
for a more in-depth analysis.   
 
 \section{Epistemic Language}
\label{section:epistemiclanguage}

In the introduction, we have discussed the value of being able to model beliefs in arguments, various
types of relations between arguments, context–sensitivity, and more. Our proposal, capable of meeting the
postulated requirements, comes in the form of epistemic graphs, which can be equipped with particular
formulae specifying the beliefs in arguments and the interplay between them. In this section, we will focus on 
providing the language for these formulae. We describe its syntax and semantics as well as introduce an appropriate 
proof system later in Section \ref{section:ReasoningConstraints}.

\subsection{Syntax and Semantics}

The epistemic language consists of Boolean combinations
of inequalities involving statements about probabilities of formulae built out of arguments. 
Throughout the section, we will assume that we have a directed graph $\graph$.
The building block of an epistemic formula is a statement \enquote{probability of $\alpha$}, where $\alpha$ 
is a propositional formula on arguments (further referred to as terms). We can then speak about additions and subtractions 
of probabilities of such terms (further referred to as operational formulae). Comparing them to actual numerical values 
through equalities and inequalities forms epistemic atoms, which can then through negation, disjunction, conjunction etc. 
be joined into epistemic formulae. Let us now formally introduce the language:

\begin{definition}
The epistemic language based on $\graph$ is defined as follows:

\begin{itemize}
\item a {\bf term} is a Boolean combination of arguments. We use $\lor$, $\land$ 
and $\neg$ as connectives in the usual way, and can derive secondary connectives, such as implication $\rightarrow$, as usual. 
$\terms(\graph)$ denotes all the terms that can be formed from the arguments in $\graph$.

\item an {\bf operational formula} is of the form
$\pprob(\alpha_i) \star_1 \ldots \star_{k-1} \pprob(\alpha_{k})$ where all $\alpha_i \in \terms(\graph)$ and $\star_j \in \{+,-\}$. 
$\oformulae(\graph)$ denotes all possible operational formulae of $\graph$ and we read $\pprob(\alpha)$ as 
\enquote{probability of $\alpha$}. 

\item an {\bf epistemic atom} is of the form $f \ineq x$
where 
$\ineq \in \operators$, $x \in [0,1]$ and $f \in \oformulae(\graph)$.

\item an {\bf epistemic formula} is a Boolean combination of epistemic atoms.
$\eformulae(\graph)$ denotes the set of all possible epistemic formulae of $\graph$.  
\end{itemize}
\end{definition}
 
For $\alpha \in \terms(\graph)$, $\args(\alpha)$ denotes the set of all arguments appearing in $\alpha$
and for a set of terms $\Gamma \subseteq \terms(\graph)$, $\args(\Gamma)$ denotes 
the set of all arguments appearing in $\Gamma$.
Given a formula $\psi \in \eformulae(\graph)$, let $\fterms(\psi)$ denote 
the set of terms appearing in $\psi$ and let $\fargs(\psi) = \args(\fterms(\psi))$ be the set 
of arguments appearing in $\psi$. With ${\sf Num}(\psi)$ we denote the collection of all numerical values 
$x$ appearing in $\psi$. For an operational formula $f = \pprob(\alpha_i) \star_1 \ldots \star_{k-1} \pprob(\alpha_{k})$,
$\arop(f) = (\star_1, \star_2, \ldots, \star_{k-1})$ denotes the, possibly empty, sequence of arithmetic operators appearing in $f$. 
By abuse of notation, by $\arop(\varphi)$ for an epistemic atom $\varphi$ we will understand the sequence of operators
of the operational formula of $\varphi$.
 
\begin{example}
Assume a graph $\graph$ s.t.  $\{\xa,\xb,\xc,\xd\} \subseteq \nodes(\graph)$. 
$\psi \formis  \pprob(\xa\land \xb) - \pprob(\xc) - \pprob(\xd) > 0$ is 
an example of an epistemic formula on $\graph$. 
The terms of that formula are $\fterms(\psi) = \{\xa \land \xb, \xc, \xd\}$, 
the arguments appearing in them are $\fargs(\psi) = \{\xa, \xb,\xc,\xd\}$. 
The sequence of operators of $\psi$ is $\arop(\psi) = (-,-)$.
Finally, in this case, ${\sf Num}(\psi) = \{0\}$. 
\end{example}

Having defined the syntax of our language, let us now focus on its semantics, which comes in the form of belief distributions. 
A {\bf belief distribution} on arguments is a function $\prob: 2^{\nodes(\graph)} \rightarrow [0,1]$ s.t. 
$\sum_{\Gamma \subseteq \nodes(\graph)} \prob(\Gamma) = 1$. With $\dist(\graph)$ we denote the set of all 
belief distributions on $\nodes(\graph)$. 
Each $\Gamma \subseteq \nodes(\graph)$ corresponds to an interpretation of arguments. 
We say that $\Gamma$ \emph{satisfies an argument} $\xa$ and write $\Gamma \models \xa$
iff $\xa \in \Gamma$.  
Essentially $\models$ is a classical satisfaction relation and can be extended to complex terms as usual.  
For instance,
$\Gamma \models \neg\alpha$ iff $\Gamma \not\models \alpha$
and $\Gamma  \models \alpha \land \beta$ iff $\Gamma  \models \alpha$ and $\Gamma  \models \beta$.
For each graph $\graph$, we assume an ordering over the arguments $\langle \xa_1,\ldots,\xa_n \rangle$ 
so that we can encode each model by a binary number: for a model $X$, if the i-th argument is in $X$, 
then the i-th digit is 1, otherwise it is 0. For example, for $\langle \xa,\xb,\xc \rangle$, the model $\{\xa,\xc\}$ is 
represented by 101.  

The {\bf probability of a term} is defined as 
the sum of the probabilities (beliefs) of its models:
$$\prob(\alpha)  = \sum_{\Gamma \subseteq \nodes(\graph) \mbox{ s.t. } \Gamma \models \alpha} \prob(\Gamma).$$
We say that an agent believes a term $\alpha$ to some degree if $\prob(\alpha) > 0.5$,
disbelieves $\alpha$ to some degree if $\prob(\alpha) < 0.5$,
and neither believes nor disbelieves $\alpha$ when $\prob(\alpha) = 0.5$. 
Please note in this notation, $\prob(\xa)$ stands for the probability of a simple term $\xa$ (i.e. sum of probabilities 
of all sets containing $\xa$), which is different 
from $\prob(\{\xa\})$, i.e. the probability assigned to set $\{\xa\}$.  

\begin{definition}
\label{def:sat:valued}
Let $\varphi$ be an epistemic atom $\pprob(\alpha_i) \star_1 \ldots \star_{k-1} \pprob(\alpha_{k}) \ineq b$. 
The {\bf satisfying distributions} of $\varphi$ are defined  as 
$\sat(\varphi) = \{\prob \in \dist(\graph) \mid \prob(\alpha_i) \star_1 \ldots \star_{k-1} \prob(\alpha_{k}) \ineq b\}$.
 
The set of satisfying distributions for a given epistemic formula is as follows 
where $\phi$ and $\psi$ are epistemic formulae:  
\begin{itemize}
\item  $\sat(\phi\land\psi) = \sat(\phi) \cap \sat(\psi)$;
\item $\sat(\phi\lor\psi) = \sat(\phi) \cup \sat(\psi)$; 
and 
\item $\sat(\neg\phi) = \sat(\top) \setminus \sat(\phi)$.
\end{itemize}
For a set of epistemic formulae $\Phi = \{ \phi_1,\ldots, \phi_n\}$,
the set of satisfying distributions is $\sat(\Phi)$ 
=  $\sat(\phi_1) \cap \ldots \cap \sat(\phi_n)$.
\end{definition}
\begin{example}
Consider a graph with nodes $\{\xa,\xb,\xc,\xd\}$ and the formulae 
$\psi \formis  \pprob(\xa\land \xb) - \pprob(\xc) - \pprob(\xd) \! > \! 0 \, \land \, \pprob(\xd) \!> \!0$. 
A probability distribution $\prob_1$ with $\prob_1(\xa \wedge \xb) =0.7$, $\prob_1(\xc) = 0.1$ and 
$\prob_1(\xd) = 0.1$ is in $\sat(\psi)$. However, a distribution $\prob_2$ 
with $\prob_2(\xa \wedge \xb) = 0$ cannot satisfy $\psi$ and so  
$\prob_2 \notin \sat(\psi)$. 
\end{example}  

\subsection{Restricted Language}

The full power of the epistemic language, while useful in various scenarios, may be redundant in other. 
For instance, one of the most commonly employed tools in opinion surveys is a Likert scale, which typically 
admits from 5 to 11 possible answer options. Consequently, we would also like to consider the restricted 
epistemic language, i.e. one where the sets of values that the probability function can take on and that can appear as 
numerical values in the formulae are fixed and finite.   

We start by defining the restricted value set, which has to be closed under addition and subtraction (assuming the resulting 
value is still in the $[0,1]$ interval). We can then create subsets of this set according to a given inequality and \enquote{threshold} 
value, as well as sequences of values that can be seen as satisfying a given arithmetical formula:

\begin{definition}
A finite set of rational numbers from the unit interval $\Pi$ is a {\bf restricted value set} iff 
for any $x, y \in \Pi$ it holds that if $x+y \leq 1$, then $x + y \in \Pi$, and if $x-y \geq 0$, then $x-y \in \Pi$. 
For $\Pi \neq \emptyset$, with $\Pi^x_{\ineq} = \{ y \in \Pi \mid y \ineq x \}$ we denote the subset of $\Pi$ 
obtained according to the value $x$ and relationship $\ineq \in \operators$. 
The {\bf combination set} for a nonempty restricted value set $\Pi$ and 
a sequence of arithmetic operations $(*_1, \ldots, *_k)$ where $*_i \in \{+,-\}$ 
and $k\geq 0$ is defined as: 
\[
\Pi^{x,(*_1, \ldots, *_k)}_{\ineq} = \begin{cases}
   \{(v) \mid v \in \Pi^{x}_{\ineq} \}&	k=0 \\
             \{(v_1, \ldots, v_{k+1}) \mid v_i \in \Pi, \, v_1 *_1 \ldots *_k v_{k+1} \ineq x\} & \mbox{ otherwise } \\
\end{cases}\]
\end{definition} 

\begin{example} 
Let $\Pi_1 = \{ 0, 0.5, 0.75, 1 \}$. We can observe that it is not a restricted value set, since $0.75-0.5 = 0.25$ is missing 
from $\Pi_1$. Its modification, $\Pi_2 = \{ 0, 0.25, 0.5, 0.75, 1 \}$, is a restricted value set. Similarly, it is easy to show that 
$\Pi_3 = \{0, \frac{1}{3}, \frac{2}{3}, \frac{3}{3}\}$ and $\Pi_4 = \{0, \frac{2}{5}, \frac{4}{5}\}$ are also restricted value sets.  
The subsets of $\Pi_2$ for $x = 0.25$ under various inequalities are as follows: 
${\Pi_2}^x_{>} = \{ 0.5,0.75,1\}$, 
${\Pi_2}^x_{<} = \{ 0\}$, 
${\Pi_2}^x_{\geq} = \{0.25,0.5,0.75,1\}$, 
${\Pi_2}^x_{\leq} = \{ 0, 0.25\}$, 
${\Pi_2}^x_{\neq} = \{0, 0.5,0.75,1\}$, and
${\Pi_2}^x_{=} = \{ 0.25\}$. 

Assume we have a restricted value set $\Pi_3 = \{0,0.5,1\}$, a sequence of operations $(+, -)$, an operator $=$ 
and a value $x=1$. In order to find appropriate combination sets, 
we are simply looking for triples of values $(\tau_1,\tau_2,\tau_3)$ s.t. $x+y-z = 1$. 
 By collecting such 
combinations of values from $\Pi_3$, we obtain six possible value sequences, i.e. 
${\Pi_3}^{1,(+,-)}_{=} = \{
 (	0	,	1	,	0	)$, 
$(	0.5	,	0.5	,	0	)$,
$(	0.5	,	1	,	0.5	)$,
$(	1	,	0	,	0	)$, 
$(	1	,	0.5	,	0.5	)$,
$(	1	,	1	,	1	) \}$. 
\end{example}

On the basis of a given restricted value set, we can now constrain our approach both in a syntactic 
and in a semantic way: 

\begin{definition} 
Let $\Pi$ be a restricted value set. 
An epistemic formula $\psi \in \eformulae(\graph)$ is {\bf restricted} w.r.t. $\Pi$ iff ${\sf Num}(\psi) \subseteq \Pi$. 
Let $\eformulae(\graph,\Pi)$ denote this set of restricted epistemic formulae.    
\end{definition}

\begin{definition}
Let $\Pi$ be a restricted value set.  
A probability distribution $\prob \in \dist(\graph)$ is {\bf restricted} w.r.t. $\Pi$ iff
 for every $X \subseteq \nodes(\graph)$, 
$\prob(X) \in \Pi$ and for every argument $\xa \in \nodes(\graph)$, $\prob(\xa) \in \Pi$. 
Let $\dist(\graph, \Pi)$ denote the set of restricted distributions of $\graph$. 
\end{definition}

\begin{definition}
\label{def:restricteddist}
Let $\Pi$ be a restricted value set. 
For $\psi \in \eformulae(\graph,\Pi)$, the {\bf restricted satisfying distribution} w.r.t. $\Pi$, denoted $\sat(\psi,\Pi)$, is
\[
\sat(\psi,\Pi) = \sat(\psi) \cap \dist(\graph,\Pi)
\]
\end{definition}

Due to the properties of $\cap$, $\cup$ and $\setminus$, we can observe that restricted satisfying distributions 
can be manipulated similarly to the unrestricted ones, i.e. the following hold for formulae $\psi$ and $\phi$:
\begin{itemize}
\item  $\sat(\phi\land\psi,\Pi) = \sat(\phi,\Pi) \cap \sat(\psi,\Pi)$;
\item $\sat(\phi\lor\psi,\Pi) = \sat(\phi,\Pi) \cup \sat(\psi,\Pi)$; 
and 
\item $\sat(\neg\phi,\Pi) = \sat(\top,\Pi) \setminus \sat(\phi,\Pi)$.
\end{itemize}
 
\begin{example} 
Let $\Pi = \{ 0, 0.5, 1 \}$.  
In the epistemic language restricted w.r.t. $\Pi$, we can only have atoms of the form 
$\beta \ineq 0$, $\beta \ineq 0.5$, and $\beta \ineq 1$, where $\beta \in \oformulae(\graph)$ and $\ineq \in \operators$. 
From these atoms we compose epistemic formulae 
using the Boolean connectives.  Let us assume we have a formula $\pprob(\xa) + \pprob(\xb) \leq 0.5$ on a graph 
s.t. $\{\xa, \xb\} = \nodes(\graph)$. 
We can create three
restricted satisfying distributions, 
namely $\prob_1$ s.t. $\prob_1(00) = 1$,  $\prob_1(10) = 0$, $\prob_1(01) = 0$ and $\prob_1(11) = 0$,
$\prob_2$ s.t. $\prob_2(00) = 0.5$,  $\prob_2(10) = 0.5$, $\prob_2(01) = 0$ and $\prob_2(11) = 0$,
and $\prob_3$ s.t. $\prob_3(00) = 0.5$,  $\prob_3(10) = 0$, $\prob_3(01) = 0.5$ and $\prob_3(11) = 0$.
\end{example}

We can observe that depending on the graph and the restricted value set, it might not be possible to create a restricted distribution. 
For example, we can consider the set $\{0, 0.9\}$. Although it meets the restricted value set requirements, 
there is no way to add or subtract $0$ and $0.9$ such that they add up to $1$. This means that it is not possible 
to define a distribution with this set. Thus, it makes sense to consider also a stronger version of $\Pi$ that prevents 
such scenarios:

\begin{definition}
Let $\Pi$ be a restricted value set. 
$\Pi$ is {\bf reasonable} iff for every graph $\graph$ s.t. $\nodes(\graph) \neq \emptyset$, 
$\dist(\graph, \Pi) \neq \emptyset$. 
\end{definition}

The following simple properties allow us to easily detect reasonable restricted sets: 

\begin{restatable}{lemma}{reasonablerestricted}
\label{lemma:reasonablerestricted}
The following hold:
\begin{itemize}
\item If $\Pi$ is a nonempty restricted value set, then $0 \in \Pi$. 
\item If $\Pi$ is a reasonable restricted value set, then $0 \in \Pi$. 
\item A restricted value set $\Pi$ is reasonable iff $1 \in \Pi$.  
\end{itemize}
\end{restatable}

It can happen that the combination sets or value subsets of $\Pi$ are empty. However, as we can see, this occurs 
only if particular conditions are met:

\begin{restatable}{proposition}{restrnonempty}
\label{restrnonempty}
Let $\Pi$ be a nonempty restricted value set, $x \in \Pi$ a value, $\ineq \in \operators$ an inequality, and 
$(*_1, \ldots, *_k)$ a sequence of operators where $*_i \in \{+,-\}$ and $k\geq 0$. 
Let $max(\Pi)$ denote the maximal value of $\Pi$. The following hold:
 
\begin{itemize}
\item $\Pi^x_\ineq = \emptyset$ if and only if: 
	\begin{enumerate} 
	\item $\Pi = \{0\}$ and $\ineq = \neq$, or
	\item $\ineq$ is $>$ and $x = max(\Pi)$, or
	\item $\ineq$ is $<$ and $x = 0$. 
	\end{enumerate}

\item $\Pi^{x,(*_1, \ldots, *_k)}_{\ineq} = \emptyset$ if and only if:
	\begin{enumerate} 
		\item $k=0$ and $\Pi^x_\ineq = \emptyset$, or
		\item $k >0$, $\ineq$ is $>$, $x = max(\Pi)$ and for no $*_i$, $*_i = +$, or
		\item $k >0$, $\ineq$ is $>$ and $\Pi = \{0\}$, or
		\item $k >0$, $\ineq$ is $<$, $x = 0$ and for no $*_i$, $*_i = -$, or
		\item $k >0$, $\ineq$ is $<$ and $\Pi = \{0\}$.
		\item $k >0$, $\ineq $ is $\neq$, $\Pi = \{0\}$.
	\end{enumerate}
\end{itemize}

\end{restatable}

The restricted language is appropriate for applications where a restricted set of belief values are available. 
For instance, it could be used when the beliefs in arguments are obtained from surveys using the Likert scale. 
When we consider the proof theory for constraints, the restricted language also has advantages if we want to 
harness automated reasoning with the logical statements.

\subsection{Distribution Disjunctive Normal Form}

In propositional logic, we often analyze formulae in various normal forms due to their useful properties. 
Traditional forms include the negation normal form NNF, conjunctive normal form CNF and disjunctive normal form DNF. 
Given that epistemic formulae extend propositional logic, they can also be transformed into various normal forms 
if we look at epistemic atoms as propositions. In principle, for every formula $\varphi$ we can find at least 
one formula $\varphi'$ that is in NNF, CNF or DNF and s.t. $\sat(\varphi) = \sat(\varphi')$.   
 
However, further notions can be introduced once we take the meaning of the atoms into account. 
In this section we introduce a normal form for epistemic formulae from which it is easy to 
read if and how a given formula can be satisfied. 
Let us start by observing that for every probability distribution, we can create an epistemic formula describing precisely that 
distribution. As we may remember, a probability distribution maps sets of arguments to probabilities. 
For every such set, we can create a term (i.e. a propositional formula over arguments) describing it, 
where arguments contained in the set appear as positive literals and those not in the set appear as negative literals. 
This brings us to the notion of argument complete terms: 
 
\begin{definition}
Let $\langle \xa_1,\ldots,\xa_n \rangle$ be the order of arguments in $\graph$ and $\varphi \in \terms(\graph)$
a term. Then $\varphi$ 
is $\textbf{argument complete}$ iff it is of the form $\alpha_1 \land \ldots \land \alpha_n$, 
where $\alpha_i = \xa_i$ or $\alpha_i = \neg \xa_i$. With $\argcomplete(\graph) = \{c_1, \ldots, c_j\}$ we 
denote the set of all complete terms on $\graph$, where $j = 2^n$. 
\end{definition} 

\begin{example}
Let us consider a graph with arguments $\xa$, $\xb$ and $\xc$ and ordering $\langle \xa, \xb, \xc \rangle$. 
We can create the following argument complete terms: 
$\neg \xa \land \neg \xb \land \neg \xc$, 
$\xa \land \neg \xb \land \neg \xc$, 
$\neg \xa \land  \xb \land \neg \xc$, 
$\neg \xa \land \neg \xb \land  \xc$, 
$ \xa \land  \xb \land \neg \xc$, 
$\neg \xa \land  \xb \land  \xc$, 
$ \xa \land \neg \xb \land  \xc$, 
and $ \xa \land \xb \land \xc$.
\end{example}

By using atoms containing only complete terms, we can create a formula describing precisely one distribution: 

\begin{definition}
Let $\prob \in \dist(\graph)$ be a probability distribution and $\argcomplete(\graph) = \{c_1, \ldots, c_j\}$ 
the collection of all argument complete terms for $\graph$. 
The \textbf{epistemic formula associated with} 
$\prob$ is $\varphi^\prob = \pprob(c_1) = x_1 \land \pprob(c_2) = x_2 \land \ldots \land \pprob(c_j) = x_j$, 
where $x_i = \prob(c_i)$. 
\end{definition}

\begin{restatable}{proposition}{satdistrdnf}
\label{satdistrdnf}
Let $\prob \in \dist(\graph)$ be a probability distribution and $\varphi^\prob$ its associated epistemic formula. 
Then $\{\prob\} = \sat(\varphi^\prob)$. 
\end{restatable}


\begin{example}
\label{ex:ddnftabs}
Assume we a have a graph s.t. $\{\xa, \xb\} = \nodes(\graph)$. 
Below we have tabulated some of the possible distributions for our graph and their associated formulae. 
\[
\begin{array}{c|cccc|c} 
	&	\emptyset	&	\{\xa\}	&	\{\xb\}	&	\{\xa,\xb\}	&		 \varphi^{P_i}							\\
\hline
\prob_1	&	0	&	1	&	0	&	0	&	\pprob(\neg \xa \land \neg \xb) =	0	\land	\pprob( \xa \land \neg \xb) =	1	\land	\pprob(\neg \xa \land \xb) =	0	\land	\pprob( \xa \land \xb) =	0	\\
\prob_2	&	0	&	0	&	1	&	0	&	\pprob(\neg \xa \land \neg \xb) =	0	\land	\pprob( \xa \land \neg \xb) =	0	\land	\pprob(\neg \xa \land \xb) =	1	\land	\pprob( \xa \land \xb) =	0	\\
\prob_3	&	0	&	0	&	0	&	1	&	\pprob(\neg \xa \land \neg \xb) =	0	\land	\pprob( \xa \land \neg \xb) =	0	\land	\pprob(\neg \xa \land \xb) =	0	\land	\pprob( \xa \land \xb) =	1	\\
\prob_4	&	0	&	0.5	&	0	&	0.5	&	\pprob(\neg \xa \land \neg \xb) =	0	\land	\pprob( \xa \land \neg \xb) =	0.5	\land	\pprob(\neg \xa \land \xb) =	0	\land	\pprob( \xa \land \xb) =	0.5	\\
\prob_5	&	0	&	0	&	0.5	&	0.5	&	\pprob(\neg \xa \land \neg \xb) =	0	\land	\pprob( \xa \land \neg \xb) =	0	\land	\pprob(\neg \xa \land \xb) =	0.5	\land	\pprob( \xa \land \xb) =	0.5	\\
\prob_6	&	0	&	0.5	&	0.5	&	0	&	\pprob(\neg \xa \land \neg \xb) =	0	\land	\pprob( \xa \land \neg \xb) =	0.5	\land	\pprob(\neg \xa \land \xb) =	0.5	\land	\pprob( \xa \land \xb) =	0	\\
\prob_7	&	0.1	&	0.3	&	0.2	&	0.4	&	\pprob(\neg \xa \land \neg \xb) =	0.1	\land	\pprob( \xa \land \neg \xb) =	0.3	\land	\pprob(\neg \xa \land \xb) =	0.2	\land	\pprob( \xa \land \xb) =	0.4	\\
\end{array}
\] 
\end{example}

Consequently, for every epistemic formula $\varphi$, we can create a semantically equivalent formula $\varphi'$ that 
is built from the formulae associated with the distributions satisfying $\varphi$. 
We refer to this new formula 
as the distribution disjunctive normal form. Given the fact that an epistemic formula 
can potentially be satisfied by infinitely many distributions, we only consider this form in the context of restricted reasoning. 
 
\begin{definition}  
Let $\Pi$ be a reasonable restricted value set, 
$\psi \in \eformulae(\graph, \Pi)$ be a restricted epistemic formula and $\{\prob_1,\ldots,\prob_n\} = \sat(\psi, \Pi)$ 
the set of distributions satisfying $\psi$ under $\Pi$. 
The \textbf{distribution disjunctive normal form} (abbreviated DDNF) of $\psi$ is $\bot$ iff 
$\sat(\psi, \Pi) = \emptyset$, and  $\varphi^{\prob_1} \lor \varphi^{\prob_2} \ldots \lor \varphi^{\prob_n}$ otherwise, 
where $\varphi^{\prob_i}$ is the epistemic formula associated with $\prob_i$.   
\end{definition} 

\begin{restatable}{proposition}{satdistrdnfform}
\label{satdistrdnfform}
Let $\Pi$ be a reasonable restricted value set, 
$\psi \in \eformulae(\graph, \Pi)$ be a restricted epistemic formula and $\varphi$ its distribution 
disjunctive normal form. Then $\sat(\psi, \Pi) = \sat(\varphi, \Pi)$. 
\end{restatable}
 
\begin{example}
\label{ex:ddnf1}
Let us continue Example \ref{ex:ddnftabs} and assume we have an epistemic atom $\pprob(\xa \lor \xb) > 0.5$ 
and a reasonable restricted value set $\Pi = \{0, 0.5, 1\}$.
Distributions $\prob_1$ to $\prob_6$ are the restricted satisfying distributions of our formula
and the DDNF associated with $\pprob(\xa \lor \xb) > 0.5$ is 
$\varphi^{\prob_1} \lor \varphi^{\prob_2} \lor \varphi^{\prob_3} \lor \varphi^{\prob_4} \lor 
\varphi^{\prob_5} \lor \varphi^{\prob_6}$. 
%
\end{example}

We will harness the DDNF when we provide correctness results for the consequence relation for the epistemic 
language in Section \ref{section:consequencerelation}.

\section{Reasoning with the Epistemic Language}
\label{section:ReasoningConstraints}

Previously, we have considered the syntax and semantics of our epistemic language. 
However, we have not yet explained how two 
epistemic formulae can be related based on their satisfying distributions, or what can be logically inferred from a given 
formula. We would like to address this here by first introducing 
the notion of epistemic entailment and then by providing a consequence relation, with the latter primarily focused on 
the restricted language. 
From now on, unless stated otherwise, we will assume that the 
argumentation framework we are dealing with is finite and nonempty (i.e. the set of arguments in the graph 
is finite and nonempty). 

\subsection{Epistemic Entailment}

Let us start with the unrestricted epistemic entailment relation, which is defined in the following manner: 

\begin{definition} 
Let $\{\phi_1,\ldots,\phi_n\} \subseteq  \eformulae(\graph)$ be a set of epistemic formulae, and 
$\psi \in \eformulae(\graph)$ be an epistemic formula.
The {\bf epistemic entailment relation}, denoted $\VDash$, is defined as follows.
\[
\{\phi_1,\ldots,\phi_n\}\VDash\psi
\mbox{ iff } \sat(\{\phi_1,\ldots,\phi_n\}) \subseteq \sat(\psi)
\]
\end{definition}

\begin{example}
The following are some instances of epistemic entailment.
\begin{itemize}
\item $\{ \pprob(\xa) < 0.2 \} \VDash \pprob(\xa) < 0.3$ 
\item $\{ \pprob(\xa) < 0.2 \} \VDash \pprob(\xa \land \xb) < 0.2$ 
\item $\{ \pprob(\xa) < 0.9, \pprob(\xa) > 0.7 \} \VDash \pprob(\xa) \geq 0.7 \land \neg (\pprob(\xa) > 0.9)$ 
\end{itemize}
\end{example}
 
Let us now focus on reasoning in the restricted scenario, which can be defined 
similarly to the standard epistemic entailment through the use of restricted satisfying distributions: 
 
\begin{definition} 
Let $\Pi$ be a restricted value set,
$\{\phi_1,\ldots,\phi_n\} \subseteq  \eformulae(\graph,\Pi)$ a set of epistemic formulae, and 
$\psi \in \eformulae(\graph)$ an epistemic formula.
The {\bf restricted epistemic entailment relation} w.r.t. $\Pi$, denoted $\VDash_{\Pi}$, is defined as follows.
\[
\{\phi_1,\ldots,\phi_n\}\VDash_{\Pi} \psi
\mbox{ iff } \sat(\{\phi_1,\ldots,\phi_n\},\Pi) \subseteq \sat(\psi,\Pi)
\]
\end{definition}

\begin{example}
Consider $\Pi = \{ 0, 0.25, 0.5, 0.75, 1 \}$
and restricted epistemic formulae $\pprob(\xa) + \pprob(\neg \xb) \leq 1$ and $\pprob(\xa) + \pprob(\neg \xb) \leq 0.75$. 
It holds that 
\[
\{  \pprob(\xa) + \pprob(\neg \xb) \leq 0.75 \} \VDash_{\Pi} \pprob(\xa) + \pprob(\neg \xb) \leq 1
\]
\end{example} 

Let us now discuss how the restricted satisfying distributions and the restricted entailment are related to the 
unrestricted versions. 
First of all, by Definition \ref{def:restricteddist}, we can observe that every restricted satisfying distribution 
for an epistemic formula is also a satisfying distribution. Thus, we can easily show that epistemic entailment 
implies restricted entailment:

\begin{restatable}{proposition}{temp}
\label{temp1}
Let $\Pi$ be a restricted value set, $\Phi \subseteq  \eformulae(\graph,\Pi)$ a set of epistemic formulae, and 
$\psi \in \eformulae(\graph)$ an epistemic formula.  
If $\Phi \VDash \psi$ then $\Phi \VDash_{\Pi} \psi$. 
\end{restatable}

In principle, we can observe that a \enquote{less} restricted entailment implies a \enquote{more} restricted one: 

\begin{restatable}{proposition}{tempmore}
\label{temp2}
Let $\Pi_1 \subseteq \Pi_2$ be restricted value sets, $\Phi \subseteq  \eformulae(\graph,\Pi_1)$ a set of epistemic formulae, and 
$\psi \in \eformulae(\graph)$ an epistemic formula.  
If $\Phi \VDash_{\Pi_2} \psi$ then $\Phi \VDash_{\Pi_1} \psi$. 
\end{restatable} 
 
Note, it does not necessarily hold that if one formula follows from another in a restricted manner, then it also follows in the 
unrestricted one as illustrated below:

\begin{example}
Consider two formulae $\varphi_1: \pprob(\xa) \neq 0.5$ and $\varphi_2: \pprob(\xa) = 0 \lor \pprob(\xa) = 1$ 
and a reasonable restricted set $\Pi = \{0, 0.5, 1\}$. 
We can observe that  
$\sat(\varphi_1, \Pi) = \sat(\varphi_2, \Pi)$ and therefore $\{\varphi_1\} \VDash_{\Pi} \varphi_2$. 
However, in the unrestricted case we can consider a probability distribution $\prob$
s.t. $\prob(\xa) = 0.9$ in order to show that $\sat(\varphi_1) \nsubseteq \sat(\varphi_2)$. 
We can observe that this issue would have been bypassed if, instead of $\Pi$, we considered 
the set $\Pi_2 = \{0, 0.25, 0.5, 0.75, 1\}$, for which $\{\varphi_1\} \nVDash_{\Pi_2} \varphi_2$.
Consequently, although restricted entailment does not in general imply unrestricted entailment, 
for a given set of formulae it is possible to find such a $\Pi$ for which this property holds. 
\end{example}

The reason that an inference from the restricted entailment relation is not necessarily an inference from the 
unrestricted entailment relation is that the restricted case contains more information. 
The set $\Pi$ is extra information that restricts the possible assignments for the probability distribution. 
Indeed, it could equivalently be represented as a set of formulae that could be added to the left-hand side of the 
unrestricted entailment relation. 
This is analogous to the use of explicit formulae on the domain in order to formalise the closed world assumption in 
predicate logic \cite{Reiter1978}.


\subsection{Consequence Relation}
 \label{section:consequencerelation}

In order to provide a proof theoretic counterpart to the entailment relation, we present a consequence relation in this subsection. 
For this, we will focus on the restricted language. 

The advantage of having a consequence relation is that we can now obtain inferences from a set of epistemic formulae. This means 
we can for instance determine if one constraint is implied by another, and whether there is redundancy in a set of constraints (i.e. 
whether the set is not minimal). More generally, we will see that both the entailment and consequence relations are important in 
examining properties of epistemic graphs as covered in Section \ref{sec:epistemicgraphs}. 

Before we present the epistemic proof system, we introduce 
some subsidiary definitions associated with the arithmetic nature of operational formulae. Although 
they are not limited to restricted formulae only, we prefer to have them at hand due to the fact that we will be using 
them in our epistemic proof system:

\begin{definition}
Let $f_1 \formis  \pprob(\alpha_1) *_1 \pprob(\alpha_2) *_2 \ldots  *_{m-1} \pprob(\alpha_m)$, 
and $f_2 \formis \pprob(\beta_1) \star_1 \pprob(\beta_2) \star_2 \ldots  \star_{l-1} \pprob(\beta_l)$, 
where $\alpha_i, \beta_i \in \terms(\graph)$ 
and $*_i, \star_i \in \{+, -\}$, be operational formulae. 
 $f_1 \succeq_{su} f_2$ denotes the {\bf subject inequality relation} that holds when $f_2$ is obtained from $f_1$ by logical 
weakening of an element $\pprob(\alpha_i)$ of $f_1$ to $\pprob(\alpha'_i)$ where $\{\alpha_i\}\vdash\alpha'_i$, and all other 
elements are the same in $f_1$ and $f_2$. Additionally:
\begin{itemize}
\item with $f_1 \succeq_{su}^+ f_2$ we denote the case where $f_1 \succeq_{su} f_2$ and either $i=1$ or $*_{i-1} = +$
\item with $f_1 \succeq_{su}^- f_2$ we denote the case where $f_1 \succeq_{su} f_2$, $i > 1$ and $*_{i-1} = -$  
\end{itemize} 

Let $\varphi_1 = f_1 \ineq x$ and $\varphi_2 = f_2 \ineq x$, where $\ineq \in \operators$ and 
$x \in [0,1]$, be epistemic atoms. We say that $\varphi_1 \succeq_{su} \varphi_2$ iff $f_1 \succeq_{su} f_2$ and
 with $\varphi_1 \succeq_{su}^+ \varphi_2$ (resp. $\varphi_1 \succeq_{su}^- \varphi_2$) we denote the case where 
$f_1 \succeq_{su}^+ f_2$ (resp. $f_1 \succeq_{su}^- f_2$).
\end{definition}   

\begin{example}
The following illustrate the subject inequality relation.
\[
\begin{array}{c}
\pprob(\xb) - \pprob(\xa \land \xc)  > x  \succeq^+_{su}   \pprob(\xb \vee \xd) -\pprob(\xa \land \xc) >x  \\
 \pprob(\xb) - \pprob(\xa \land \xc \land \xe)  < x   \succeq^-_{su} \pprob(\xb) - \pprob(\xa \land \xc) < x \\
\end{array}
\]
\end{example}
 
\begin{restatable}{proposition}{propformsat}
\label{prop:formsat}
For epistemic atoms $\varphi_1= f_1 \ineq x$ and $\varphi_2 = f_2 \ineq x$ in $\eformulae(\graph,\Pi)$, 
the following hold: 
\begin{itemize} 
\item if $\varphi_1  \succeq_{su}^+ \varphi_2$ 
and $\ineq \in \{> ,\geq\}$  
then $\sat(\varphi_1) \subseteq \sat(\varphi_2)$, and if $\ineq \in \{<, \leq\}$, 
then $\sat(\varphi_2) \subseteq \sat(\varphi_1)$
  
\item if $\varphi_1  \succeq_{su}^- \varphi_2$
and $\ineq \in \{< ,\leq\}$  
then $\sat(\varphi_1) \subseteq \sat(\varphi_2)$, 
and if $\ineq \in \{>, \geq\}$, 
then $\sat(\varphi_2) \subseteq \sat(\varphi_1)$

\end{itemize}  
\end{restatable}

We can now introduce the proof rule system for the epistemic formulae. 
The basic rules grasp the primitive properties of probabilities, i.e. that any probability is in the unit interval, and that probabilities 
of $\top$ and $\bot$ are respectively $1$ and $0$. 
The probabilistic rule allows us to express the probability of conjunction (disjunction) of two argument terms 
through the probabilities of these terms.  
The subject rules capture the behaviour of epistemic formulae that are connected 
through the subject inequality relation. 
The enumeration rules allow us to transform any inequality into a formula using only equality under the given restricted set $\Pi$. 
However, given the results of Proposition \ref{restrnonempty}, in some cases it can happen that the appropriate subsets of $\Pi$ are 
empty. Thus, wherever applicable, we make it clear that the resulting formula 
should be seen as falsity. Finally, the propositional rules capture how the reasoning extends classical propositional logic. 
  
\begin{definition}
Let $\ineq \in \operators$ and $x \in [0,1]$. 
Let $\Pi$ be a restricted value set, 
and $\Pi^{x,(*_1,\ldots,*_{m-1})}_{\ineq}$ be the combination set of $\Pi$ obtained according to the value $x$, 
relationship $\ineq$ and the 
sequence arithmetic operations of arithmetic operations $(*_1,\ldots,*_{m-1})$.  
Also let $f_1 \formis  \pprob(\alpha_1) *_1 \pprob(\alpha_2) *_2 \ldots *_{k-1} \pprob(\alpha_k)$
and $f_2 \formis  \pprob(\beta_1) \star_1 \pprob(\beta_2) \star_2 \ldots \star_{l-1} \pprob(\beta_l)$, 
where $k,l \geq 1$, $\alpha_i,\beta_i \in \terms(\graph)$ and $\star_j,*_i \in \{+ ,-\}$ be  operational formulae.  
The {\bf restricted epistemic consequence relation}, denoted $\Vdash_\Pi$, is defined as follows, where $\vdash$ is 
propositional consequence relation, 
$\Phi \subseteq \eformulae(\graph,\Pi)$, and $\phi,\psi \in \eformulae(\graph,\Pi)$.

The following proof rules are the {\bf basic rules}:  
\begin{align*}
(B1) \hspace{1mm} \Phi\Vdash_\Pi \pprob(\alpha) \geq 0 & \mbox{ iff } \Phi\Vdash_\Pi \top   
& (B2)  \hspace{1mm} \Phi\Vdash_\Pi \pprob(\alpha) \leq 1 & \mbox{ iff } \Phi\Vdash_\Pi  \top \\
 (B3)  \hspace{1mm} \Phi\Vdash_\Pi \pprob(\top) = 1   & \mbox{ iff } \Phi\Vdash_\Pi  \top  
& (B4)  \hspace{1mm}   \Phi\Vdash_\Pi \pprob(\bot) = 0  &  \mbox{ iff } \Phi\Vdash_\Pi  \top    
\end{align*}

The following rule is the {\bf probabilistic rule}:
\[
\begin{array}{rl}   

(PR1) & 
\Phi\Vdash_\Pi \pprob(\alpha \lor \beta)  - \pprob(\alpha) -  \pprob(\beta) +  \pprob(\alpha \land \beta) = 0 \\
 
\end{array}
\]   

The following proof rules are the {\bf subject rules}. 
\[
\begin{array}{rl} 

(S1) & 
\Phi\Vdash_\Pi f_1 > x
\mbox{ and } f_1 \succeq_{su}^+ f_2
\mbox{ implies } \Phi\Vdash_\Pi f_2 > x \\
 
(S2) & 
\Phi\Vdash_\Pi f_1 \geq x
\mbox{ and } f_1 \succeq_{su}^+ f_2
\mbox{ implies } \Phi\Vdash_\Pi f_2 \geq  x \\

(S3) & 
\Phi\Vdash_\Pi f_1 < x
\mbox{ and } f_1 \succeq_{su}^- f_2
\mbox{ implies } \Phi\Vdash_\Pi f_2 < x \\
 
(S4) & 
\Phi\Vdash_\Pi f_1 \leq x
\mbox{ and } f_1 \succeq_{su}^- f_2
\mbox{ implies } \Phi\Vdash_\Pi f_2 \leq  x \\

(S5) & 
\Phi\Vdash_\Pi f_2 < x
\mbox{ and } f_1 \succeq_{su}^+ f_2
\mbox{ implies } \Phi\Vdash_\Pi f_1 < x \\
 
(S6) & 
\Phi\Vdash_\Pi f_2 \leq x
\mbox{ and } f_1 \succeq_{su}^+ f_2
\mbox{ implies } \Phi\Vdash_\Pi f_1 \leq x \\

(S7) & 
\Phi\Vdash_\Pi f_2 > x
\mbox{ and } f_1 \succeq_{su}^- f_2
\mbox{ implies } \Phi\Vdash_\Pi f_1 > x \\
 
 (S8) & 
\Phi\Vdash_\Pi f_2 \geq x
\mbox{ and } f_1 \succeq_{su}^- f_2
\mbox{ implies } \Phi\Vdash_\Pi f_1 \geq  x \\

\end{array}
\]  

The next rules are the {\bf enumeration rules}.
\begin{align*}
(E1)  \hspace{1mm}   \Phi\Vdash_\Pi f_1 \ineq x & \mbox{ iff } (\Phi\Vdash_\Pi \bigvee_{(v_1,\ldots,v_k) \in \Pi^{x, \arop(f_1)}_{\ineq}} (\pprob(\alpha_1) = v_1 \land \pprob(\alpha_2) = v_2 \land \ldots \land \pprob(\alpha_k)= v_k) \\
&\mbox{ if } \Pi^{x, \arop(f_1)}_\ineq \neq \emptyset  \mbox{ and } \Phi \Vdash_\Pi \bot \mbox { otherwise}) \\
(E2)  \hspace{1mm}  \Phi\Vdash_\Pi f_1 > x & \mbox{ iff } \Phi\Vdash_\Pi \neg(\bigvee_{(v_1,\ldots,v_k) \in \Pi^{x, \arop(f_1)}_{\leq}} (\pprob(\alpha_1) = v_1 \land \pprob(\alpha_2) = v_2 \land \ldots \land \pprob(\alpha_k)= v_k)) \\
 (E3)\hspace{1mm}\Phi\Vdash_\Pi f_1 \geq x & \mbox{ iff } (\Phi\Vdash_\Pi \neg(\bigvee_{(v_1,\ldots,v_k) \in \Pi^{x, \arop(f_1)}_{<}} (\pprob(\alpha_1) = v_1 \land \pprob(\alpha_2) = v_2 \land \ldots \land \pprob(\alpha_k)= v_k))\\
&\mbox{ if } \Pi^{x, \arop(f_1)}_{<} \neq \emptyset \mbox{ and } \Phi \Vdash_\Pi \neg(\bot) \mbox { otherwise}) \\
(E4)  \hspace{1mm}  \Phi\Vdash_\Pi f_1 < x & \mbox{ iff } \Phi\Vdash_\Pi \neg(\bigvee_{(v_1,\ldots,v_k) \in \Pi^{x, \arop(f_1)}_{\geq}} (\pprob(\alpha_1) = v_1 \land \pprob(\alpha_2) = v_2 \land \ldots \land \pprob(\alpha_k)= v_k)) \\
 (E5)  \hspace{1mm}  \Phi\Vdash_\Pi f_1 \leq x & \mbox{ iff } (\Phi\Vdash_\Pi \neg(\bigvee_{(v_1,\ldots,v_k) \in \Pi^{x,\arop(f_1)}_{>}} (\pprob(\alpha_1) = v_1 \land \pprob(\alpha_2) = v_2 \land \ldots \land \pprob(\alpha_k)= v_k))\\
&\mbox{ if } \Pi^{x, \arop(f_1)}_{>} \neq \emptyset \mbox{ and } \Phi \Vdash_\Pi \neg(\bot) \mbox { otherwise})  \\ 
\end{align*}  

The following proof rules are the {\bf propositional rules}.
\[
\begin{array}{rl}
(P1) &  \Phi\Vdash_\Pi\phi_1 \mbox{ and } \ldots \mbox{ and  }\Phi\Vdash_\Pi \phi_n \mbox{ and } n \geq 1 \mbox{ and } \{\phi_1,\ldots,\phi_n\}\vdash\psi \mbox{ implies } \Phi\Vdash_\Pi \psi \\ 
(P2) & \mbox{ if } \Phi\vdash \varphi   \mbox{ then } \Phi\Vdash_\Pi \varphi 
\end{array}
\]
\end{definition}

\begin{example}
For $\Pi = \{0, 0.2, 0.4, 0.6, 0.8, 1.0\}$, the following illustrate the restricted epistemic consequence relation. 
\begin{itemize}
\item $\{ \pprob(\xa) + \pprob(\xb) \leq 1, \pprob(\xa) - \pprob(\xb) \geq 1  \} \Vdash_\Pi  \pprob(\xa) + \pprob(\xb) = 1$
\item $\{ \pprob(\xa) > 0.8, \pprob(\xa) > 0.5 \rightarrow \pprob(\xb) > 0.5 \} \Vdash_\Pi  \pprob(\xb) > 0.5$
\item $\{ \pprob(\xc) > 0.5 \rightarrow \pprob(\xb) > 0.5 \land \pprob(\xa) > 0.5, 
												\pprob(\xb) > 0.5 \rightarrow \pprob(\xa) \leq 0.5 \} \Vdash_\Pi  \bot$
\item $\{ \pprob(\xa) > 0.6 \} \Vdash_\Pi  \pprob(\xa) = 0.8 \vee \pprob(\xa) = 1$

\end{itemize}
\end{example} 

We can use the epistemic consequence relation to infer relationships between unconnected nodes as illustrated next.

\begin{example}
\label{ex:reasoning:2}
For the following graph, consider 
the formulae $\con = \{ \pprob(\xc) > 0.5 \rightarrow \pprob(\xb) > 0.5, \pprob(\xb) > 0.5 \rightarrow \pprob(\xa) \leq 0.5 \}$. 
From $\con$, we can infer $\pprob(\xc) > 0.5 \rightarrow \pprob(\xa) \leq 0.5$. 
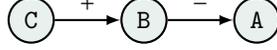
\begin{center}
\begin{tikzpicture}[->,>=latex,thick, 
main node/.style={shape=rounded rectangle,fill=darkgreen!10,draw,minimum size = 0.6cm,font=\normalsize\bfseries} ]
\node[main node] (a) at (3,0) {$\xa$};
\node[main node] (b) at (1.5,0) {$\xb$};
\node[main node] (c) at (0,0) {$\xc$};
\path (c) edge[] node[above] {$+$} (b);
\path (b) edge[] node[above] {$-$} (a);
\end{tikzpicture}
\end{center}
\end{example}

The following is a correctness result showing that the restricted epistemic consequence relation is sound 
with respect to the restricted epistemic entailment relation. 

\begin{restatable}{proposition}{restrictedvalsound}
\label{prop:restrictedvalsound}
Let $\Pi$ be  a restricted value set. 
For $\Phi \subseteq \eformulae(\graph,\Pi)$, and $\psi \in \eformulae(\graph,\Pi)$, if 
$\Phi \Vdash_\Pi  \psi$ then $\Phi\VDash_{\Pi}\psi$. 
\end{restatable}

However, as it is often the case, the completeness is somewhat more difficult to show. We may recall
that for every probability distribution, we can create an epistemic formula describing precisely that 
distribution. From the disjunction of such formulae, we have created the distribution disjunctive normal form (DDNF) of 
every formula, the models of which were identical with the original formula.  
The challenge of the completeness proof is therefore 
to show that the DDNF of a given formula is equivalent to it 
not only semantically, but also syntactically. 

This can be achieved by first transforming every term into a disjunction of argument 
complete terms, then separating this epistemic atom into further atoms s.t. every one of them contains precisely 
one complete term through the use of probabilistic rules. The probabilities of the complete terms that are not present 
yet can be inferred from the ones that are, and we can use all of this to show the 
syntactical equivalence of the epistemic formula and its DDNF: 
 
\begin{restatable}{proposition}{valdistributiondnf}
\label{valdistributiondnf}
Let $\Pi$ be a reasonable restricted value set, 
$\Phi \subseteq \eformulae(\graph,\Pi)$ a set of epistemic formulae and $\psi \in \eformulae(\graph,\Pi)$ 
an epistemic formula. Then $\Phi \Vdash_\Pi  \psi$ iff $\Phi \Vdash_\Pi  \varphi$, 
where $\varphi$ is the distribution disjunctive 
normal form of $\psi$  
\end{restatable} 

The ability to transform any formula into its DDNF both semantically and syntactically, along 
with the previous soundness results, brings us to the final correctness result for our system:

\begin{restatable}{proposition}{restrictedvalsoundcomp}
\label{prop:restrictedvalsoundcomp}
Let $\Pi$ be  a restricted value set. 
For $\Phi \subseteq \eformulae(\graph,\Pi)$, and $\psi \in \eformulae(\graph,\Pi)$, 
$\Phi \Vdash_\Pi  \psi$ iff $\Phi\VDash_{\Pi}\psi$. 
\end{restatable}

In addition,  the following property can be shown, which indicates that we can develop algorithms for automated reasoning based on proof by contradiction:

\begin{restatable}{proposition}{negval}
\label{prop:negval}
Let $\Pi$ be a reasonable restricted value set. 
For $\Phi \subseteq \eformulae(\graph,\Pi)$ and $\psi \in \eformulae(\graph,\Pi)$,
$\Phi\Vdash_\Pi  \psi \mbox{ iff } \Phi \cup \{\neg \psi\} \Vdash_{\Pi} \bot$. 
\end{restatable} 

We can also observe that for a finite set of rational 
numbers from the unit interval $\Pi$, representing and reasoning with the restricted epistemic  
language w.r.t. $\Pi$ is equivalent to propositional logic. We show this via the next two lemmas. 

\begin{restatable}{lemma}{lemmavformula}
\label{lemma:vformula}
Let $\Pi$ be a reasonable restricted value set. 
There is a set of propositional formulae $\Omega$ with $\Lambda \subseteq \Omega$, 
and there is a function $f: \eformulae(\graph,\Pi) \rightarrow \Omega$ 
s.t. for each $\{\phi_1,\ldots,\phi_n\} \subseteq \eformulae(\graph,\Pi)$,
and for each $\psi \in \eformulae(\graph,\Pi)$, 
\[
\{\phi_1,\ldots,\phi_n\} \Vdash_\Pi  \psi \mbox{ iff } \{f(\phi_1),\ldots,f(\phi_n)\} \cup \Lambda \vdash f(\psi)
\]
\end{restatable}

\begin{restatable}{lemma}{restrictedequivalencetwo}
\label{lemma:2}
Let $\Omega$ be a propositional language composed from a set of atoms and the usual definitions for the Boolean connectives.  
There is a restricted epistemic language $\eformulae(\graph,\Pi)$ where $\Pi = \{ 0,1 \}$ 
and there is a function $g: \Omega \rightarrow \eformulae(\graph,\Pi)$ 
s.t. for each set of propositional formulae $\{\alpha_1,\ldots,\alpha_n\} \subseteq \Omega$  
and for each propositional formula $\beta \in \Omega$, 
\[
\{\alpha_1,\ldots,\alpha_n\} \vdash \beta \mbox{ iff } \{g(\alpha_1),\ldots,g(\alpha_n)\} \Vdash_\Pi g(\beta)
\]
\end{restatable}

From Lemma \ref{lemma:vformula} and Lemma \ref{lemma:2}, we obtain the following result. 
This means that whatever can be represented or inferred in the restricted epistemic language can be represented 
or inferred in the classical propositional language  and vice versa. 

\begin{restatable}{proposition}{propclassequivalenttwo}
The restricted epistemic language with the restricted epistemic consequence relation is equivalent to the classical 
propositional language with the classical propositional consequence relation. 
\end{restatable}
 
The restricted language (where the values for the inequalities are restricted to a finite set of values from the unit interval) allows for inequalities to be rewritten as a disjunction of equalities. This then allows for an epistemic consequence relation to be defined as a conservative extension of the classical propositional consequence relation. The advantage of this restricted version is that it can be easily implemented using \emph{constraint satisfaction} techniques \cite{Dechter2003,Rossi:2006,Tsang1993}. These allow for a declarative representation of constraints and provide sophisticated methods for determining solutions. For some applications, such as user modelling in persuasion dialogues, having a restricted set of values (such as corresponding to a Likert scale) would offer a sufficiently rich framework.  

\subsection{Closure}  

Last, but not the least, we define the notion of an epistemic closure, which will become particularly useful in the analysis 
of relation coverage and labelings in Sections \ref{sec:RelationCoverage} and \ref{sec:consistentlabel}. To put 
it simply, closure produces the set of 
all formulae derivable from a given set: 

\begin{definition}
Let $\Phi \subseteq \eformulae(\graph)$.
The {\bf epistemic closure function} is defined as follows.
\[
\closure(\Phi) = \{ \psi \mid \Phi\VDash \psi \}
\] 
\end{definition}

We can observe that closure 
can produce infinitely many formulae that, depending on how we intend to use it, can be seen as redundant. 
For example, from a formula $\pprob(\xa) > 0.5$ we can derive $\pprob(\xa)>y$ for every real number $y \in [0,0.5]$. 
Consequently, in many cases it makes sense to focus on closure w.r.t. a given reasonable restricted set of values $\Pi$:

\begin{definition}
Let $\Pi$ be a reasonable restricted value set, 
and let $\Phi \subseteq \eformulae(\graph,\Pi)$.
The {\bf restricted epistemic closure function} is defined as follows.
\[
\closure(\Phi,\Pi) = \{ \psi \mid \Phi\VDash_\Pi \psi \}
\] 
\end{definition}

Given the soundness and completeness results for our proof systems, we can observe that 
closure can also be defined using $\Vdash_\Pi$. 
The closure function is monotonic on both of its arguments (i.e. if $\Phi\subseteq\Phi'$ and $\Pi\subseteq\Pi'$, then 
$\closure(\Phi,\Pi) \subseteq \closure(\Phi,\Pi)$). 

\begin{example}
Let us consider the reasonable restricted value set $\Pi = \{0, 0.1, 0.2, \ldots, 0.9, 1\}$ and the 
set of formulae $\Phi = \{\pprob(\xa) < 0.5, (\pprob(\xb) > 0.5 \land \pprob(\xa) > 0.4) \rightarrow 
\pprob(\xc) > 0.6, \pprob(\xc) = 1 \rightarrow \pprob(\xb) = 0.9\}$. 
We can observe that $\Phi  \VDash_\Pi \pprob(\xa) \leq x$ for $x \in \{0.5, 0.6, \ldots, 1\}$, thus these formulae 
belong to the (both restricted and standard) closure of $\Phi$. On the other hand, the formula $\pprob(\xa) = 0.7$ does not. 
The formula $\pprob(\xa) < 0.5 \land (\pprob(\xb) \leq 0.5 \lor \pprob(\xc) > 0.6) \land (\pprob(\xc) < 1 \lor 
\pprob(\xb) = 0.9)$ also belongs to the closure. The formula $\pprob(\xb) = 0.8 \land \pprob(\xc) = 0.2$ 
does not. 
\end{example}

 We will use the closure function in the next section when we consider properties of epistemic graphs in terms of their constraints.

\section{Epistemic Graphs}
\label{sec:epistemicgraphs} 

In the introduction, we have discussed the value of being able to model beliefs in arguments, various types of relations 
between arguments, context--sensitivity, and more. Our proposal, capable of meeting the postulated requirements, 
comes in the form of epistemic graphs, which are labelled graphs 
equipped with particular formulae specifying the beliefs in arguments and the interplay between them. 
In this section we formalize the idea of epistemic graphs: we explain how constraints can be specified and interpreted, 
define epistemic semantics and provide an example of how our proposal can be used in practical applications. 
  
Our aim in this section is to provide a general representation formalism for epistemic probabilistic argumentation. Although we will, at times, discuss reasoning methods and introduce concepts that may help in implementing a working system based on our formalism, our focus will be on the conceptual level. In general, reasoning with epistemic constraints can be seen as a special case of \emph{constraint satisfaction problems} \cite{Dechter2003,Rossi:2006,Tsang1993} (CSP, as mentioned earlier) and CSP software could be used to implement our proposal. We will point to general concepts from the CSP literature when appropriate, but leave a deeper discussion of the implementation issues for future work. 
  
An epistemic graph is, to put it simply, a labelled graph equipped with a set of epistemic constraints,   
which are defined as epistemic formulae that contain at least one argument. This restriction
is to exclude constraints that operate only on truth values and are simply redundant. Nevertheless, we note 
that it is optional and can be lifted if desired. 
 
\begin{definition} 
An {\bf epistemic constraint} is an epistemic formula $\psi \in \eformulae(\graph)$ s.t. $\fargs(\psi) \neq \emptyset$.  
An {\bf epistemic graph} is a tuple $(\graph,\lab,\con)$ where $(\graph,\lab)$ is a labelled graph, and 
$\con \subseteq \eformulae(\graph)$ is a set of epistemic constraints associated with the graph.  
\end{definition} 

We will say that an epistemic graph is consistent iff its set of constraints is consistent.  
Please note that the graph (and its labelling, which we will discuss in Section \ref{sec:consistentlabel}) is not necessarily induced 
by the constraints and therefore it contains additional information.  
The actual direction of the edges in the graph is also not derivable from $\con$. For example, 
if we had two arguments $\xa$ and $\xb$ connected by an edge, 
a constraint of the form $\pprob(\xa) < 0.5 \lor \pprob(\xb) < 0.5$ would not tell us the direction of this edge.
While for the sake of readability, we may use implications that reflect the directions of the edges, the syntactical features of the constraints should in general not be treated as cues for the graph structure.
The constraints may also involve unrelated arguments, similarly as in \cite{CosteMarquisDM06}. 
We will now consider some examples of epistemic graphs. 

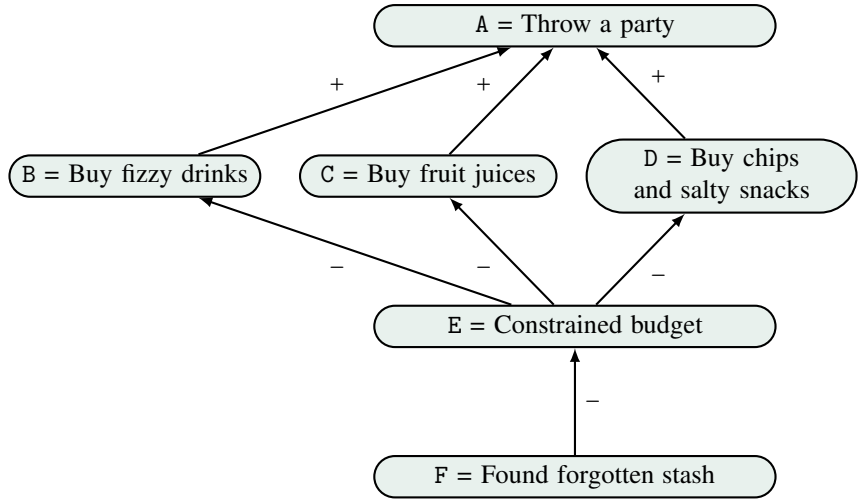
\begin{figure}[!ht]
\centering  
 \begin{tikzpicture}[->,>=latex,thick,auto, arg/.style={draw,text centered,shape = rounded rectangle,fill=darkgreen!10}]
\node[arg] (a) [text width=5cm] at (1.85,4) {$\xa$ = Throw a party}; 
\node[arg] (b) [text width=3cm]at (-4,2) {$\xb$ = Buy fizzy drinks};
\node[arg] (c) [text width=3.075cm]at (-0.1,2) {$\xc$ = Buy fruit juices};
\node[arg] (d) [text width=3.015cm]at (3.8,2) {$\xd$ = Buy chips and salty snacks}; 

\node[arg] (e) [text width=5cm] at (1.85,0) {$\xe$ = Constrained budget};  
\node[arg] (f) [text width=5cm] at (1.85,-2) {$\xf$ = Found forgotten stash}; 

\path (b) edge node{$+$} (a)
	 (c) edge node{$+$}  (a)
	 (d) edge node[swap]{$+$}  (a) 
	  (e) edge node{$-$} (b)
	 (e) edge node{$-$}  (c)
	 (e) edge node[swap]{$-$}  (d) 
	 (f) edge node[swap]{$-$}  (e);  
\end{tikzpicture} 
\caption{Party organization graph. The $+$ labels denote support and $-$ denote attack. }
\label{fig:party}
\end{figure}

\begin{example} 
\label{ex:party}
Let us consider an example in which Mary and Jane are organizing a small party at the student dormitory. Although 
the guests will bring some beer, Mary and Jane need to buy 
some non--alcoholic drinks and snacks. This can be represented with arguments $\xa$, $\xb$, $\xc$ and $\xd$ 
as seen in Figure \ref{fig:party} and expressed with the following constraints:

\begin{itemize}
\item $\varphi_1 \formis (\pprob(\xb) >0.5 \lor \pprob(\xc) >0.5) \land  \pprob(\xd) >0.5  \rightarrow \pprob(\xa) > 0.5$
\item $\varphi_2 \formis (\pprob(\xb) <0.5 \land \pprob(\xc) <0.5) \lor   \pprob(\xd) <0.5  \rightarrow \pprob(\xa)  <0.5$
\end{itemize}

We can observe that $\xb$, $\xc$ and $\xd$ are supporters of $\xa$ in the sense that the acceptance of $\xa$ requires 
the acceptance of $\xd$ and $\xb$ or $\xc$. 

Let us assume that Mary and Jane realize that their budget is somewhat limited. 
We could create a constraint stating
that at least one of the items has to be rejected:

\begin{itemize}
\item $\varphi_3 \formis \pprob(\xb) <0.5 \lor \pprob(\xc) <0.5 \lor \pprob(\xd) < 0.5$
\end{itemize}

However, instead of this, we can also decide to represent the budget limitations as an argument $\xe$ and replace 
$\varphi_3$ with $\varphi'_3$: 
\begin{itemize} 
\item $\varphi'_3 \formis \pprob(\xe) >0.5 \rightarrow \pprob(\xb) <0.5 \lor \pprob(\xc) <0.5 \lor \pprob(\xd) < 0.5$
\item $\varphi'_4 \formis \pprob(\xe) >0.5$
\end{itemize}

We can observe that in this case, the relation between $\xe$ and $\xb$, $\xc$ and $\xd$ is more attacking, in the sense that 
acceptance of $\xe$ leads to the rejection of at least one of $\xb$, $\xc$ and $\xd$. 

Although the former solution is more concise, the latter also has its benefits. Let us assume that Mary now finds
some spare money in her backpack and they can afford to buy all of the items. Thus, we add argument $\xf$, 
and the constraint $\varphi'_4$ will 
need to be replaced:
\begin{itemize} 
\item $\varphi''_4 = \pprob(\xf) >0.5 \rightarrow \pprob(\xe) <0.5$
\item $\varphi''_5 = \pprob(\xf) >0.5$
\end{itemize}
Clearly, the relation between $\xf$ and $\xe$ is conflicting. 
\end{example}

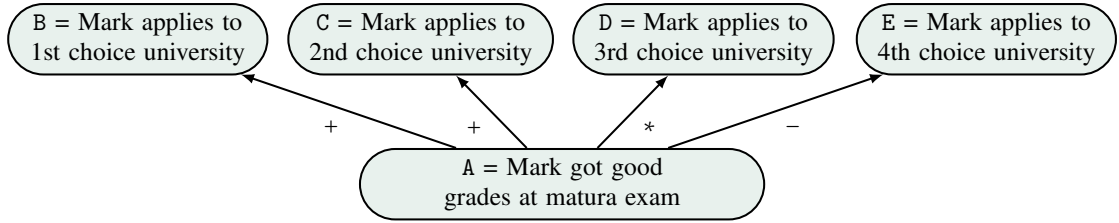
\begin{figure}[!ht]
\centering  
\resizebox{\textwidth}{!}{
 \begin{tikzpicture}[->,>=latex,thick,auto, arg/.style={draw,text centered,shape = rounded rectangle,fill=darkgreen!10}]
\node[arg] (a) [text width=5cm] at (1.85,0) {$\xa$ = Mark got good grades at matura exam}; 
\node[arg] (b) [text width=3cm]at (-4,2) {$\xb$ = Mark applies to 1st choice university};
\node[arg] (c) [text width=3.075cm]at (-0.1,2) {$\xc$ = Mark applies to 2nd choice university};
\node[arg] (d) [text width=3.015cm]at (3.8,2) {$\xd$ = Mark applies to 3rd choice university};
\node[arg] (e) [text width=3cm]at (7.7,2) {$\xe$ = Mark applies to 4th choice university};

\path (a) edge node{$+$} (b)
	 (a) edge node{$+$}  (c)
	 (a) edge node[swap]{$*$}  (d)
	 (a) edge node[swap]{$-$}  (e);  
\end{tikzpicture} 
}
\caption{Mark's university choice graph. The $+$ labels denote support, $-$ attack, and $*$ dependency.}
\label{fig:matura}
\end{figure}

\begin{example}
Let us consider the graph depicted in Figure \ref{fig:matura}. Given the rules in his country, Mark has written the matura exam (national exam after high school allowing a person to apply to a university) and can now register for up to two 
universities that interest him. 
He will be accepted or rejected once the exam results are in. 
We create the following constraints expressing what Mark plans to do: 
\begin{itemize}
\item If Mark strongly disbelieves that he will get good grades, he will apply only to his 4th choice university:

$\pprob(\xa) \leq 0.2 \rightarrow \pprob(\xb) < 0.5 \land \pprob(\xc) < 0.5 \land \pprob(\xd) < 0.5 \land \pprob(\xe) >0.5$

\item If Mark moderately does not believe that he will get good grades, he will apply only to his 3rd and 4th choice universities:

$\pprob(\xa) > 0.2 \land \pprob(\xa) \leq 0.5 \rightarrow \pprob(\xb) < 0.5 \land \pprob(\xc) < 0.5 \land \pprob(\xd) > 0.5 \land \pprob(\xe) >0.5$

\item If Mark moderately believes his grades will be good, he will apply only to his 2nd and 3rd choice universities: 

$\pprob(\xa) > 0.5 \land \pprob(\xa) < 0.8 \rightarrow \pprob(\xb) < 0.5 \land \pprob(\xc) > 0.5 \land \pprob(\xd) > 0.5 \land \pprob(\xe) < 0.5$

\item If Mark strongly believes he will get good grades, he will apply only to his 1st and 2nd choice universities:

$\pprob(\xa) \geq 0.8  \rightarrow \pprob(\xb) > 0.5 \land \pprob(\xc) > 0.5 \land \pprob(\xd) < 0.5 \land \pprob(\xe) < 0.5$

\end{itemize}
We can consider the relation between $\xa$ and $\xe$ to be conflicting, as once $\xa$ is believed we disbelieve $\xe$. 
Given that believing $\xa$ (to a sufficiently high degree) also leads to believing $\xb$ and $\xc$, the relations between 
these arguments can be seen as supporting. However, the interaction between $\xa$ and $\xd$ cannot be clearly 
classified as supporting or attacking, given that as the belief in $\xa$ increases, $\xd$ can be disbelieved, believed, 
and then disbelieved again. 
\end{example}

Epistemic graphs are therefore quite flexible in representing various restrictions on beliefs. 
However, given the freedom we have in defining constraints, we can create epistemic graphs in which the constraints 
do not really reflect the structure of the graph and vice versa. Moreover, a probability distribution satisfying our requirements 
may be further refined in various ways, independently of the graph in question. Thus, in the next section we would like 
to explore testing if and how the graph structure can be reflected by the constraints and introduce various 
types of epistemic semantics.


\subsection{Coverage}
\label{sec:CoverageAnalysis}

Previously, we have stated that it is not necessary for the constraints to account for all arguments and 
all the relations between them. 
While the ability to operate a not fully defined framework is valuable from the practical point of view, for example when 
dealing with limited knowledge about an opponent during a dialogue, having a graph in which the constraints cover all possible 
scenarios has undeniable benefits. In this section we will therefore focus on notions that can be used to measure if, and 
to what degree, arguments and relations between them are accounted for by the constraints. We will consider possible 
means of using this information in Section \ref{sec:intercoh}. 
 
The general idea of verifying coverage relies on modulating beliefs in certain arguments and observing 
whether it results in particular behaviours in the arguments we are interested in. Key notion in this 
is a constraint combination, which we will use as a \enquote{modulating} component: 

\begin{definition}
\label{def:constraintcomb}
Let $F = \{\xa_1, \ldots, \xa_m\} \subseteq \nodes(\graph)$ be a set of arguments. An 
\textbf{exact constraint combination} for $F$ is a set 
${\cal CC}^F = \{\pprob(\xa_1) = x_1, \pprob(\xa_2) = x_2, \ldots, \pprob(\xa_m) = x_m \}$, where $x_1, \ldots, x_m \in [0,1]$. 
A \textbf{soft constraint combination} for $F$ is a set 
${\cal CC}^F = \{\pprob(\xa_1) \ineq_1 x_1, \pprob(\xa_2) \ineq_2 x_2, \ldots, \pprob(f_m) \ineq_m x_m \}$, where 
$x_1, \ldots, x_m \in [0,1]$ and $\ineq_1, \ldots, \ineq_m \in \operators$. 
With ${\cal CC}^F\rvert_{G}$ for $G \subseteq \nodes(\graph)$ we denote the subset of ${\cal CC}^F$ 
that consists of all and only constraints of ${\cal CC}^F$ that are on arguments contained in $F \cap G$. 
\end{definition}

Verifying if and how the belief in an argument changes given the beliefs in other arguments can possess certain challenges 
depending on how the set of constraints is defined. Amending the set of constraints with the above combinations 
might lead to inconsistencies coming from the fact that the arguments in the combinations themselves are interrelated 
or because the set of constraints already affects the belief in one of the arguments in the combination by default. 
Furthermore, we need to take into account the fact that the set of constraints associated with the graph might not be consistent
to start with. In the following sections we will work under the assumption that we are 
dealing with a graph s.t. the associated set of constraints is satisfiable, and for a discussion on inconsistent constraints 
refer to \cite{HunterPT2018Arxiv}.

\subsubsection{Argument Coverage}

On its own, an argument can be assigned any probability value from $[0,1]$. One of the purposes of the constraints 
is - as the name suggest - to constrain the range of values that an argument may take, for example by the values 
assigned to its parents. Coverage means that there is at least one value for the degree of belief 
of an argument cannot 
be assigned, be it straight from the constraints or under certain assumptions concerning the beliefs in other arguments, cf.\ general constraint propagation to restrict the domain of variables \cite{Dechter2003}. 
The most basic form of coverage is the default coverage, where we can find a degree of belief that an argument 
cannot take straightforwardly from the constraints and without imposing additional assumptions:

\begin{definition}
\label{def:defcov} 
Let $X = (\graph, \lab, \con)$ be a consistent epistemic graph.  
We say that an argument $\xa \in \nodes(\graph)$ is 
\textbf{default covered in $X$} if there is a value $x \in [0,1]$ s.t. $\con \VDash \pprob(\xa) \neq x$.   
\end{definition} 

\begin{example}
\label{ex:coverage1}
Let us consider the graph depicted in Figure \ref{fig:defcov} and the associated set of constraints $\con$:
$$\{\pprob(\xa) > 0.5, 
\pprob(\xa) > 0.5 \rightarrow \pprob(\xb) < 0.5, 
(\pprob(\xb) < 0.5 \land \pprob(\xc) > 0.5) \rightarrow \pprob(\xd) \leq 0.5,
\pprob(\xc) \leq 0.5 \rightarrow \pprob(\xd) > 0.5\}$$

In this case, we can observe that both $\xa$ and $\xb$ are covered by default.  
For example, $\con \VDash \pprob(\xa) \neq 0.5$ and $\con \VDash \pprob(\xb) \neq 0.5$. This comes from 
the fact that the belief in $\xa$ is restricted from the very beginning and from it we can derive the restrictions for $\xb$. 
However, arguments $\xc$ and $\xd$ are not default covered. Although they are constrained and, for example, 
it cannot be the case that they are both believed or both disbelieved at the same time, for every belief value 
$x \in [0,1]$ we can still find a probability distribution $\prob$ s.t. $\prob(\xc) = x$ (resp. $\prob(\xd) = x$). 
\end{example}

\begin{figure}[!ht]
\centering
\begin{minipage}[t]{.5\textwidth} 
\centering
  \begin{tikzpicture}
[->,>=stealth,shorten >=1pt,auto,node distance=1.7cm,
  thick,main node/.style={shape=rounded rectangle,fill=darkgreen!10,draw,minimum size = 0.6cm,font=\normalsize\bfseries} 
]

\node[main node] (a) {$\xa$};
\node[main node] (b) [right of=a] {$\xb$};
\node[main node] (c) [right of=b] {$\xc$};
\node[main node] (d) [right of=c] {$\xd$}; 
 
 \path
	(a) edge node {$-$} (b) 
	(b) edge [bend right=45] node [swap] {$+$} (d)
	(c) edge [bend right] node[swap] {$-$} (d) 
    (d) edge [bend right] node[swap] {$-$} (c);
\end{tikzpicture}
\captionof{figure}{An argument graph}
\label{fig:defcov} 
\end{minipage}%
\begin{minipage}[t]{.45\textwidth} 
        \centering
  \begin{tikzpicture}
[->,>=stealth,shorten >=1pt,auto,node distance=1.7cm,
  thick,main node/.style={shape=rounded rectangle,fill=darkgreen!10,draw,minimum size = 0.6cm,font=\normalsize\bfseries} 
]

\node[main node] (a) at(0,0) {$\xa$};
\node[main node] (b) at(1,1) {$\xb$};
\node[main node] (c) at(-1,1) {$\xc$}; 
 
 \path
	(b) edge node[swap] {$-$} (a)
	(c) edge node[swap] {$-$} (a)
	(b) edge node[swap] {$-$} (c);
\end{tikzpicture} 
\captionof{figure}{A conflict--based argument graph}
\label{fig:inconsistentcombination} 
\end{minipage}
\end{figure}

The above example also shows that in some cases, the default coverage may be too restrictive. 
Although neither $\xc$ nor $\xd$ are default covered, the belief we have in one restricts the belief we have in the other. 
Thus, our intuition is that some form of coverage should exist. In our case, every level of belief we had 
in $\xc$ had constrained $\xd$ and vice versa. However, even weaker forms may be considered: 
 
\begin{example}
\label{ex:coverage2} 
Let us consider the framework depicted in Figure \ref{fig:inconsistentcombination} and the following set of constraints $\con$:
$$\{\varphi_1 \formis \pprob(\xb) >0.5 \rightarrow \pprob(\xc) \leq 0.5,
\varphi_2 \formis (\pprob(\xb) >0.5  \land \pprob(\xc) \geq 0.5) \rightarrow \pprob(\xa) <0.5\}$$


Let us analyze how the belief in $\xa$ is constrained in the graph. Our intuition is that some coverage does exist. 
In particular, we can observe that if $\xb$ is believed and $\xc$ is not disbelieved, then $\xa$ is disbelieved and thus there are some 
probabilities it cannot take in this context. However, 
if this condition is not satisfied, then $\xa$ can take on any probability. Thus, the coverage is, in a sense, \enquote{partial}. 
\end{example}

We therefore introduce the 
additional notions of coverage below. 
Given the fact that the constraints can occur between unrelated arguments and that for certain types of relations 
the belief in an argument is more affected by the arguments it is targeting rather than by those that are its parents, 
we allow for testing coverage against an arbitrary set of arguments. 
We say that an argument is partially covered by a set of arguments $F$
if we can find a belief assignment for $F$ that respects the existing constraints and leads to our argument 
not being able to take on some values. Full coverage states that every appropriate belief assignment for $F$
should lead to the argument not taking on some values.  
 
\begin{definition} \label{def:coverage}
Let $X = (\graph, \lab, \con)$ be a consistent epistemic graph, $\xa \in \nodes(\graph)$ 
an argument and $F \subseteq \nodes(\graph)\setminus\{\xa\}$ a set of arguments.  
We say that $\xa$ is: 
\begin{itemize}  
\item \textbf{partially covered by $F$} in $X$ if there exists a constraint combination ${\cal CC}^{F}$ and a value $x \in [0,1]$ s.t. 
${\cal CC}^{F} \cup \con \not\VDash \bot$ and 
${\cal CC}^{F} \cup \con \VDash \pprob(\xa) \neq x$ 

\item \textbf{fully covered by $F$} in $X$ if 
for every constraint combination ${\cal CC}^{F}$ s.t. ${\cal CC}^{F} \cup \con \not\VDash \bot$,
there exists a value $x \in [0,1]$ s.t.  
${\cal CC}^{F} \cup \con \VDash \pprob(\xa) \neq x$
\end{itemize}
\end{definition}   

We note that for a graph that possesses a consistent set of constraints, for every 
set of arguments $F$ we can find a constraint combination 
${\cal CC}^{F}$ for $F$ s.t. ${\cal CC}^{F} \cup \con$ is consistent (see also Definition \ref{def:constraintcomb}). 
It is also worth noting that for $F = \emptyset$, the definitions of partial, full and default coverage coincide. The set ${\cal CC}^{F}$ is also called an \emph{eliminating explanation} \cite{VanBeek:2006}.

We can observe that in the above definition, we exclude the effect an argument may have on itself (i.e. 
the set $F$ cannot contain the argument in question). While it has clear technical benefits, 
we also observe that constraints representing directly self--attacking and self--supporting arguments 
either provide default coverage or no coverage at all. 
Let us consider a simple graph with an argument $\xa$ s.t $\xa$ is a self--attacker, which can be represented 
with constraints $\pprob(\xa) >0.5 \rightarrow \pprob(\xa) < 0.5$ (i.e. if $\xa$ is believed, then 
$\xa$ is disbelieved) and  $\pprob(\xa) < 0.5 \rightarrow 
\pprob(\xa) >0.5$ (i.e. if $\xa$ is disbelieved, then its attackee (and/or attacker) $\xa$ is believed). 
From this we can infer that $\pprob(\xa) = 0.5$ which 
provides default coverage. Performing a similar analysis for a self--supporter (i.e. if $\xa$ is believed, then $\xa$ is believed 
and if $\xa$ is disbelieved, then $\xa$ is disbelieved) leads to a tautology constraint and provides no coverage at all. 

%

\begin{example}
\label{ex:coveragecont}
Let us consider the graph from Example \ref{ex:coverage1} and look at arguments $\xc$ and $\xd$. 
We can start by analyzing whether arguments $\xa$ and $\xb$ provide any coverage for them. We can see 
that any constraint combination $\{\pprob(\xa) = x, \pprob(\xb) = y\}$ 
for these two arguments that is consistent with the existing formulae is such that $x \in (0.5,1]$ and $y \in [0, 0.5)$.
Nevertheless, there is no value $z \in [0,1]$ s.t. the union of our constraint combination 
and the original set of constraints entails $\pprob(\xc) \neq z$ or $\pprob(\xd) \neq z$. Consequently, 
these arguments provide no coverage (be it full or partial), which is in accordance with our intuition. 

Let us therefore consider constraint combinations on $\xc$ and analyze the argument $\xd$. 
We can observe that any set $\{\pprob(\xc) = v\}$ for $v \in [0,1]$ is consistent with $\con$. 
For $v \in [0, 0.5]$, we can observe that $\con \cup \{\pprob(\xc) = v\} \VDash \pprob(\xd) > 0.5$. 
Thus, for example, $\con \cup \{\pprob(\xc) = v\} \VDash \pprob(\xd) \neq 0$. 
For $v \in (0.5, 1]$, we can observe that $\con \cup \{\pprob(\xc) = v\} \VDash \pprob(\xd) \leq 0.5$. 
Therefore, for example, $\con \cup \{\pprob(\xc) = v\} \VDash \pprob(\xd) \neq 1$. 
Hence, we can argue that $\xd$ is both partially and fully covered by $\{\xc\}$ (and, as a result, also by sets containing $\xc$). 
Similar arguments can be made for showing that $\xc$ is partially and fully covered by $\{\xd\}$.  
\end{example}

\begin{example}
\label{ex:counterexfullpart}
Let us come back to Example \ref{ex:coverage2} and check whether argument $\xa$ is covered by the set $\{\xb,\xc\}$. 
We can observe that all constraint combinations $\{\pprob(\xb) = x, \pprob(\xc) = y\}$ are consistent 
with $\con$ as long as 
either $x \leq 0.5$ or $y \leq 0.5$. We can observe that 
$\{\pprob(\xb) = 1, \pprob(\xc) = 0.5\} \cup \con \VDash \pprob(\xa) < 0.5$. 
Thus, for example, $\{\pprob(\xb) = 1, \pprob(\xc) = 0.5\} \cup \con \VDash \pprob(\xa) \neq 1$, 
and we have at least partial coverage. 
However, if we consider $\{\pprob(\xb) = 0.5, \pprob(\xc) = 0.5\}$, 
then $\xa$ can be assigned any belief from $[0,1]$. In other words, 
there is no value $z \in [0,1]$ s.t. $\{\pprob(\xb) = 0.5, \pprob(\xc) = 0.5\} \cup \con \VDash \pprob(\xa) \neq z$. 
Thus, the coverage is not full.  
\end{example}

In the above partial and full versions of the coverage, we needed to select the arguments against which we wanted to 
check whether the belief in an argument is restricted or not. For some applications, this extra information might be unnecessary, 
and thus we can consider the arbitrary versions of partial and full coverage, i.e. ones in which the actual set $F$ is not important
as long as at least one exists. 

\begin{definition}
Let $X = (\graph, \lab, \con)$ be a consistent epistemic graph. An argument $\xa \in \nodes(\graph)$ 
has \textbf{arbitrary full/partial coverage} iff there exists a set of arguments 
$F \subseteq \nodes(\graph)\setminus \{\xa\}$ s.t. $\xa$ is fully or partially covered w.r.t. $F$. 
\end{definition}
 
The following relationships between the various forms of coverage can be shown straightforwardly:

\begin{restatable}{proposition}{conscoverage}
Let $X = (\graph, \lab, \con)$ be a consistent epistemic graph, $\xa \in \nodes(\graph)$ be an 
argument and $F = \nodes(\graph)\setminus\{\xa\}$ be a set of arguments. 
The following hold:
\begin{itemize}
\item If $\xa$ is default covered in $X$, then it is partially and fully covered w.r.t. any set of arguments $G \subseteq  \nodes(\graph)\setminus\{\xa\}$, but not necessarily vice versa
\item If $\xa$ is fully covered in $X$ w.r.t. $F$, then it is partially covered in $X$ w.r.t. $F$, but not necessarily vice versa 
\end{itemize}
\end{restatable}  

Finally, we can observe that for epistemic graphs whose constraints have the same satisfying distributions, the coverage 
analysis leads to the same results:

\begin{restatable}{proposition}{conscoverageequiv}
\label{conscoverageequiv}
Let $X = (\graph, \lab, \con)$ and $X' = (\graph', \lab', \con')$ be consistent epistemic graphs 
s.t. $\sat(\con) = \sat(\con')$. An argument $\xa \in \nodes(\graph)$ is default (partially, fully) covered 
in $X$ (and w.r.t. $F \subseteq \nodes(\graph) \setminus \{\xa\}$) iff it is default (partially, fully) covered in $X'$ (w.r.t. $F$).
\end{restatable} 
 
This result further highlights the need of contrasting information in the graph 
with the information in the constraints, which we will address in Section \ref{sec:intercoh}. 

\subsubsection{Relation Coverage}
\label{sec:RelationCoverage}

In the previous section we have discussed properties concerning whether an argument is sufficiently covered by the constraints. 
However, it also makes sense 
to check whether every relation is covered by the constraints as well. For example, we can consider an argument $\xa$ 
and its parents $\xb$ and $\xc$. It is possible that the constraints are defined in a way that only $\xb$ has an actual effect 
on $\xa$. Thus, the relation between $\xc$ and $\xa$ might have no real impact, despite the fact that $\xa$ may be fully 
covered in the graph. Hence, we also test for the effectiveness of a given relation, which is understood as the ability of the source 
to change the belief restrictions on the target argument. 
We therefore introduce the following definition, which simply states that there is a point at which changing the belief 
of the source of a relation will lead to a change in the belief we have in the target. In order to be able to look 
at effectiveness of explicit as well as implicit relations, we do not limit ourselves only to those mentioned in $\arcs(\graph)$:

\begin{definition} 
Let $X = (\graph, \lab, \con)$ be a consistent epistemic graph, $F \subseteq \nodes(\graph) \setminus \{\xb\}$ 
and $G  = F \setminus \{\xa\}$ be sets of arguments.  
The relation represented by $(\xa, \xb)\in \nodes(\graph) \times \nodes(\graph)$ is: 
\begin{itemize}
\item \textbf{effective w.r.t. $F$} if there exists a constraint combination ${\cal CC}^{F}$ and 
values $x,y \in [0,1]$ s.t. 
\begin{itemize}
\item $\con \cup {\cal CC}^{F} \nVDash \bot$, and
\item $\con \cup {\cal CC}^F\rvert_{G} \cup \{\pprob(\xa) = y\} \nVDash \bot$, and
\item at least one of the following conditions holds: 
	\begin{itemize} 
	\item $\con \cup {\cal CC}^{F} \nVDash \pprob(\xb) \neq x$ 
and $\con \cup {\cal CC}^F\rvert_{G} \cup \{\pprob(\xa) = y\} \VDash \pprob(\xb) \neq x$, or 
	\item $\con \cup {\cal CC}^{F} \VDash \pprob(\xb) \neq x$
and $\con \cup {\cal CC}^F\rvert_{G} \cup \{\pprob(\xa) = y\} \nVDash \pprob(\xb) \neq x$.  
\end{itemize}
\end{itemize}
\item 
\textbf{strongly effective w.r.t. $F$} if for every constraint combination ${\cal CC}^{F}$ s.t.
$\con \cup {\cal CC}^{F} \nVDash \bot$, there exist  
values $x,y \in [0,1]$ s.t. $\con \cup {\cal CC}^F\rvert_{G} \cup \{\pprob(\xa) = y\} \nVDash \bot$,
 and at least one of the following conditions holds: 
	\begin{itemize} 
	\item $\con \cup {\cal CC}^{F} \nVDash \pprob(\xb) \neq x$ 
and $\con \cup {\cal CC}^F\rvert_{G} \cup \{\pprob(\xa) = y\} \VDash \pprob(\xb) \neq x$, or 
	\item $\con \cup {\cal CC}^{F} \VDash \pprob(\xb) \neq x$
and $\con \cup {\cal CC}^F\rvert_{G} \cup \{\pprob(\xa) = y\} \nVDash \pprob(\xb) \neq x$.  
\end{itemize}
\end{itemize}
\end{definition} 
  
\begin{example}
\label{ex:effectivenesscounterex}
Let us consider a simple set of constraints $\con = \{ \pprob(\xa) > 0.5 \rightarrow \pprob(\xb) \leq 0.5\}$ 
and analyze the impact $\xa$ has on $\xb$. For this analysis, we assume $F = \{\xa\}$ and $G = \emptyset$. 
Consequently, we will focus on analyzing what we can conclude from $\con \cup \{\pprob(\xa) = z\}$ for 
selected values of $z \in [0,1]$.
We can observe that for every value of $z$, $\con \cup \{\pprob(\xa) = z \}  \not\VDash \bot$. 
Consequently, the first two conditions of effectiveness are easily satisfied. 
Let $z = 1$ and ${\cal CC}^{F} = \{\pprob(\xa) = 1 \}$. 
It holds that $\con \cup \{\pprob(\xa) = 1 \} \VDash \pprob(\xb) \leq 0.5$. 
Thus, for example, $\con \cup \{\pprob(\xa) = 1 \} \VDash \pprob(\xb) \neq 1$. 
However, if we set the probability of $\xa$ to $0$, then $\xb$ is allowed to take on any probability. 
Since ${\cal CC}^{G} = \emptyset$, it suffices to show that 
$\con \cup \emptyset \cup \{\pprob(\xa) = 0 \} \not\VDash \pprob(\xb) \neq 1$. Consequently, 
the third condition of effectiveness is also satisfied, and the $(\xa, \xb)$ 
relation is effective. In a similar fashion, we can show that it is strongly effective. 
\end{example}

\begin{example}
\label{ex:effectivenesscounterex2}
Let us now consider the following set of constraints $\con$:
$$\{\varphi_1 \formis \pprob(\xa) > 0.5 \rightarrow \pprob(\xb) >0.5, \,
 \varphi_2 \formis \pprob(\xc) >0.5 \rightarrow \pprob(\xb) > 0.9\}$$

We can analyze how $\xa$ and $\xc$ affect $\xb$ and consider the constraint combinations on $F=\{\xa,\xc\}$. 
We observe 
that both $(\xa, \xb)$ and $(\xc, \xb)$ are effective w.r.t. $\{\xa,\xc\}$. For instance, in case of $\xa$, 
we can take the constraint combinations ${\cal CC}^{F} = \{\pprob(\xa) = 0, \pprob(\xc) = 0\}$ and 
${\cal CC}^{G} = \{\pprob(\xc) = 0\}$ to see that 
$\con \cup {\cal CC}^{F} \nVDash \pprob(\xb) \neq 0.4$ and 
$\con \cup {\cal CC}^{G} \cup \{\pprob(\xa) = 0.7\} \VDash \pprob(\xb) \neq 0.4$. 
Similar analysis can be carried out for $\xc$. 

We can also observe that $(\xc, \xb)$ is strongly effective w.r.t. $F$. 
Let ${\cal CC}^{F} = \{\pprob(\xa) = x, \pprob(\xc) = y\}$ be an arbitrary constraint combination. 
If $x \in [0,1]$ and $y \leq 0.5$, then we can take ${\cal CC}^{G} = \{\pprob(\xa) = x\}$
and $\{\pprob(\xc) = 1\}$ to observe that 
observe that ${\cal CC}^{F} \cup \con \not\VDash \pprob(\xb) \neq 0.9$
and
${\cal CC}^{G} \cup \{ \pprob(\xc) =1\} \cup \con \VDash \pprob(\xb) \neq 0.9$. 
If $x \in [0,1]$ and $y > 0.5$, then we can take $\{\pprob(\xc) = 0\}$ 
to show that
${\cal CC}^{F}  \cup \con \VDash \pprob(\xb) \neq 0.9$
and ${\cal CC}^{G} \cup \{\pprob(\xc) = 0\} \cup \con \not\VDash \pprob(\xb) \neq 0.9$. 
Hence, in all cases, modifying the belief associated with $\xc$ changes the restrictions on $\xb$. 

We note that unlike $(\xc,\xb)$,
 $(\xa,\xb)$ is not strongly effective w.r.t. $F$. 
For example, we can consider the combination ${\cal CC}^{F} = \{\pprob(\xa) = 0.6, \pprob(\xc) = 0.6\}$ 
for $\xa$. In this case, ${\cal CC}^{G} = \{\pprob(\xc) = 0.6\}$. We can observe that 
${\cal CC}^{G} \cup \con \VDash \pprob(\xb) > 0.9$
and no matter the value of $x \in [0,1]$, adding $\{\pprob(\xa) = x\}$ to our premises will not change 
the restrictions on $\xb$. 
\end{example} 

The above definition of effectiveness is in fact a rather demanding one in the sense that even though 
there might exist a constraint from which we can see how two arguments are connected, 
other constraints in the graph might make it impossible for it to ever become \enquote{active}, so to speak. 
For example, coverage such as default, can interfere with detecting the effectiveness of a given 
relation. Let us consider the following scenario:

\begin{example}
\label{ex:effectivegap}
Let us look at the following set of constraints $\con$ and assume that $\{(\xb,\xa),(\xc, \xa)\} = \arcs(\graph)$:
$$\{\varphi_1 \formis \pprob(\xb) \leq 0.5 \land \pprob(\xc) < 0.5, \,
 \varphi_2 \formis (\pprob(\xb) \leq 0.5  \land \pprob(\xc) < 0.5) \rightarrow \pprob(\xa) <0.5\}$$
We can observe that even though $\xb$ and $\xc$ are not default covered, $\xa$ is. 
In particular, $\con \VDash \pprob(\xa) < 0.5$. In other words, no constraint combination on $\{\xb,\xc\}$
that is consistent with $\con$ will affect the restrictions on the probability of $\xa$. 
Hence, the $(\xb,\xa)$ and $(\xc, \xa)$ relations will not be considered effective. 
\end{example}

Given this, we can consider a weaker form of effectiveness, where the impact of other constraints may be disregarded. 
To achieve this, we test effectiveness not against the set of constraints $\con$, but against any consistent 
set of constraints derivable from it: 

\begin{definition} 
\label{def:effectiveness}
Let $X = (\graph, \lab, \con)$ be a consistent epistemic graph, 
$Z \subseteq \closure(\con)$ be a consistent set of epistemic constraints, $F \subseteq \nodes(\graph) \setminus \{\xb\}$ 
and $G  = F \setminus \{\xa\}$ be sets of arguments.  
Then $(\xa, \xb) \in \nodes(\graph) \times \nodes(\graph)$ is: 
\begin{itemize}
\item \textbf{semi--effective w.r.t. $(Z,F)$} if there exist a constraint combination ${\cal CC}^{F}$ and 
values $x,y \in [0,1]$ s.t. 
\begin{itemize}
\item $Z \cup {\cal CC}^{F} \nVDash \bot$, and
\item $Z \cup {\cal CC}^F\rvert_{G} \cup \{\pprob(\xa) = y\} \nVDash \bot$, and
\item at least one of the following conditions holds: 
	\begin{itemize} 
	\item $Z \cup {\cal CC}^{F} \nVDash \pprob(\xb) \neq x$ 
and $Z \cup {\cal CC}^F\rvert_{G} \cup \{\pprob(\xa) = y\} \VDash \pprob(\xb) \neq x$, or 
	\item $Z \cup {\cal CC}^{F} \VDash \pprob(\xb) \neq x$
and $Z \cup {\cal CC}^F\rvert_{G} \cup \{\pprob(\xa) = y\} \nVDash \pprob(\xb) \neq x$.  
\end{itemize}
\end{itemize}
\item 
\textbf{strongly semi--effective w.r.t. $(Z,F)$} if 
for every constraint combination ${\cal CC}^{F}$ 
s.t. $Z \cup {\cal CC}^{F} \nVDash \bot$, 
there exist values $x,y \in [0,1]$ s.t. $Z \cup {\cal CC}^F\rvert_{G} \cup \{\pprob(\xa) = y\} \nVDash \bot$ 
and at least one of the following conditions holds: 
	\begin{itemize} 
	\item $Z \cup {\cal CC}^{F} \nVDash \pprob(\xb) \neq x$ 
and $Z \cup {\cal CC}^F\rvert_{G} \cup \{\pprob(\xa) = y\} \VDash \pprob(\xb) \neq x$, or 
	\item $Z \cup {\cal CC}^{F} \VDash \pprob(\xb) \neq x$
and $Z \cup {\cal CC}^F\rvert_{G} \cup \{\pprob(\xa) = y\} \nVDash \pprob(\xb) \neq x$.  
\end{itemize} 
\end{itemize}
\end{definition}   

\begin{example}
\label{ex:semieffectiveness}
Let us come back to Example \ref{ex:effectivegap}. We could have observed that the $(\xb, \xa)$ 
and $(\xc,\xa)$ relations were not effective. Let us take $Z = \{(\pprob(\xb) \leq 0.5  \land \pprob(\xc) < 0.5) \rightarrow \pprob(\xa) <0.5\}$ and $F = \{\xb,\xc\}$. It is easy to verify that $Z \subseteq \closure(\con)$. 
We observe that any constraint combination on the set $\{\xb, \xc\}$ is consistent with $Z$, 
i.e. for any $x,y \in [0,1]$, $\{\pprob(\xb) = x, \pprob(\xc) = y\} \cup Z \not\VDash \bot$. 
We observe that if $x \leq 0.5$ and $y < 0.5$, then 
$\{\pprob(\xb) = x, \pprob(\xc) = y\} \cup Z  \VDash \pprob(\xa) < 0.5$. 
Hence, for example, $\{\pprob(\xb) = x, \pprob(\xc) = y\} \cup Z  \VDash \pprob(\xa)\neq 1$. 
If we change either $x$ or $y$ in a way that $x > 0.5$ or $y \geq 0.5$, 
then $\xa$ can take on any probability. Thus, for such new $x'$ or $y'$, 
$\{\pprob(\xb) = x', \pprob(\xc) = y\} \cup Z  \not\VDash \pprob(\xa)\neq 1$
and 
$\{\pprob(\xb) = x, \pprob(\xc) = y'\} \cup Z  \not\VDash \pprob(\xa)\neq 1$.
Hence, the relations are semi--effective w.r.t $(Z, F)$, even though they were not effective w.r.t. $F$. 
They are unfortunately not strongly semi--effective w.r.t. $(Z, F)$. For example, 
if we took a constraint combination  $\{\pprob(\xb) = 1, \pprob(\xc) = 1\}$, 
altering the assignment for $\xb$ ($\xc$ respectively) would not change the restrictions on $\xa$. 
\end{example}

Similarly as we did in case of arbitrary argument coverage, we can speak of arbitrary (semi-) effectiveness 
as long as a suitable $(Z, F)$ pair exists. 
The following connections can be drawn between all of these forms of effectiveness: 

\begin{restatable}{proposition}{effectiveness}
\label{prop:effectiveness}
Let $X = (\graph, \lab, \con)$ be a consistent epistemic graph, 
$Z \subseteq \closure(\con)$ be a consistent set of epistemic constraints, $F \subseteq \nodes(\graph) \setminus \{\xb\}$ 
and $G  = F \setminus \{\xa\}$ be sets of arguments.  
Let $(\xa, \xb) \in \nodes(\graph)\times \nodes(\graph)$. The following hold:
\begin{itemize}
\item If $(\xa, \xb)$ is strongly effective w.r.t. $F$, then it is effective w.r.t. $F$, but not necessarily vice versa
\item If $(\xa, \xb)$ is strongly semi--effective w.r.t. $(Z,F)$, then it is semi--effective w.r.t. $(Z, F)$, but not necessarily
 and vice versa
\item If $(\xa, \xb)$ is effective w.r.t. $F$, then it is semi--effective w.r.t. $(\con, F)$ and vice versa
\item If $(\xa, \xb)$ is strongly effective w.r.t. $F$, then it is strongly semi--effective w.r.t. $(\con, F)$ and vice versa
\item If $Z \neq \con$ and $(\xa, \xb)$ is semi--effective w.r.t. $(Z, F)$, then it is not necessarily effective w.r.t. $F$
\item If $Z \neq \con$ and $(\xa, \xb)$ is strongly semi--effective w.r.t. $(Z, F)$, then it is not necessarily strongly 
effective w.r.t. $F$
\end{itemize}
\end{restatable}  

\subsection{Relation Types}
\label{sec:consistentlabel}

Labellings are useful for indicating the kind of influence one argument has on another. In epistemic graphs, 
the labels can be either provided during the instantiation process or, 
similarly as in the bipolar abstract dialectical frameworks, derived from the constraints\footnote{We note that we refer only to edge labels here, not edges in general. Epistemic graphs are structurally different from ADFs and deriving the graph from the constraints is generally not possible (we refer to the beginning of Section \ref{sec:epistemicgraphs} and to Section \ref{sec:ef-adf} for additional details).}. 
This however begs the question 
whether the way a relation is labelled is really consistent with the way it is described 
by the constraints.  
While taking the labelling as input has the benefit of being informed 
by the method that has instantiated the graph from a given knowledge base, the derivation approach offers more understanding of 
the real impact a given relation has on the arguments connected to it. By this we understand that 
determining edge types during instantiation is typically a very \enquote{local} process in which, for instance, we check whether the conclusions of two arguments are contradictory or not, or if conclusion of one is a premise of another. This often ignores the presence of other arguments. For example, it is perfectly possible for an argument $\xa$ that is locally a supporter of $\xb$ to also support an attacker $\xc$ of $\xb$, and thus have a negative influence on $\xb$ from a more \enquote{global} standpoint.
In this section we will focus on analyzing what constraints are telling us about relations between arguments. 

Inferring the type of a relation we are dealing with based on how the parent affects the target is not as trivial as one may think.
For instance, even in the case of attack relations, we have binary attack, group
attack, attacks as defined in ADFs or attacks as weakening relations, to list a few  \cite{Dung95,CayrolLS05b,Caminada:2009,BrewkaESWW13,BrewkaPW14}.
Acceptance of an attacker can lead to disbelieving the target, decreasing the belief in the target, 
or - in the presence of
e.g. overpowering supporters as in ADFs - have no effect at all. 
While there is a general consensus among argumentation formalisms that an attack should 
not have positive effects, 
the notion of a negative effect is still very broad. 
Since epistemic graphs are expressive enough to model all of these behaviours, it is therefore
valuable to study them in this setting. 

Therefore, as we can observe, even a simple attack can lead to various behaviours, and some of them may overlap with possible behaviours of supporting relations, one
has to be careful when judging a relation by the effect it has.
Additionally, even though 
two arguments can appear to be positively or negatively related on their own, taking into account the effects of other 
arguments in the graph might also bring to light other behaviours. Certain works on argument frameworks 
introduce the notions of indirect relations \cite{inproc:prudent,inproc:careful,CayrolLS13}. 
For example, one argument can support another, but at the same time 
attack another of its supporters, thus serving as an indirect attacker. It can therefore 
happen that depending on the context in which we look at two arguments, the perception of the relation between them changes:

\begin{example}
\label{ex:locglob}
Let us consider the following scenario with arguments $\xa$, $\xb$, $\xc$ and $\xd$ where  
$\xb$ and $\xc$ group support $\xa$ s.t. at least one of $\xb$ and $\xc$ needs to be believed in order to believe $\xa$, 
$\xb$ supports $\xd$ s.t. believing $\xb$ implies believing $\xd$, and $\xd$ attacks $\xa$. 
This can be depicted with the graph in Figure \ref{fig:supportex} 
and expressed with the following set of constraints $\con$:

\begin{itemize}
\item $\varphi_1 \formis \pprob(\xa) >0.5 \rightarrow \pprob(\xb) >0.5 \lor \pprob(\xc) > 0.5$
\item $\varphi_2 \formis (\pprob(\xd) <0.5 \land (\pprob(\xb) >0.5 \lor \pprob(\xc) > 0.5)) \rightarrow \pprob(\xa) >0.5$
\item $\varphi_3 \formis \pprob(\xb) >0.5 \rightarrow \pprob(\xd) >0.5$
\item $\varphi_4 \formis \pprob(\xd) >0.5 \rightarrow \pprob(\xa) < 0.5$
\end{itemize}

If we were to decide on the nature of the $\xb$--$\xa$ relation only from the constraint concerning both of them 
(i.e. constraints $\varphi_1$ and $\varphi_2$), then the supporting relation becomes quite apparent. However, if we were to 
take into account the interactions expressed in constraints $\varphi_1$ to $\varphi_4$, then we would observe that 
believing $\xb$ implies believing $\xd$ and thus disbelieving $\xa$, which is hardly a positive influence. 
\end{example}

\begin{figure}[!ht]
\centering 
\begin{minipage}[t]{.5\textwidth} 
\centering
\begin{tikzpicture}[->,>=latex,thick,auto,
main node/.style={shape=rounded rectangle,fill=darkgreen!10,draw,minimum size = 0.6cm,font=\normalsize\bfseries} ] 
 
\node[main node] (a) at (0,0) {$\xa$};
\node[main node] (b) at (1.5, 0) {$\xb$};
\node[main node] (c) at (-1.5,0) {$\xc$};
\node[main node] (d) at (3,0) {$\xd$};

\path (b) edge node {$+$} (a) 
	 (c) edge node[below] {$+$}  (a)
	(b) edge node[below] {$+$} (d) 
(d) edge [bend right] node[swap]  {$-$} (a) ; 

\end{tikzpicture}
\caption{A bipolar labelled graph for Example \ref{ex:locglob}}
\label{fig:supportex}
\end{minipage}%
\begin{minipage}[t]{.45\textwidth} 
\centering
\begin{tikzpicture}[->,>=latex,thick,auto,
main node/.style={shape=rounded rectangle,fill=darkgreen!10,draw,minimum size = 0.6cm,font=\normalsize\bfseries} ] 
 
\node[main node] (a) at (0,0) {$\xa$};
\node[main node] (b) at (1.5, 0) {$\xb$};
\node[main node] (c) at (3,0) {$\xc$}; 

\path (a) edge node[below] {$-$} (b) 
	 (c) edge node[below] {$+$}  (b) ; 

\end{tikzpicture}
\caption{A bipolar labelled graph for Example \ref{ex:29}}
\label{fig:supportex29}
\end{minipage}
\end{figure}

\begin{example}
\label{ex:29}
Let us consider the following scenario with the graph from Figure \ref{fig:supportex29}
and where we know that 
if $\xa$ is believed, then unless $\xc$ is believed, $\xb$ is disbelieved. Thus, $\xa$ carries out an attack 
that can be overruled by the support from $\xc$\footnote{This is an example of how extended argumentation frameworks 
can be modelled \cite{Polberg17}}. We can create the constraint 
$\pprob(\xa) >0.5 \land \pprob(\xc) \leq 0.5 \rightarrow \pprob(\xb) <0.5$
to reflect this. The interplay between $\xa$ and $\xc$ shows that despite the fact $\xa$ has primarily a negative effect 
on $\xb$, believing $\xa$ might not always imply disbelieving $\xb$ due to the presence of other arguments. 
\end{example}

Therefore, as we can see, both negative and positive relations can be interpreted in various ways, and their actual influence 
can change depending on the context in which they are analyzed. Hence, rather than forcing an attack to have a 
negative effect, we interpret it as a relation not having a positive effect and support as not having a negative effect.  
In this respect, our approach is similar to the one in abstract dialectical frameworks \cite{BrewkaESWW13}, 
which as seen in \cite{Polberg16} subsumes a wide range of existing methods. 
However, as motivated by Example \ref{ex:locglob}, we should additionally distinguish between local and global influence, 
the difference between them being whether all or some (parts of) constraints are taken into account.  

What we would also like to observe is that selecting the constraints against which the relations should be tested,  
is not necessarily an objective process. Let us again look at Example \ref{ex:locglob}:

\begin{example}
\label{ex:locglob2}
Let us come back to the graph depicted in Figure \ref{fig:supportex} and analyzed in Example \ref{ex:locglob}. 
%
%
In order to test for local impact that $\xb$ has on $\xa$, our intuition would be to focus on constraints 
$\varphi_1$ and $\varphi_2$. 
Let us however consider replacing $\varphi_3$ and $\varphi_4$ with an equivalent constraint
$(\pprob(\xb) >0.5 \rightarrow \pprob(\xd) >0.5) \land (\pprob(\xd) >0.5 \rightarrow \pprob(\xa) < 0.5)$. 
We observe that this replacement does not affect the satisfiability of our set. 
From the new constraint we can also infer another constraint 
$ \varphi' \formis \pprob(\xb) >0.5 \land \pprob(\xd) >0.5 \rightarrow \pprob(\xa) <0.5$. Again, adding it to the constraint set
in no way affects the satisfying distributions. However, this constraint can be interpreted as a group attack on $\xa$
by $\xb$ and $\xd$, and if we were to check  the local impact that $\xb$ has on $\xa$, the intuition would 
be to take it under consideration. Consequently, despite the logical equivalence of both the original 
and the modified sets of constraints, the perception of the relations stemming from them might not be the same. 
\end{example}

Thus, similarly as in the case of relation coverage, determining the nature of a given relation depends on the constraints 
that we choose to analyze. Likewise, we will focus on graphs with consistent constraints, and refer to \cite{HunterPT2018Arxiv} for 
discussion on handling the inconsistent ones. 

Let us first consider a simplified approach, which primarily focuses on arguments being believed, disbelieved or neither. 
We assume that a relation we want to investigate is at least semi-effective; ones that are not we will refer to as 
unspecified. Semi-effective relations can later be deemed attacking, supporting, dependent or subtle. 
Attack means that a target argument that is not believed remains as such when the source is believed. 
In other words, we 
want to avoid situations when believing the source would lead to believing the target. 
Support can be defined 
in a similar fashion. A dependent relation is seen as neither supporting nor attacking, and one that is both 
is referred to as subtle\footnote{In frameworks such as ADFs \cite{BrewkaESWW13}, a relation that is both 
attacking and supporting is redundant and can be safely 
removed from the graph \cite{Polberg17}. In our case, this more closely corresponds to unspecified ones 
due to their lack of effectiveness. Epistemic graphs are more fine-grained 
than ADFs and relations that are both positive and negative might not be redundant.}:


\begin{definition} 
\label{def:relations}
Let $X = (\graph, \lab, \con)$ be a consistent epistemic graph. 
Let 
$Z \subseteq \closure(\con)$ be a consistent set of epistemic constraints, $F \subseteq \nodes(\graph) \setminus \{\xb\}$ 
and $G  = F \setminus \{\xa\}$ be sets of arguments.   
The relation represented by $(\xa, \xb)\in \nodes(\graph) \times \nodes(\graph)$ is:
\begin{itemize}
\item \textbf{attacking w.r.t. $(Z,F)$} if it is  semi--effective w.r.t. $(Z,F)$ and
for every constraint combination ${\cal CC}^{F}$ s.t. $Z \cup {\cal CC}^{F} \not\VDash \bot$ and 
$Z \cup {\cal CC}^{F}\rvert_G \cup \{\pprob(\xa) >0.5\} \not\VDash \bot$, 
if $Z \cup {\cal CC}^{F} \VDash  \pprob(\xb) \leq 0.5$ then 
$Z \cup {\cal CC}^{F}\rvert_G \cup \{\pprob(\xa) >0.5\} \VDash \pprob(\xb) \leq 0.5$.  

\item \textbf{supporting w.r.t. $(Z,F)$} if it is semi--effective w.r.t. $(Z,F)$ and 
for every constraint combination ${\cal CC}^{F}$ s.t. $Z \cup {\cal CC}^{F} \not\VDash \bot$ and 
$Z \cup {\cal CC}^{F}\rvert_G \cup \{\pprob(\xa) >0.5\} \not\VDash \bot$, 
if $Z \cup {\cal CC}^{F} \VDash  \pprob(\xb) \geq 0.5$ then 
$Z \cup {\cal CC}^{F}\rvert_G \cup \{\pprob(\xa) >0.5\} \VDash \pprob(\xb) \geq 0.5$.

\item \textbf{dependent w.r.t. $(Z,F)$} if it is semi--effective but neither attacking nor supporting w.r.t. $(Z,F)$ 

\item \textbf{subtle w.r.t. $(Z,F)$} if it is  semi--effective and both attacking and supporting w.r.t. $(Z,F)$ 

\item \textbf{unspecified w.r.t. $(Z,F)$} if it is not semi--effective w.r.t. $(Z,F)$ 
\end{itemize}
\end{definition}

Depending on the choice of constraints and arguments that we use for testing, it can happen 
that a relation is seen as supporting or attacking due to vacuous truth. 
For example, we may never find 
an appropriate combination s.t. 
$Z \cup {\cal CC}^{F} \VDash  \pprob(\xb) \geq 0.5$ ($Z \cup {\cal CC}^{F} \VDash  \pprob(\xb) \leq 0.5$), 
or we cannot find constraint combinations that would be consistent with $Z$. 
Consequently, we can also consider the following strengthening:

\begin{definition}
Let $X = (\graph, \lab, \con)$ be a consistent epistemic graph,  
$Z \subseteq \closure(\con)$ be a consistent set of epistemic constraints, $F \subseteq \nodes(\graph)\setminus \{\xb\}$ 
and $G  = F \setminus \{\xa\}$ be sets of arguments. 
Then a supporting (resp. attacking, dependent, subtle)  w.r.t. $(Z, F)$ 
relation $(\xa, \xb)$ is \textbf{strong} w.r.t. $(Z, F)$ if:
\begin{itemize}
\item for every constraint combination ${\cal CC}^{F}$ 
it holds that $Z \cup {\cal CC}^{F} \not\VDash \bot$ and 
$Z \cup {\cal CC}^{F}\rvert_G \cup \{\pprob(\xa) >0.5\} \not\VDash \bot$
\item and 
there is at least one constraint combination ${\cal CC}^{F}$ 
s.t. 
$Z \cup {\cal CC}^{F} \VDash  \pprob(\xb) \geq 0.5$ (resp. $Z \cup {\cal CC}^{F} \VDash  \pprob(\xb) \leq 0.5$ or both). 
\end{itemize}
\end{definition} 
\begin{example}
\label{ex:rellabel}
Let us consider relation $(\xb,\xa)$. 
We can observe that 
$\con \VDash \varphi_1 \land \varphi_2$ and $\{\varphi_1 \land \varphi_2 \} 
\VDash (\pprob(\xa) \leq 0.5 \lor \pprob(\xb) > 0.5 \lor \pprob(\xc) > 0.5) 
\land (\pprob(\xa) > 0.5 \lor \pprob(\xb) \leq 0.5 \lor \pprob(\xd) \geq 0.5)$. 
Let us therefore take $Z = \{ (\pprob(\xa) \leq 0.5 \lor \pprob(\xb) > 0.5 \lor \pprob(\xc) > 0.5) 
\land (\pprob(\xa) > 0.5 \lor \pprob(\xb) \leq 0.5 \lor \pprob(\xd) \geq 0.5)\}$ and 
$F = \{\xb, \xc, \xd\}$ as our parameters for testing the nature of $(\xb,\xa)$.  
We can observe that if we take the sets $W = \{\pprob(\xb) = 0, \pprob(\xc) = 0, \pprob(\xd) = 0\}$ 
and $W' = \{\pprob(\xb) =1, \pprob(\xc) = 0, \pprob(\xd) = 0\}$, 
then $Z \cup W \VDash \pprob(\xa) \neq 1$ and $Z \cup W' \not\VDash \pprob(\xa) \neq 1$. 
Thus, the $(\xb, \xa)$ relation is semi--effective w.r.t. $(Z, F)$. 
Furthermore, for all value $y_1, y_2, y_3 \in [0,1]$, 
$Z \cup \{\pprob(\xb) =y_1, \pprob(\xc) = y_2, \pprob(\xd) = y_3\} \not\VDash \bot$
and $Z \cup \{\pprob(\xc) =y_2, \pprob(\xd) = y_3\} \cup \{ \pprob(\xb) > 0.5\} \not\VDash \bot$.
We can also observe that if $y_1 >0.5$ and $y_3<0.5$, then
$Z \cup \{\pprob(\xb) =y_1, \pprob(\xc) = y_2, \pprob(\xd) = y_3\} \VDash \pprob(\xa) > 0.5$, 
and if $y_1 \leq 0.5$ and $y_2 \leq 0.5$, then
$Z \cup \{\pprob(\xb) =y_1, \pprob(\xc) = y_2, \pprob(\xd) = y_3\} \VDash \pprob(\xa) \leq 0.5$. 
Otherwise, any probability can be assigned to $\xa$. 
Hence, for support, we only need to consider the first case, and amending the set of constraints with $\pprob(\xb) > 0.5$ 
will not change the outcome. Thus, the $(\xb, \xa)$ relation is strongly supporting w.r.t. $(Z, F)$. 
We can consider the second case and amend the constraints in the same way to see that the relation is not 
attacking. 

Let us now take into account all of the constraints and assume $Z = \con$. 
We can observe that if $F$ left the way it is, the $(\xb, \xa)$ relation is in fact unspecified. 
This is due to the fact that once the values for $\xc$ and $\xd$ are set, modifying the value of $\xb$ 
leads either to inconsistency (caused by $\varphi_3$) or does not change anything anymore. 
We can therefore reduce the set $F$ to $\{\xb, \xc\}$. At this point, we observe that for 
every $y_1, y_2 \in [0,1]$, the set $W = \{\pprob(\xb) = y_1, \pprob(\xc) = y_2\}$
is consistent with $Z$. 
Furthermore, for $y_1 > 0.5$, 
$Z \cup W \VDash \pprob(\xa) < 0.5$, 
for $y_1 \leq 0.5$ and $y_2 \leq 0.5$, $Z \cup W \VDash \pprob(\xa) \leq 0.5$,  
and for $y_1 \leq 0.5$ and $y_2 > 0.5$ $\xa$ can take any probability. 
We can therefore show that $(\xb, \xa)$ is strongly attacking w.r.t. $(Z, F)$. 
However, since we cannot derive $\pprob(\xa) \geq 0.5$, it is also supporting and subtle, even though not strongly.   
Summary of our analysis, as well as for some other relations, can be seen in Table \ref{tab:inducedlabels}.

\newcommand{\zca}{$\{(\pprob(\xa) \leq 0.5 \lor \pprob(\xb) > 0.5 \lor \pprob(\xc) > 0.5) 
\land (\pprob(\xa) > 0.5 \lor \pprob(\xc) \leq 0.5 \lor \pprob(\xd) \geq 0.5)\}$}

\newcommand{\zda}{$\{(\pprob(\xd) <0.5 \land (\pprob(\xb) >0.5 \lor \pprob(\xc) > 0.5)) \rightarrow \pprob(\xa) >0.5, 
\pprob(\xd) >0.5 \rightarrow \pprob(\xa) < 0.5\}$}

\newcommand{\zbd}{$\{\pprob(\xb) >0.5 \rightarrow \pprob(\xd) >0.5\}$}

\newcommand{\zba}{$\{(\pprob(\xa) \leq 0.5 \lor \pprob(\xb) > 0.5 \lor \pprob(\xc) > 0.5) 
\land (\pprob(\xa) > 0.5 \lor \pprob(\xb) \leq 0.5 \lor \pprob(\xd) \geq 0.5)\}$}

\begin{table}[!ht]
\centering 
\begin{tabular}{|c|c|P{6cm}|c|c|} 
\hline
Relation 		& Label 	& $Z$ 	& $F$ 				& Type 	  \\
\hline
\multirow{5}{*}{$(\xb,\xa)$}  & \multirow{5}{*}{+}  
							& \zba	& $\{\xb,\xc,\xd\}$	& (strongly) supporting		 	\\ \cline{3-5}
			&		& \multirow{4}{*}{$\con$}	& $\{\xb,\xc,\xd\}$	& unspecified		 	\\	\cline{4-5}
			&		& 		&\multirow{3}{*}{$\{\xb,\xc\}$}		& (strongly) attacking		 	\\	 
			&		&		&					& supporting \\	 
			&		&		&					& subtle \\	  
\hline
$(\xc, \xa)$	&	+	& \zca	& $\{\xb,\xc,\xd\}	$	&(strongly) supporting	 	\\ 
\hline
$(\xd, \xa)$	&	-	& \zda	& $\{\xb,\xc,\xd\}$	&(strongly) attacking		 	\\
\hline
\multirow{3}{*}{$(\xb, \xd)$}	&	\multirow{3}{*}{+}	& \multirow{3}{*}{\zbd}	& \multirow{3}{*}{$\{\xb\}$}
												& (strongly) supporting	 	\\
			&		&		&					& attacking	 	\\
			&		&		&					& subtle	 	\\
\hline
\end{tabular} 
\caption{Analysis of relations for epistemic graph from Examples \ref{ex:locglob} and \ref{ex:rellabel}.}
\label{tab:inducedlabels}
\end{table} 
\end{example}


In the previous definition we have dealt with positive and negative relations in a more ternary manner, i.e. 
it only mattered whether the parent and the target were believed, and not up to what degree. 
Thus, we can also use  more refined methods, coming in the form of positive and negative monotony. 
In other words, assuming beliefs in order relevant arguments remain unchanged, 
a higher belief in one argument node will ensure that there is a higher (resp. lower) belief in the other argument.   

\begin{definition} 
Let $X = (\graph, \lab, \con)$ be a consistent epistemic graph. 
Let 
$Z \subseteq \closure(\con)$ be a consistent set of epistemic constraints and $F \subseteq \nodes(\graph) \setminus \{\xb\}$ 
be a set of arguments.
The relation represented by $(\xa, \xb)\in \nodes(\graph) \times \nodes(\graph)$ is: 
\begin{itemize}
\item \textbf{positive monotonic} w.r.t. $(Z,F)$ if for every $\prob, \prob' \in \sat(Z)$ 
s.t. 
\begin{itemize}
\item $\prob(\xa) > \prob'(\xa)$, and
\item for all $\xc \in F$, if $\xc \neq \xa$ and $\xc \neq \xb$ then $\prob(\xc) = \prob'(\xc)$,
\end{itemize}
it holds that $\prob(\xb) > \prob'(\xb)$.

\item \textbf{negative monotonic} w.r.t. $(Z,F)$  if for every $\prob, \prob' \in \sat(Z)$ 
s.t. 
\begin{itemize}
\item $\prob(\xa) > \prob'(\xa)$, and
\item for all $\xc \in F$, if $\xc \neq \xa$ and $\xc \neq \xb$ then $ \prob(\xc) = \prob'(\xc)$,
\end{itemize}
it holds that $\prob(\xb) < \prob'(\xb)$. 

\item \textbf{non--monotonic dependent} w.r.t. $(Z,F)$  if it is neither positive nor negative monotonic 

%
\end{itemize}
\end{definition} 

Similarly as previously, we will call a relation arbitrary positive (negative) monotonic or non-monotonic dependent 
if a suitable pair $(Z, F)$ can be found. 
 
\begin{example}
\label{ex:monoton}
If we look at Example \ref{ex:rellabel} once more, we can observe that w.r.t. to the previously analyzed $(Z,F)$ pairs, 
all of the relations are non--monotonic dependent. The constraints, while they can specify whether the target 
argument should be believed or not, are not specific enough to state the precise degree of this belief. 
%
%
Instead, let us now consider a simple graph $(\{\xa,\xb,\xc\}, \{ (\xb,\xa), (\xc,\xa)\}$ where the $(\xb,\xa)$ relation
is labelled with $+$ and the $(\xc, \xa)$ relation is labelled with $-$, and we have a single 
constraint $\varphi \formis \pprob(\xc) + \pprob(\xa) - \pprob(\xb) = 1$. 
Let us focus on the $(\xb,\xa)$ relation and assume an arbitrary probability distribution $\prob$ satisfying our constraint. 
Let $\prob(\xc) = x$. Then, $\prob(\xa) = y + \prob(\xb)$ for $y = 1-x$. We can thus show that 
any increase in the belief in $\xb$ will result in a proportional increase in the belief in 
$\xa$ and that this relation is positive monotonic 
w.r.t. $(\{\varphi\}, \{\xb,\xc\})$. 
In a similar fashion we can show $(\xc,\xa)$ to be negative monotonic under the same parameters. 
\end{example}

We would like to stress that while epistemic graphs are standard directed graphs and not hypergraphs (i.e. edges join only two arguments), the relations between arguments do not need to be binary. 
By this we understand that  
it can happen that only the presence of more than one argument can start to impact another argument. 
This is reflected in the way our relation-specific notions are defined. For instance, Definition 
\ref{def:effectiveness} demands that certain conditions are met at least once, not that they are met all the time, and Definition \ref{def:relations} distinguishes attacking and supporting relations as not having a given effect instead of having one. To exemplify, forcing a supporting relation to always have an explicit positive effect would mean that it can impact a target argument on its own, without the presence of other arguments. This is the binary interpretation that is not demanded here. Our notions reflect similar concepts from ADFs, which have been shown to subsume a wide range of non-binary relations despite the fact that link type analysis is done between pairs of arguments \cite{Polberg16}.

\subsection{Internal Graph Coherence}
\label{sec:intercoh}

In the previous sections we have considered what information about arguments and relations between them we can 
extract from the constraints associated with the graph. Comparing this information with what is presented in the graph 
can provide us with insight into the completeness and internal coherence of the graph. There are many ways 
in which we could determine whether the coverage and labeling consistency of the graph are \enquote{good} 
and we intend to explore this more in the future. Currently, we focus on the following notions, though please note 
that the list is by no means exhaustive: 
\begin{definition}
Let $X = (\graph, \lab, \con)$ be an epistemic graph. 
Let ${\sf DirectRels}(\xa) = \{ \xb \mid \xb \in \parent(\xa)$ or $\xa \in \parent(\xb)\}$ 
be the set of arguments directly connected to an argument $\xa \in \nodes(\graph)$ in $\graph$. 
Let $\arcs^*(\graph) = \{ (\xa, \xb) \mid $ there exists undirected path from $\xa$ to $\xb$ in $\graph \}$ denote the 
set of pairs of all arguments connected in the graph. 
We say that $X$ is:
\begin{itemize}
\item \textbf{bounded} if every argument is default or arbitrary fully covered 
\item  \textbf{entry-bounded} if every argument is default or arbitrary fully covered apart from possibly arguments $\xa$ 
s.t. $\parent(\xa) = \emptyset$
\item  \textbf{directly connected} if every relation in $\arcs(\graph)$ is arbitrary semi-effective 
\item \textbf{indirectly connected} if every relation $(\xa,\xb) \in  \arcs^*(\graph)$ is arbitrary semi-effective
\item \textbf{hidden connected} if there exists an arbitrary semi-effective relation $(\xa,\xb) \notin \arcs^*(\graph)$ 
\item \textbf{locally connected} if for every $\xa$, there exists a consistent set $Z \subseteq \closure(\con)$ s.t. $\xa$ is fully covered w.r.t. $(Z, {\sf DirectRels}(\xa) \setminus \{\xa\})$ and for every $\xb \in {\sf DirectRels}(\xa) \setminus \{\xa\}$, 
$(\xb,\xa)$ or $(\xa, \xb)$ is semi--effective w.r.t. $(Z, {\sf DirectRels}(\xa) \setminus \{\xa\})$. 
\end{itemize}
\end{definition}

A bounded graph is an epistemic graph we would expect to obtain through translating various existing argumentation 
frameworks under their standard semantics (see also Section \ref{section:Comparison}). The purpose of an entry-bounded 
graph is to represent situations in which the internal reasoning of the graph is stated, but the actual resulting beliefs in a 
distribution depend on the \enquote{input} beliefs provided by the user. Connectedness of a graph contrasts the 
relation coverage deduced from constraints with the existence of connections within the graph. In particular, we 
can distinguish hidden connections,  which may reflect a user that is providing constraints 
that are not reflected by the structure of the graph. 
 
Let us now focus on comparing 
the nature of a relation induced from the constraints and the nature defined by the labeling. 
The presented definitions could be used to verify whether the relation labels are in some way reflected by the constraints
or, if possible, to assign labels to relations when they are missing. A possible - though not the only - way to do so 
is shown in the definition below. 

\begin{definition}
Let $X = (\graph, \lab, \con)$ be a consistent epistemic graph. 
We say that $\lab$ is \textbf{(strongly) consistent} if for every $(\xa, \xb) \in \arcs(\graph)$, the following holds:
\begin{itemize}
\item if $+ \in \lab((\xa, \xb))$, then $(\xa,\xb)$ is arbitrary (strongly) supporting

\item if $- \in \lab((\xa, \xb))$, then $(\xa,\xb)$ is arbitrary (strongly) attacking

\item if $* \in \lab((\xa, \xb))$, then $(\xa,\xb)$ is arbitrary (strongly) dependent

\item if $\lab((\xa, \xb)) = \emptyset$, then $(\xa,\xb)$ is arbitrary unspecified 
\end{itemize} 

We say that $\lab$ is \textbf{monotonic consistent} if for every $(\xa, \xb) \in \arcs(\graph)$, the following holds:
\begin{itemize}
\item if $+ \in \lab((\xa, \xb))$, then $(\xa,\xb)$ is arbitrary positive monotonic  

\item if $- \in \lab((\xa, \xb))$, then $(\xa,\xb)$ is arbitrary  negative monotonic  

\item if $* \in \lab((\xa, \xb))$, then $(\xa,\xb)$ is arbitrary non--monotonic dependent 
\end{itemize}
\end{definition}

In this case, we could either use the set $\{+,-\}$ to denote subtle relations, or introduce a new label in order to avoid 
ambiguity. We also observe that in practice, every relation can be arbitrary unspecified, given that one can decide to test 
relations against the empty set of constraints. 

These approaches can be further refined in the future by putting restrictions on how the $Z$ and $F$ sets are chosen, 
imposing certain ranking on the relations (for example, if a relation is seen as strongly supporting and not strongly attacking, 
strong support could take precedence) and/or making the label conditions even stronger (for example, we can demand that 
$\lab((\xa, \xb)) = \emptyset$ iff it $(\xa, \xb)$ unspecified w.r.t. every $(Z, F)$ pair). 
 
\begin{example}
We can consider the analysis performed in Example \ref{ex:rellabel} to show that the labeling proposed 
for the graph from Example \ref{ex:locglob} is strongly consistent with the analyzed set of constraints. The same analysis also 
shows that it is not the 
only possible consistent labeling. Following the analysis in Example \ref{ex:monoton}, we can also argue 
that a labeling that assigns $*$ to every relation would be more adequate based on the monotonicity analysis. 
We can observe that the labeling for the graph from Example \ref{ex:monoton} is monotonic consistent with the assumed constraint. 
\end{example}

The \enquote{quality} of our epistemic graph, particularly in terms of boundedness, can affect our choice of how 
the graph is evaluated. 
For instance, the less we know about the graph, the more risky credulous reasoning becomes. 
In turn, connectedness can affect choice of arguments in applications, such as persuasive dialogues. It can 
be used to both highlight ineffective relations that should not be taken into consideration as system moves as well as 
incomplete relations that, if used by the user, could lead to situations in which the system cannot decide what to do next.  
We will come back to this in Section \ref{section:CaseStudy}. 

Unfortunately, there is also the issue of inconsistent relations labelings. Natural language arguments are often 
enthymemes and the labels we would obtain through instantiating the graph are not necessarily the ones that 
the users would recognize. The fact that the personal views or knowledge of a given agent affects 
their decoding of the graph has already been observed in \cite{PolbergHunter17}. Recovering consistency of the relation 
labelings is, however, not trivial. 

While one can investigate the constraints and override the graph labeling to force consistency, 
it is unclear what methods would be optimal for allowing information in the graph to override or delete the information in the constraints. Furthermore, one could also consider cases where both parts of the graph and the constraints are sacrificed, as well as where neither, and the existence of inconsistency becomes an additional piece of information that we take advantage of. The actual chosen strategy, as well as the criteria by which it is judged, can depend on the application and the methods with which epistemic graphs are combined.

We can consider criteria such as accuracy or popularity when determining whether the graph or constraints (or possibly both) should be adjusted. For instance, argument and relation mining methods are of varying accuracy, and we can expect that the epistemic constraint mining methods will perform differently depending on the quality and amount of available data. Given conflicting graphs and constraints, one can therefore select which one to prioritize and which one to adjust depending on how much we can trust them to be an accurate representation of, for example, the reasoning patterns and knowledge of (sets of) agents. Concerning popularity, we observe that the graphs and the constraints extracted from a single source of knowledge can differ between agents performing the extraction. One can therefore sacrifice parts of graphs or constraints leading to inconsistency depending on how unpopular or unlikely they are w.r.t. the given population.

There are also situations in which keeping both the graph and the constraints, despite the issues between them, could be informative. An argument graph instantiated with a given structured argumentation approach from, for instance, a legal text, can be viewed as a normative representation of a given problem. It can later be augmented with the constraints of a given agent, which offer a more subjective representation. The dichotomy between the two could be a source of information of its own and, for instance, serve as a measure of how reliable or reasonable a given agent is. Favouring either of the perspectives could also be used as guidance as to whether the graph or the constraints should be sacrificed in the face of contradiction.

These are only few possible scenarios, and we leave investigating consistency retrieving strategies for future work.

\subsection{Epistemic Semantics}  
\label{sec:epsem}
 
Epistemic graphs offer us a number 
of ways in which we can decide how much 
a given argument should be believed or disbelieved depending on the remaining arguments. 
Evaluating the graph and deciding what probabilities should be assigned to (sets of) arguments is the role of the epistemic 
semantics:

\begin{definition}
Let $X = (\graph, \lab, \con)$ be an epistemic graph. An 
\textbf{epistemic semantics} associates $X$ with a set ${\cal R} \subseteq \dist(\graph)$, 
where $\dist(\graph)$ is the set of all belief distributions over $\graph$. 
\end{definition} 

Although the main aim of the 
epistemic semantics is to select those probability distributions that satisfy our requirements, one can also 
enforce additional restrictions for refining the sets of acceptable distributions, on which we will focus in this section.
First of all, the simplest possible semantics is the one that associates 
a given graph with the set of distributions satisfying its constraints: 

\begin{definition}
For an epistemic graph $(\graph, \lab, \con)$, a distribution $\prob \in \dist(\graph)$ meets 
the \textbf{satisfaction semantics} iff $\prob \in \sat(\con)$. 
\end{definition}

Given that an inconsistent graph is not particularly interesting, we will aim at specifying 
epistemic graphs that have consistent constraints.
However, we would like to note that this may not always be possible and that inconsistency does not necessarily mean 
that the constraints are not rational. For example, the stable semantics \cite{CosteMarquis:2007} 
for argumentation graphs does not always produce
any extensions, and this is a result of the restrictive nature of this semantics. We can therefore expect that epistemic graphs 
aiming to emulate this may have inconsistent sets of constraints. 
 
Various properties which can be quite useful concern minimizing or maximizing certain aspects of a distribution. 
Similarly as in other types of argumentation semantics, we can aim to 
maximize or minimize the set of arguments that are believed up to any degree, disbelieved up to any degree 
or undecided.  
We can also consider the information ordering, such as the one used in \cite{BrewkaESWW13}, 
which maximizes or minimizes belief and disbelief together. We can therefore introduce the following means of comparing 
distributions:
 
\begin{definition}
Let $X = (\graph, \lab, \con)$ be an epistemic graph and $\prob, \prob' \in \dist(\graph)$ be probability distributions. 
We say that:
\begin{itemize}
\item $\prob \lesssim_A \prob'$ iff $\{\xa \mid \prob(\xa) >0.5 \} \subseteq \{\xa \mid \prob'(\xa) > 0.5 \}$
\item $\prob \lesssim_R \prob'$ iff $\{\xa \mid \prob(\xa) <0.5 \} \subseteq \{\xa \mid \prob'(\xa) < 0.5 \}$
\item $\prob \lesssim_U \prob'$ iff $\{\xa \mid \prob(\xa) =0.5 \} \subseteq \{\xa \mid \prob'(\xa) = 0.5 \}$
\item $\prob \lesssim_I  \prob'$ iff $\{\xa \mid \prob(\xa) >0.5 \} \subseteq \{\xa \mid \prob'(\xa) > 0.5 \}$ 
and $\{\xa \mid \prob(\xa) <0.5 \} \subseteq \{\xa \mid \prob'(\xa) < 0.5 \}$
\end{itemize}
\end{definition}
We will refer to these orderings as acceptance, rejection, undecided and information orderings.  

These approaches can be further refined to 
take the actual degrees into account as well. For example, in some scenarios a distribution s.t. $\prob(\xa) = \prob(\xb) = 1$ 
and $\prob(\xc) = 0.49$ might be preferable to one s.t. $\prob(\xa) = \prob(\xb) = \prob(\xc) = 0.51$, 
even if the actual number of believed arguments is smaller. Thus, we also consider the belief maximizing and 
minimizing approaches, based on the notion of entropy: 

\begin{definition}
For a probability distribution $\prob$, the \textbf{entropy} $H(\prob)$ of $\prob$ is defined as
\begin{align*}
        H(\prob) = -\sum_{\Gamma\subseteq \nodes(\graph)} \prob(\Gamma) log \prob(\Gamma)
\end{align*}
with $0 \log 0 =0$. 
\end{definition} 

The entropy measures the amount of indeterminateness of a probability distribution $\prob$. 
A probability function $\prob_{1}$ that describes absolute certain knowledge, i.\,e.\ $\prob_{1}(\Gamma)=1$ 
for some $\Gamma\subseteq \nodes(\graph)$ and $\prob_{1}(\Gamma')=0$ for every other 
$\Gamma'\subseteq \nodes(\graph)$, yields 
minimal entropy $H(\prob_{1})=0$. 
The uniform probability function $\prob_{0}$ with 
$\prob_{0}(\Gamma)=\frac{1}{|2^{\nodes(\graph)}|}$ for every 
$\Gamma\subseteq \nodes(\graph)$ yields maximal entropy 
$H(\prob_{0})=  -\log \sfrac{1}{|2^{\nodes(\graph)}|}$. 
Hence, entropy is minimal when we are completely certain about the possible world, and entropy is maximal when we are completely uncertain about the possible world. 


\begin{definition}
Let $X = (\graph, \lab, \con)$ be an epistemic graph and $\prob, \prob' \in \dist(\graph)$ be probability distributions. 
We say that  $\prob \lesssim_B \prob'$ iff $H(\prob') \leq H(\prob)$.
\end{definition}

We will refer to the above as belief ordering.

Given that the purpose of an epistemic semantics is to grasp various optional properties whenever and however they 
are needed, a new semantics can be defined \enquote{on top} of a previous semantics, such as the satisfaction semantics.  
We can therefore propose the following, parameterized definition:

\begin{definition}
\label{def:paramsem}
Let $(\graph, \lab, \con)$ be an epistemic graph and ${\cal R}$ the set of distributions associated with it according to a given 
semantics $\sigma$. Let $v \in \{A,R,U,I,B\}$ denote acceptance, rejection, undecided, information or belief. 
A distribution $\prob \in \dist(\graph)$ meets 
the \textbf{$\sigma{\mhyphen}v$ maximizing (minimizing) semantics} 
iff $\prob \in {\cal R}$ and $\prob$ is maximal (minimal) w.r.t. $\lesssim_v$ among the elements of ${\cal R}$. 
\end{definition}

There are, of course, additional properties we may want to impose in order to refine the distributions produced by the 
constraints associated with a framework. In particular, we may want to limit the values that the distribution may take.  
With some exceptions, most of these restrictions can be expressed as straightforward constraints in the epistemic 
graphs. However, one has to observe that in a sense, they are completely independent of the underlying structure of the graph.  
Thus, we believe it is more appropriate to view them as additional, optional properties:

\begin{definition}
A distribution $\prob$ is:
\begin{itemize}
\item \textbf{minimal} iff for every $\xa \in \nodes(\graph)$, $\prob(\xa) = 0$
\item \textbf{maximal}  iff for every $\xa \in \nodes(\graph)$, $\prob(\xa) = 1$
\item \textbf{neutral}  iff for every $\xa \in \nodes(\graph)$, $\prob(\xa) = 0.5$
\item \textbf{ternary} iff for every $\xa \in \nodes(\graph)$, $\prob(\xa) \in \{0, 0.5, 1\}$
\item \textbf{non--neutral}\footnote{Please note 
that this property was previously referred to as \emph{binary} \cite{PolbergHunter17}.}
 iff for every $\xa \in \nodes(\graph)$, $\prob(\xa) \neq 0.5$ 
\item \textbf{n--valued} iff $\left\vert{ \{ x \mid \exists \xa \in \nodes(\graph), \, \prob(\xa) = x\} }\right\vert = n$
\end{itemize}
\end{definition}

We can therefore observe that there are various ways of refining probability distributions. We have proposed 
a number of ways we can minimize or maximize different aspects of the distributions, and it is possible that 
on certain epistemic graphs they will coincide. However, as the following examples will show, all of the methods 
are in principle distinct.

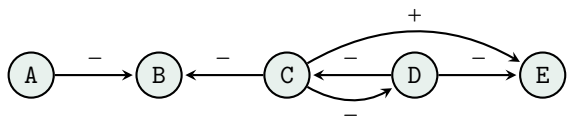
\begin{figure}[!ht]
\centering
  \begin{tikzpicture}
[->,>=stealth,shorten >=1pt,auto,node distance=1.7cm,
  thick,main node/.style={shape=rounded rectangle,fill=darkgreen!10,draw,minimum size = 0.6cm,font=\normalsize\bfseries} 
]

\node[main node] (a) {$\xa$};
\node[main node] (b) [right of=a] {$\xb$};
\node[main node] (c) [right of=b] {$\xc$};
\node[main node] (d) [right of=c] {$\xd$};
\node[main node] (e) [right of=d] {$\xe$};  
 
 \path
	(a) edge node {$-$} (b)
	(c) edge node[swap] {$-$} (b)
	(d) edge node {$-$} (e)
    	(c) edge [bend left] node {$+$} (e)
 	 (c) edge [bend right] node[swap] {$-$} (d)
    (d) edge node[swap] {$-$} (c);
\end{tikzpicture}
\caption{A labelled graph}
\label{fig:episemex5}
\end{figure}

\begin{example}
\label{ex:episem5}
Let us consider the graph depicted in Figure \ref{fig:episemex5} and the following set of constraints $\con$:
\begin{gather*} 
\{ \pprob(\xa) >0.5,\, \pprob(\xb) + \pprob(\xa)\leq 1  \land \pprob(\xb)+ \pprob(\xc)\leq 1, \, 
\pprob(\xc) \neq 0.5, \, \\ \pprob(\xc) + \pprob(\xd)= 1, \, (\pprob(\xc) > 0.5 \land \pprob(\xd) < 0.5) \rightarrow \pprob(\xe) = 0.5\}
\end{gather*}

The types of ternary satisfying distributions of our graph are listed in Table \ref{tab:episemtab2}. 
We also include an analysis of which of them are additionally maximizing or minimizing according to a given criterion. 
We observe that while only belief in arguments are listed, we only have single distributions producing them. 
The patterns set out by $\prob_2$ and $\prob_4$ describe distributions that assign probability $1$ to the 
set of formed of arguments believed with degree $1$ and $0$ to the rest (e.g. in $\prob_2$, the probability of $\{\xa,\xd\}$
would be $1$ and $0$ for everything else). In turn, for $\prob_1$ and $\prob_3$, we would assign probability $0.5$
to the set of arguments that are not disbelieved, and additional $0.5$ to the set of arguments that are believed (e.g. 
for $\prob_1$, $\{\xa,\xc,\xe\}$ and $\{\xa,\xc\}$ would be assigned $0.5$ and all other sets would be assigned $0$). 
This allows us to easily verify belief maximizing and minimizing patterns. 
%

\begin{table}[!ht]
\centering 
\begin{tabular}{|c|c|c|c|c|c|c|c|c|c|c|c|c|c|c|c|} 
\cline{7-16}
\multicolumn{6}{c|}{} & \multicolumn{5}{c|}{Maximizing} & \multicolumn{5}{c|}{Minimizing } \\ \cline{2-16}
\multicolumn{1}{c|}{} & $\xa$	& $\xb$	& $\xc$	& $\xd$	& $\xe$ &  $\lesssim_A$ &  $\lesssim_R$ &  $\lesssim_I$ &  $\lesssim_U$ &  $\lesssim_B$ &   $\lesssim_A$ &  $\lesssim_R$ &  $\lesssim_I$ &  $\lesssim_U$ &  $\lesssim_B$\\
\hline
$\prob_1$	&	1	&	0	&	1	&	0	&	0.5	&	\cm & \cm & \cm& \cm & \xm & \cm & \cm & \cm & \xm &\cm\\
$\prob_2$	&	1	&	0	&	0	&	1	&	0	&	\xm & \cm & \cm& \xm & \cm &\cm & \xm & \xm & \cm &\xm \\
$\prob_3$	&	1	&	0	&	0	&	1	&	0.5	&	\xm & \xm & \xm& \cm & \xm & \cm & \cm & \cm & \xm &\cm \\
$\prob_4$	&	1	&	0	&	0	&	1	&	1	&	\cm & \xm & \cm& \xm & \cm &\xm & \cm & \xm & \cm & \xm \\ 
\hline
\end{tabular} 
\caption{Types of probability distributions meeting the ternary and satisfaction semantics from Example \ref{ex:episem5} and 
their conformity to given maximizing and minimizing semantics.}
\label{tab:episemtab2}
\end{table} 
\end{example}

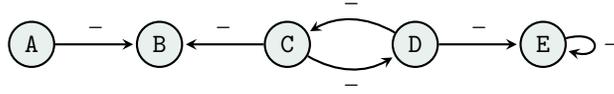
\begin{figure}[!ht]
\centering
  \begin{tikzpicture}
[->,>=stealth,shorten >=1pt,auto,node distance=1.7cm,
  thick,main node/.style={shape=rounded rectangle,fill=darkgreen!10,draw,minimum size = 0.6cm,font=\normalsize\bfseries} 
]

\node[main node] (a) {$\xa$};
\node[main node] (b) [right of=a] {$\xb$};
\node[main node] (c) [right of=b] {$\xc$};
\node[main node] (d) [right of=c] {$\xd$};
\node[main node] (e) [right of=d] {$\xe$};  
 
 \path
	(a) edge node {$-$} (b)
	(c) edge node[swap] {$-$} (b)
	(d) edge node {$-$} (e)
    	(e) edge [loop right] node {$-$} (e)
 	 (c) edge [bend right] node[swap] {$-$} (d)
    (d) edge [bend right] node[swap] {$-$} (c);
\end{tikzpicture}
\caption{A conflict--based argument graph}
\label{fig:af3}
\end{figure} 

\begin{example}
Consider the graph from Figure \ref{fig:af3} and the following set of constraints $\con$:
\begin{itemize}
\item $\varphi_1 \formis \pprob(\xa) > 0.5$

\item $\varphi_2 \formis (\pprob(\xb) > 0.5 \leftrightarrow (\pprob(\xa) < 0.5 \land \pprob(\xc) < 0.5)) \land 
(\pprob(\xb) < 0.5 \leftrightarrow (\pprob(\xa) > 0.5 \lor \pprob(\xc) > 0.5))$

\item $\varphi_3 \formis (\pprob(\xc) > 0.5 \leftrightarrow \pprob(\xd) < 0.5) \land 
(\pprob(\xc) < 0.5 \leftrightarrow \pprob(\xd) > 0.5)$

\item $\varphi_4 \formis (\pprob(\xe) > 0.5 \leftrightarrow (\pprob(\xe) < 0.5 \land \pprob(\xd) < 0.5)) 
\land 
(\pprob(\xe) < 0.5 \leftrightarrow (\pprob(\xe) > 0.5 \lor \pprob(\xd) > 0.5))$

\item $\varphi_5 \formis \pprob(\xc) > 0.5 \lor \pprob(\xd) > 0.5$
\end{itemize}

We obtain two patterns for ternary satisfying distributions, namely $\prob_1$ s.t. 
$\prob_1(\xa) = \prob(\xc) = 1$, $\prob_1(\xb) = \prob_1(\xd) = 0$, $\prob_1(\xe) = 0.5$, 
and 
$\prob_2$ s.t. 
$\prob_2(\xa) =  \prob_2(\xd) =1$, $\prob_2(\xb) = \prob_2(\xc) =\prob_2(\xe)=0$. 
Both describe distributions that are also 
information minimizing, but only $\prob_1$ fits the undecided maximizing requirements.  
\end{example}

%

\begin{figure}[!ht]
\centering 
  \begin{tikzpicture}
[->,>=stealth,shorten >=1pt,auto,node distance=1.7cm,
  thick,main node/.style={shape=rounded rectangle,fill=darkgreen!10,draw,minimum size = 0.6cm,font=\normalsize\bfseries} 
]

\node[main node] (a) {$\xa$};
\node[main node] (b) [right of=a] {$\xb$};
 
 \path
	(b) edge [loop right] node {$-$} (b);
 	
\end{tikzpicture}
\caption{A labelled graph showing the distinction between undecided minimizing and belief maximizing}
\label{fig:episemex2}
\end{figure}
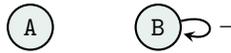

\begin{example}
\label{ex:episem2}
Consider the graph depicted in Figure \ref{fig:episemex2} and the following set of constraints 
$\con = \{\pprob(\xa) \geq 0.5$, $\pprob(\xb)+\pprob(\neg \xb) = 1\}$. 
For this graph, the distribution $\prob'_{1}$ defined via:
\begin{align*}
	\prob'_{1}(\emptyset) & = 0 &	\prob'_{1}(\{\xa\}) & = 0.5 &	\prob'_{1}(\{\xb\}) & = 0 &	\prob'_{1}(\{\xa,\xb\}) & = 0.5
\end{align*}
satisfies the constraints and is both undecided minimizing (only $\xb$ is undecided) and belief maximizing. 
However, the distribution $\prob'_{2}$ defined via: 
\begin{align*}
	\prob'_{2}(\emptyset) & = 0.5 &	\prob'_{2}(\{\xa\}) & = 0 &	\prob'_{2}(\{\xb\}) & = 0 &	\prob'_{2}(\{\xa,\xb\}) & = 0.5
\end{align*}
is belief maximizing but not undecided minimizing (both $\xa$ and $\xb$ are undecided).
 \end{example}

\subsection{Case Study}
\label{section:CaseStudy}

In order to illustrate how epistemic graphs could be acquired for an application, we consider using them 
as domain models in persuasive dialogue systems. Recent developments in computational argumentation are 
leading to a new generation of persuasion technologies \cite{Hunter16comma}. 
An automated persuasion system (APS) is a system that can engage in a dialogue with a user (the persuadee) in order to 
convince the persuadee to do (or not do) some action or to believe (or not believe) something. 
The system achieves this by putting forward arguments that have a high chance of influencing the persuadee. 
In real-world persuasion, in particular in applications such as behaviour change, presenting convincing arguments, 
and presenting counterarguments to the user's arguments, is critically important. 
For example, for a doctor to persuade a patient to drink less alcohol, that doctor has to give good arguments 
why it is better for the patient to drink less, and how (s)he can achieve this. 

Two important features of an APS are the domain model and the user model, which are closely related, 
and together are harnessed by the APS strategy for optimizing the choice of move in a persuasion dialogue.

\begin{description}

\item[Domain model] This contains the arguments that can be presented in the dialogue by the system, and it also contains the
 arguments that the user may entertain. Some arguments will attack other arguments, and some arguments will support other
 arguments. As we will see, the domain model can be represented by an epistemic graph. 

\item[User model] This contains information about the user that can be utilized by the system in order to choose 
the most beneficial actions. The information in the user model is what the system believes is true about that user.  
A key dimension that we consider in the user model are the beliefs that the user may have in the arguments, and as the dialogue
 proceeds, the model can be updated \cite{Hunter15ijcai} based on the results of the queries and of the arguments posited. 

\end{description}

By using an epistemic graph to represent the domain model, and a probability distribution over arguments to 
represent the user model, we can have a tight coupling of the two kinds of model. Furthermore, the probability
 distribution can be harnessed directly in a decision-theoretic approach to optimize the choice of move \cite{Hadoux17}.

To illustrate the use of epistemic graphs for domain/user modelling, we consider a case study in behaviour change. 
The aim of this behaviour change application is to persuade users to book a regular dental check-up.

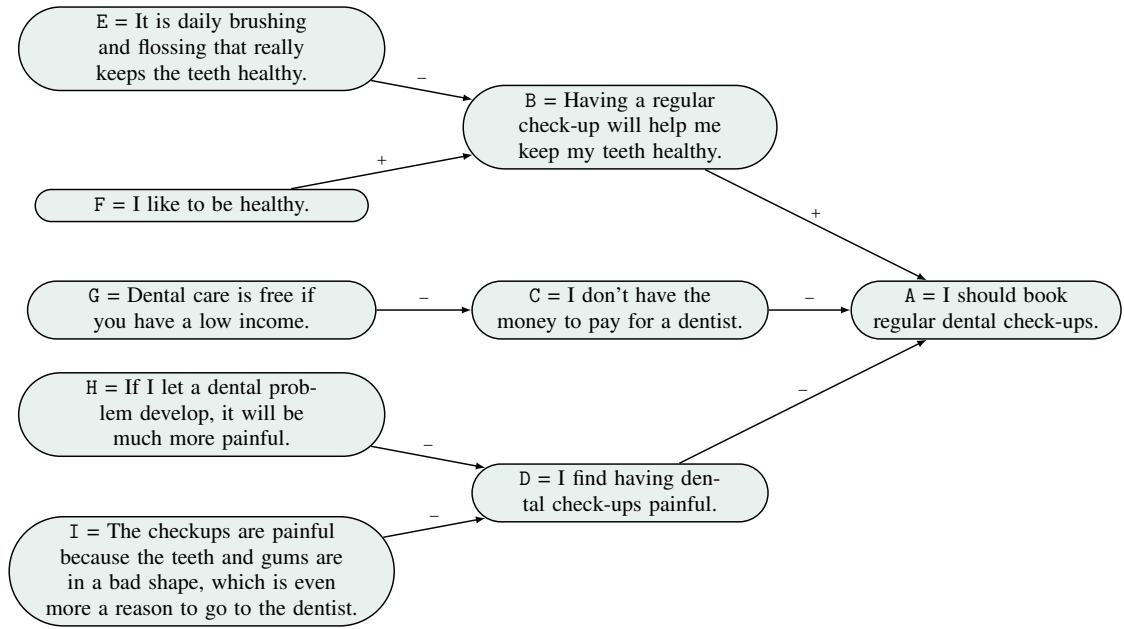
\begin{figure}
\begin{center}
\resizebox{\textwidth}{!}{
\begin{tikzpicture}[->,>=latex,thick, arg/.style={draw,text centered,shape = rounded rectangle,fill=darkgreen!10,font=\large}]
\node[arg] (a) [text width=4.5cm] at (1,0) {$\xa$ = I should book regular dental check-ups.};

\node[arg] (b) [text width=5cm] at (-6,3.5) {$\xb$ =  Having a regular check-up will help me keep my teeth healthy.};
\node[arg] (e) [text width=6cm] at (-14,5) {$\xe$ = It is daily brushing and flossing that really keeps the teeth healthy.};
\node[arg] (f) [text width=6cm] at (-14,2) {$\xf$ = I like to be healthy.};

\node[arg] (c) [text width=5cm] at (-6,0) {$\xc$ = I don't have the money to pay for a dentist.};
\node[arg] (g) [text width=6cm] at (-14,0) {$\xg$ = Dental care is free if you have a low income.};

\node[arg] (d) [text width=5cm] at (-6,-3.5) {$\xd$ = I find having dental check-ups painful.};
\node[arg] (h) [text width=6cm] at (-14, -2) {$\myarg{H}$ = If I let a dental problem develop, it will be much more
 painful.};
\node[arg] (i) [text width=6cm] at (-14,-5) {$\myarg{I}$ = The checkups are painful because 
the teeth and gums are in a bad shape, which is even more a reason to go to the dentist. };

\path[]	(b) edge[] node[above] {$+$} (a);
\path[]	(c) edge[] node[above] {$-$} (a);
\path[]	(d) edge[] node[above] {$-$} (a); 

\path[]	(e) edge[] node[above] {$-$} (b);
\path[]	(f) edge[] node[above] {$+$} (b); 
 
\path[]	(g) edge[] node[above] {$-$} (c); 

\path[]	(h) edge[] node[above] {$-$} (d);
\path[]	(i) edge[] node[above] {$-$} (d); 
\end{tikzpicture}
}
\end{center}
\caption{\label{figure:dental}Epistemic graph for the domain model for a case study on encouraging people to take 
regular dental check-ups.}
\end{figure} 

\begin{example}
Assume we have the graph presented in Figure \ref{figure:dental} and that through, for instance, crowdsourcing data, 
we have learned which constraints should be associated with a given user profile. The obtained domain model(s) can now be used 
in automated persuasion systems, and we assume we are now dealing with a user of such a system whose profile 
lead to the selection of the following constraints in order to describe his or her behaviour:
 
\begin{enumerate}

\item \label{abcd} This constraint states that if $\xb$ is believed or $\xc$ is disbelieved or $\xd$ is disbelieved, 
then $\xa$ is believed and vice versa:
\[ (\pprob(\xb) > 0.5 \lor \pprob(\xc) < 0.5 \lor \pprob(\xd) < 0.5) \leftrightarrow \pprob(\xa) > 0.5
\]

\item \label{ab} This constraint states that if $\xb$ is at least moderately believed then $\xa$ is strongly believed, 
and if $\xb$ is at least strongly believed then $\xa$ is completely believed:
\[ (\pprob(\xb) > 0.65 \rightarrow \pprob(\xa) > 0.8) \land (\pprob(\xb) > 0.8 \rightarrow \pprob(\xa) = 1)
\]


\item \label{ad} This constraint states that if $\xd$ is strongly disbelieved then $\xa$ is at least moderately believed  
\[ \pprob(\xd) < 0.2 \rightarrow \pprob(\xa) > 0.65
\]

\item \label{bf} This constraint states that if $\xf$ is believed then $\xb$ is at least moderately believed 
and if $\xf$ is disbelieved, then so is $\xb$
\[ (\pprob(\xf) > 0.5 \rightarrow \pprob(\xb) > 0.65) \land (\pprob(\xf) < 0.5 \rightarrow \pprob(\xb) < 0.5)
\]

\item \label{cg} This constraint states that disbelief in $\xc$ is proportional to belief in $\xg$
\[ \pprob(\xg) + \pprob(\xc) \leq 1
\]

\end{enumerate}


We can use these constraints together with the 
epistemic graph and probability distribution over the subsets of arguments to model the agent in a persuasion dialogue. We assume 
that we want to persuade the agent to believe argument $\xa$, the more the better. 
The initial belief distribution for such applications can be obtained through crowdsourcing data about 
various participants, and for the purpose of our example 
we assume that a suitable distribution $\prob_0$ denoting the initial belief that we think the agent has in the arguments 
has been obtained. 


Given $\prob_0$, and the need to get a change in belief in $\xa$ so that it is believed, 
we can use constraints \ref{abcd}, \ref{ab} and \ref{ad} as a guide. In other words, we can either 
increase the belief in $\xb$ or decrease the belief in $\xc$ or $\xd$. We observe that using $\xb$ 
can lead to the biggest increase in belief in $\xa$. The effect of using $\xd$ may be smaller and $\xc$ the smallest. 
We observe that none of these arguments (and none of their parents) are default covered. We can thus see this user as
being flexible and open to a discussion. If, for instance, $\xf$ had been default covered by a constraint $\pprob(\xf) < 0.5$, 
then $\xb$ would have been default covered as well and putting forward $\xb$ could be seen as an ineffective move. 

We therefore have three options to explore, and analyzing them is valuable due to the fact that prolonged argument 
exchanges significantly decrease the 
chances of changing an opinion \cite{TanNDNML16}. Consequently, exhausting all possible routes  
may yield negative results, and a persuasion system will need to be able to optimize the choice of dialogue moves.   

\begin{description}
\item[Option $\xb$]  
By looking at the graph and the labels, we may expect that the increase in belief in $\xb$ may be achieved by 
increasing the belief in $\xf$ and/or decreasing the belief in $\xe$. 
However, we observe that even though $\xe$ is stated to be an attacker of $\xb$, analysis of the constraints tells us 
that it cannot be anything else than unspecified and that the labeling is inconsistent. This can potentially be the result of how the user 
was profiled and what constraints for the graph have been created for the profile (s)he fitted. 
By analyzing the constraints, we can observe that increasing the belief in $\xb$ can only be 
done by using argument $\xf$ (constraint \ref{bf}). 

We can therefore choose to rely on the information in the graph or the information in the constraints. 
An optimistic system, which selects the easier and more favourable options, would in this case 
assume that the learned constraints accurately describe the user. It could, for instance, 
take $\prob_{opt}$ as the predicted probability distribution and determine that using $\xf$ 
is a good move\footnote{We note that $\prob_{opt}$ is one of many possible distributions that could be picked 
to satisfy the assumed constraints.}. 
A more pessimistic system, which is convinced that if something can go wrong, it will, 
can consider the learned data to be incomplete. Thus, $\xe$ can be seen as a potential attacker, 
despite being unspecified. The system could take $\prob_{pes}$ as the predicted
distribution and decide not 
to proceed with $\xf$ due to the potential chance of failure\footnote{We note that $\prob_{pes}$ is 
one of many possible distributions that could be picked 
to satisfy the assumed constraints.}. 
 
\begin{center}
\begin{tabular}{|c|c|c|c|c|c|c|c|c|c|c|c| }
\hline
		  & $\xa$ & $\xb$ &  $\xc$ &   $\xd$ & $\xe$ &  $\xf$ &  $\xg$ &  $\xh$ &  $\myarg{I}$ \\
\hline
$\prob_0$ 	 & $0.3$ & $0.4$ &  $0.7$ &  $0.6$ & $0.7$ & $0.45$ &  $0.2$ &  $0.4$ &  $0.3$ 		 \\
$\prob_{opt}$ & $0.85$ & $0.7$ &  $0.7$ &  $0.6$ & $0.7$ & $0.8$ &  $0.2$ &  $0.4$ &  $0.3$ 		 \\
$\prob_{pes}$ & $0.3$ & $0.45$ &  $0.7$ &  $0.6$ & $0.7$ & $0.8$ &  $0.2$ &  $0.4$ &  $0.3$ 		 \\
\hline
\end{tabular}
\end{center}

\item[Option $\xd$] Making argument $\xd$ disbelieved will cause $\xa$ to be believed and based 
on the information in the graph, we can suspect that it can be done through increasing the belief in $\xh$ 
and/or $\myarg{I}$. However, once more we can observe that the graph labeling is not consistent with 
the constraints and the impact of $\xh$ and $\myarg{I}$ on $\xd$ is in fact unspecified. 

We can therefore again choose to rely either on the graph or on the constraints. 
In a similar fashion as before, an optimistic system could decide to carry on 
with presenting either $\xh$ or $\myarg{I}$ and hope to observe a decrease in $\xd$. 
A possible predicted distribution $\prob_{opt}$ associated with presenting $\myarg{I}$ is seen in the table below.
A pessimistic system 
may choose to abandon this dialogue line due to its potential ineffectiveness (see predicted distribution $\prob_{pes})$.

\begin{center}
\begin{tabular}{|c|c|c|c|c|c|c|c|c|c|c|c| }
\hline
		  	 & $\xa$ & $\xb$ &  $\xc$ &   $\xd$ & $\xe$ &  $\xf$ &  $\xg$ &  $\xh$ &  $\myarg{I}$ \\
\hline
$\prob_0$ 	 & $0.3$ & $0.4$ &  $0.7$ &  $0.6$ & $0.7$ & $0.45$ &  $0.2$ &  $0.4$ &  $0.3$ 		 \\
$\prob_{opt}$ & $0.7$ & $0.4$ &  $0.7$ &  $0.1$ & $0.7$ & $0.45$ &  $0.2$ &  $0.4$ &  $0.9$ 		 \\
$\prob_{pes}$ & $0.3$ & $0.4$ &  $0.7$ &  $0.6$ & $0.7$ & $0.45$ &  $0.2$ &  $0.4$ &  $0.9$ 		 \\
\hline
\end{tabular}
\end{center}

\item[Option $\xc$] 
We can observe that $\xg$ is an attacker of $\xc$ both in the graph and in the constraints. 
While increasing the belief in $\xg$ (and thus decreasing the belief in $\xc$) 
yields the smallest gain in $\xa$, it may be seen as the safest way to go given the contrast between the 
constraints and the graph. Consequently, both the optimistic and the pessimistic systems can have similar 
predictions on this route. 

\begin{center}
\begin{tabular}{|c|c|c|c|c|c|c|c|c|c|c|c| }
\hline
		  & $\xa$ & $\xb$ &  $\xc$ &   $\xd$ & $\xe$ &  $\xf$ &  $\xg$ &  $\xh$ &  $\myarg{I}$ \\
\hline
$\prob_0$ 	 & $0.3$ & $0.4$ &  $0.7$ &  $0.6$ & $0.7$ & $0.45$ &  $0.2$ &  $0.4$ &  $0.3$ 		 \\
$\prob_{opt}$ & $0.55$ & $0.4$ &  $0.4$ &  $0.6$ & $0.7$ & $0.45$ &  $0.6$ &  $0.4$ &  $0.3$ 		 \\
$\prob_{pes}$ & $0.55$ & $0.4$ &  $0.4$ &  $0.6$ & $0.7$ & $0.45$ &  $0.6$ &  $0.4$ &  $0.3$ 		 \\
\hline
\end{tabular}
\end{center}

\end{description}
 
Note, we do not consider here how the precise value is picked for the update in the belief in each argument. 
We direct the interested reader to our work on updates in epistemic graphs \cite{HunterPP18} 
and other relevant materials \cite{Hunter15ijcai,HunterPotyka17}. We also do not explicitly consider the issue of verifying predicted 
distributions, but note that it can be done by, for instance, querying the user during a dialogue \cite{Hunter15ijcai}.
\end{example}

The above case study illustrates how the framework in this paper can be incorporated in a user model, 
and then used to guide the choice of moves in a persuasion dialogue.   

\section{Related Work}
\label{section:Comparison}

The constraints in epistemic graphs quite naturally generalize the epistemic postulates 
\cite{Thimm:2012,Hunter:2013,Hunter:2014,PolbergHunter17}. 
Given the fact that in the epistemic 
graphs we can decide whether a given property should hold for a particular argument or not, the desired postulate needs 
to be repeated for every element of the framework. Nevertheless, the general method is straightforward, and using our 
approach we can elevate the classical postulates for conflict--based frameworks to a much more general setting. 
In this section we will focus on describing in more 
detail further argumentation approaches which satisfy at least some of the requirements we have stated in the introduction, 
and consider some other relevant works.  

\subsection{Weighted and Ranking--Based Semantics}

There is a wide array of computational models 
of argument that allow for modelling  
argument weights or strengths 
\cite{Amgoud2013,AmgoudBenNaim16b,AmgoudBenNaim16a,AmgoudBenNaim17,Amgoud16kr,
AmgoudBNDV17,Bonzon16,CayrolLS05,CayrolLS05b,LeiteMartins11,Rago16,Costa-Pereira:2011}, which offer a more fine--grained 
alternative for Dung's approach. 
Some of these works also permit certain forms of support or positive influences on arguments 
\cite{CayrolLS05, AmgoudBenNaim17,AmgoudBenNaim16a,Rago16,LeiteMartins11}. 
Given certain 
structural similarities between these approaches and epistemic semantics, it is therefore natural to compare them. 

Although in both cases what we receive can be seen as \enquote{assigning numbers from $[0,1]$} to 
arguments (either as side or end product), probabilities in the epistemic approach are interpreted as belief, 
while weights remain abstract and open to a number of possible instantiations. The meaning that is assigned to the values 
is derived from the structure of the graph and comparing weightings or rankings between different graphs 
can distort the picture. For instance, given one dense and one sparse graph, it is possible 
that the highest grade achieved by any argument in the former graph is the same as the lowest grade achieved  
in the latter graph. Given the comparative grades w.r.t. other arguments in the graph, we can therefore make different 
judgments about the arguments, which has both its negative as well as positive aspects. 
In turn, the belief and disbelief interpretation of probabilities would more uniformly point to a decision.

We also need to note that many of the postulates set out in the weighted 
and ranking--based methods are, by design, counter--intuitive in the epistemic approach, even though they can 
be perfectly applicable in other scenarios. We can for instance consider the principles from \cite{AmgoudBenNaim17}. 
\textit{(Bi-variate) Independence} states that ranking between two arguments should be independent of any other 
argument that is not 
connected to either of them, and any hidden connected epistemic graph would violate this. The same holds for 
\textit{(Bi-variate) Directionality}, which forces the rank of a given argument 
to depend only on arguments connected to it through a directed path. 
\textit{(Bi-variate) Equivalence} would 
tell us that the strength of the argument depends only on the strength of its parents and not arguments that it attacks or supports. 
This is also clearly not something we intend to force in epistemic graphs. 
All further postulates, such as how an increase or decrease in beliefs in attackers (or supporters) should be matched 
with an appropriate decrease or increase in the belief of the target argument, can, but do not have to, be satisfied 
a given graph. This can be caused either by the constraints themselves simply not adhering to a given axiom on purpose,
or by the constraints being possibly not very specific. 
Already a simple formula such as $\pprob(\xa) > 0.5 \rightarrow \pprob(\xb) \leq 0.5$, which embodies one of the core 
concepts of the classical epistemic approach \cite{Thimm:2012,Hunter:2013,Hunter:2014}, 
violates what is referred to as \textit{Weakening} and 
\textit{Strengthening}. The list continues, however, it should not be taken as a criticism of the weighted or epistemic 
approaches, but only as a highlight of striking conceptual differences.

Another major difference between the epistemic graphs and the weighted or ranking semantics is that in the latter, 
the patterns set out by the semantics have to be global, which leads to side effects not desirable in the epistemic approach. 
In particular, 
two arguments supported and attacked by the same sets of arguments will need to be assigned the same value  
(assuming their initial weights or weights assigned to relations are similar, if applicable). 
In other words, it would be contrary to the intuitions of the weighted approach to have e.g. an attack 
relation $(\xa, \xb)$ described with a constraint $\pprob(\xa) + \pprob(\xb) = 1$ to co-exist in the same 
graph with another attack relation $(\xc,\xd)$ described through $\pprob(\xc) > 0.5 \leftrightarrow \pprob(\xd) \leq 0.5$. 
The first constraint is more specific than the second one and describes the attack relation more closely. Although this 
is generally a desirable thing, it might not be realistic. For instance, when sourcing epistemic graphs and their constraints from 
participants, we have no guarantees that every argument and relation in the graph will be described with the same 
quality and consistency. Forcing a uniform modelling would make us either create specific constraints even for 
parts of the graph for which the data does not support this, or create general constraints even for parts where a better 
description is available. Epistemic graphs aim to bypass this by allowing varying quality of constraints to be used. 
 
Another property of the weighted and ranking semantics (and that is not enforced in the epistemic approach) is that 
given the values of the parents, a single value of the target is returned. This may be a restriction if we want the 
flexibility to express a margin of error or  vagueness.  Depending on how the constraints are defined in the epistemic
 approach, we can force the target to take on a single probability as well as allow it any value from a given range. 
Consequently, we have a certain form of control over specificity in the epistemic graphs. A more relaxed approach 
can be useful in modelling imperfect agents or incomplete situations, and such tasks can pose certain difficulties 
to the weighted semantics.
%
 
In conclusion, we can observe that despite certain high--level similarities, there are significant differences 
between the weighted and epistemic approaches. Although one can argue that it is possible to represent 
certain weighting functions as constraints and the other way around, particularly if multiplication or division
were allowed in the latter, we would either obtain constraints that violate 
the meaning of epistemic probabilities or semantics that do not conform to the required axioms.

\subsection{Abstract Dialectical Frameworks}  
\label{sec:ef-adf}

Epistemic graphs share certain similarities with abstract dialectical frameworks (ADFs)
\cite{Strass13a,Strass13,BrewkaESWW13,StrassWallner15,thesis:polberg,Polberg16,puhrer15}, 
particularly when it comes to their ability to express a wide range of relations between arguments.  
Before we compare the two approaches, we 
briefly review ADFs and some of their semantics\footnote{We note we use the propositional representation of ADFs \cite{BrewkaESWW13}.}.

\begin{definition}
\label{def:funcadf}
An \textbf{abstract dialectical framework} (ADF) is a tuple $(\graph, \lab, \acc)$, where
$(\graph, \lab)$ is a labelled graph and 
$\acc = \{\acc_{\xa}\mid \acc_{\xa}$ is a propositional formula 
over $\parent(\xa)\}_{\xa \in \nodes(\graph)}$ is a set of \textbf{acceptance conditions}. 
\end{definition}
 
In the labeling--based semantics for ADFs, we use three--valued interpretations which 
assign truth values $\{\tvt, \tvf, \tvu\}$ to arguments that are compared according to
precision (information) ordering: $\tvu \leq_i \tvt$ and $\tvu \leq_i \tvf$. 
The pair $(\{\textbf{t}, \textbf{f}, \textbf{u}\}, \leq_i)$ forms a complete meet--semilattice with the meet 
operation $\sqcap$ assigning values in the following way:
 $\textbf{t} \, \sqcap \, \textbf{t} = \textbf{t}$, $\textbf{f} \, \sqcap \, \textbf{f} = \textbf{f}$ and $\textbf{u}$ in all other cases.
These notions can be easily extended to interpretations. For two interpretations $v$ and $v'$ on $\nodes(\graph)$, 
$v \leq_i v'$ iff for every argument $\xa \in \nodes(\graph)$, $v(\xa) \leq_i v'(\xa)$. 
In the case $v$ is three and $v'$ two--valued (i.e. contains no $\tvu$ mappings), 
we say that $v'$ extends 
$v$\footnote{This means that the elements mapped originally to $\tvu$ are now assigned either $\tvt$ or $\tvf$.}. The
set of all two--valued interpretations extending $v$ is denoted $\lbrack v \rbrack_2$. 
Given an acceptance condition $\acc_{\xa}$ for an argument $\xa \in \nodes(\graph)$ and an interpretation $v$, 
we define a shorthand $v(\acc_{\xa})$ as value of $\acc_{\xa}$ for $v^\tvt \cap \parent(\xa)$.

\begin{definition}
Let $D=(\graph, \lab, \acc)$ be an ADF and $v$ a three-valued interpretation on $\nodes(\graph)$. 
The \textbf{three--valued characteristic operator} of $D$ is a function s.t.
$\Gamma(v) = v'$ where 
$v'(\xa) = \bigsqcap_{w \in \lbrack v \rbrack_2} w(\acc_{\xa})$ for $\xa \in \nodes(\graph)$.  
An interpretation $v$ is:
\begin{itemize} 
\item a \textbf{complete} labeling iff $v = \Gamma (v)$. 
\item a \textbf{preferred} labeling iff it is $\leq_i$--maximal admissible labeling. 
\item a \textbf{grounded} labeling iff it is the least fixpoint of $\Gamma$.
\end{itemize}
\end{definition}

\begin{figure}[!ht] 
\begin{center}
\hspace*{-0.8cm}
\begin{subfigure}[b]{0.37\textwidth} 
\centering  
  \begin{tikzpicture}
[->,>=stealth,shorten >=1pt,auto,node distance=1.7cm,
  thick,main node/.style={shape=rounded rectangle,fill=darkgreen!10,draw,minimum size = 0.6cm,font=\normalsize\bfseries},
condition/.style={rectangle,fill=none,draw=none,minimum size = 1cm,font=\normalsize\bfseries}]
 
\node[main node] (e) at (0,0) {$\xe$};
\node[main node] (a)  at (0,-2) {$\xa$};
\node[main node] (b) at (2,0){$\xb$};
\node[main node] (c) at (1,1.5) {$\xc$};
\node[main node] (d) at (2,-2)   {$\xd$};

\node[condition](ca) [left of= a,xshift=1.1cm] {$\xe$};
\node[condition](cb) [right of= b, xshift=-0.4cm] {$\xd \lor (\neg \xc \land \xe)$};
\node[condition](cc) [above of= c, yshift=-1.1cm] {$\neg \xe$};
\node[condition](cd) [right of= d, xshift= -0.7cm] {$\neg \xa \lor \neg \xe$};
\node[condition](ce) [left of= e, xshift=0.9cm] {$\xa \land \xb$};

 \path
	(e) edge  node{$-$} (d)
	(a) edge  node{$-$} (d)

	(e) edge  node{$+$} (a)
	(d) edge  node[swap]{$+$} (b)
	(c) edge   node{$-$} (b)
	(e) edge  [bend left] node{$+$} (b)
	(e) edge  node{$-$} (c)	
(a) edge  [bend left] node{$+$} (e)
(b) edge  node{$+$} (e);
\end{tikzpicture}
\caption{Example of an ADF}
\label{fig:grd}
\end{subfigure}
\begin{subfigure}[b]{0.55\textwidth}
\centering 
\begin{tabular}{|c|c|c|c|c|c|c|c|c|} 
\hline
			& $\xa$	& $\xb$	& $\xc$  	& $\xd$	& $\xe$   	& CMP 	& PREF	& GRD 	  \\
\hline
$v_1$		& $\tvu$	& $\tvu$	& $\tvu$ & $\tvu$	& $\tvu$  	& \cm	& \xm	& \cm		\\
$v_2$		& $\tvt$	& $\tvt$	& $\tvf$ 	& $\tvf$	& $\tvt$ 	 	& \cm	& \cm	& \xm		\\
$v_3$		& $\tvf$  & $\tvt$	& $\tvt$ 	& $\tvt$	& $\tvf$ 	 	& \cm	& \cm	& \xm	 	\\
\hline
\end{tabular}
\caption{Labelings of the presented ADF}
\label{tab:adflab}
\end{subfigure} 
\end{center}
\caption{Example ADF and its complete (CMP), preferred (PREF) and grounded (GRD) labelings.}
\end{figure} 
%
%
%
%
%
%
%

\begin{example}
\label{ex:adflab} 
The admissible, complete, preferred and grounded labelings of the ADF 
depicted in Figure \ref{fig:grd} are visible in Table \ref{tab:adflab}.
%
\end{example}
 
Having an acceptance condition for each node is similar in spirit to having constraints for epistemic graphs. 
Furthermore, both frameworks can handle 
relations that are positive, negative, or neither. 
However, there are some fundamental differences between ADFs and epistemic graphs. 

The acceptance conditions 
can tell us whether an argument is accepted or rejected based on the acceptance of its parents. In contrast, 
the epistemic constraints can produce probability assignments in the unit interval that depend on the degrees of belief in 
other arguments, which offers a much more fine--grained perspective. 
It also allows epistemic graph to easily handle some forms of support, such as the
abstract or deductive supports, 
which are normally too weak to be expressed  in ADFs or require certain translations 
\cite{Polberg16,Polberg17}. The constraints also allow us 
to define a range of values that an argument may take on in given circumstances as well as a single particular value, and 
thus offers more flexibility in modelling the acceptability of an argument.  
Furthermore, in epistemic graphs the constraints are assigned per graph, not per argument.  
We can handle situations where the belief in one argument might depend not just on its parents, but also on other arguments 
for reasons known only to the agent, without necessarily forcing edges in the graph to be modified. 
Additionally, the completeness of acceptance conditions is obligatory in ADFs, while the completeness of epistemic 
constraints is optional and requiring it should be motivated by a given application. 
This control may be useful in user modelling, where we are not yet sure how a given argument and its associated relations
are perceived by the user.

These differences show that epistemic graphs are quite distinct from ADFs. 
Nevertheless, it is possible for epistemic graphs to model ADFs. We will show how this can be achieved based on 
an example. 

\begin{example}
\label{ex:adfepigraph}
Let us come back to the ADF from Example \ref{ex:adflab} and Figure \ref{fig:grd}. We will now show how acceptance
conditions can be transformed into constraints s.t. the labelings extracted from the probabilistic distributions 
correspond to the ADF labelings under a given semantics.
 
Let us focus on argument $\xe$. If we were to create a truth table for its 
condition $\xa \land \xb$, we would observe that if $\xe$ is to be accepted, then $\xa$ and $\xb$ have to be true, 
and if $\xe$ is to be rejected, then $\xa$ or $\xb$ has to be false. Taking into account 
the nature of the complete semantics, this rather straightforwardly translates to the 
following constraints: argument $\xe$ is believed iff
$\xa$ and $\xb$ are believed; and, argument $\xe$ is disbelieved iff either $\xa$ or $\xb$ is disbelieved.

What is therefore happening is that for every (propositional) 
acceptance condition we create two constraints that are the epistemic adaptations 
of the formulas ${\sf X} \leftrightarrow \acc_{\sf X}$ and $\neg {\sf X} \leftrightarrow \neg \acc_{\sf X}$, 
where, assuming that the consequent 
is in a form without nested negations, a positive literal ${\sf Z}$ is transformed into an epistemic atom $\pprob({\sf Z}) > 0.5$ 
and a negative literal $\neg {\sf Z}$ becomes $\pprob({\sf Z}) < 0.5$.   

We can gather such constraints for all arguments into a set $\con$:

\begin{itemize}
\item $(\pprob(\xa) > 0.5 \leftrightarrow \pprob(\xe) >0.5) \land (\pprob(\xa) < 0.5 \leftrightarrow  \pprob(\xe) <0.5)$

\item $(\pprob(\xb) > 0.5 \leftrightarrow  \pprob(\xd) >0.5 \lor (\pprob(\xc) <0.5 \land \pprob(\xe) >0.5)) \land  
 (\pprob(\xb) < 0.5 \leftrightarrow  \pprob(\xd) < 0.5 \land (\pprob(\xc) >0.5 \lor \pprob(\xe) <0.5))$

\item $(\pprob(\xc) > 0.5 \leftrightarrow  \pprob(\xe) <0.5) \land (\pprob(\xc) < 0.5 \leftrightarrow  \pprob(\xe) >0.5)$

\item $(\pprob(\xd) > 0.5 \leftrightarrow  \pprob(\xa) <0.5 \lor \pprob(\xe) < 0.5) \land (\pprob(\xd) < 0.5 \leftrightarrow \pprob(\xa) >0.5 \land \pprob(\xe) >0.5)$

\item $(\pprob(\xe) > 0.5 \leftrightarrow  \pprob(\xa) >0.5 \land \pprob(\xb) > 0.5) \land (\pprob(\xe) < 0.5 \leftrightarrow  \pprob(\xa) <0.5 \lor \pprob(\xb) < 0.5)$
\end{itemize} 

The ternary satisfying distributions of this set are visible in Table \ref{tab:epiadflab}. We observe that by 
transforming the distributions into  labelings that map to $\tvt$ arguments that are believed, $\tvf$ that are disbelieved and $\tvu$ that are neither, we retrieve the complete labelings of our ADF. 
By considering distribution maximal or minimal w.r.t. $\lesssim_I$ we can retrieve the preferred or grounded labelings. 

\begin{table}[!ht]
\centering
\begin{tabular}{|c|c|c|c|c|c|c|c|c|c|} 
\hline
				& $\prob(\xa)$	& $\prob(\xb)$	& $\prob(\xc)$  	& $\prob(\xd)$	& $\prob(\xe)$  	& $\sat(\con)$ 	& Max. $\lesssim_I$ &  Min. $\lesssim_I$ 	 \\
\hline
$\prob_1$		& $0.5$	& $0.5$	& $0.5$ 	& $0.5$	& $0.5$ 	 	& \cm	& \xm	& \cm		\\
$\prob_2$		& $1$	& $1$	& $0$ 	& $0$	& $1$ 	 	& \cm	& \cm	& \xm		\\
$\prob_3$		& $0$  	& $1$	& $1$ 	& $1$	& $0$ 	 	& \cm	& \cm	& \xm	 	\\
\hline
\end{tabular}
\caption{Ternary satisfying and information maximizing/minimizing distributions for the epistemic graph from Example \ref{ex:adfepigraph}.}
\label{tab:epiadflab}
\end{table} 
\end{example} 

This example shows us that it is possible for the epistemic graphs to handle ADFs under the labeling--based semantics, 
even though providing a full translation for any type of condition may be more involved than the presented approach. 
Given the fact that ADFs can subsume a number of different frameworks \cite{Polberg17}, it is also possible for 
the epistemic graphs to express many more approaches to argumentation than we recall here.  
  
There are certain generalizations of ADFs that are relevant in the context of our work. 
In \cite{PolbergDoder14}, a probabilistic version of ADFs has been introduced. However, this work follows the 
constellation interpretation of a probability, not the epistemic one, which leads to significantly different 
modelling \cite{Hunter:2013,PolbergHT17}. In a recent work \cite{Brewka18}, a new version of weighted ADFs has been proposed, 
in which conditions no longer map subsets of parents of a given argument to $in$ or $out$, 
but take values assigned to the parents and state a specific value that should be assigned to the target. 
These values can be abstract entities with some form of ordering between them as well as numbers from the $[0,1]$ 
interval. The information ordering present in the original ADFs is then adopted accordingly and then the definition of the 
existing operator--based 
semantics (admissible, grounded, preferred, complete) remains unchanged.

Despite certain possible overlaps, weighted ADFs are incomparable to epistemic graphs.  
On the one hand, similarly as in original ADFs, condition completeness and limiting the conditions to depend 
only on the parents of a given argument is enforced. Furthermore, unlike epistemic constraints, weighted 
acceptance conditions are very specific in the sense that a given combination of values assigned to a given 
argument leads to a precise, defined outcome. 
Therefore, a constraint stating that if the attacker is believed, then the attackee should be disbelieved 
(we can formalize it e.g. as $\pprob(\xa) > 0.5 \rightarrow \pprob(\xb) < 0.5$), cannot be conveniently 
expressed in weighted ADFs. This is due to the fact that the belief in the target is meant to be a function of beliefs 
of the source, while in epistemic graphs a more general relation is permitted. 
Consequently, there are properties expressible with epistemic graphs, but not with weighted ADFs. 
On the other hand, weighted ADFs are not specialized for handling probabilities, and therefore can take as input
further unspecified values, not only numbers. 
Thus, we can construct scenarios handled by weighted ADFs, but not by epistemic graphs. 
Additionally, even if values from the $[0,1]$ interval are considered, 
for computational reasons they are amended with a special element indicating that a given value is undefined 
and the interpretation of this element is different than the one of neither agreeing nor disagreeing in the epistemic proposal.

\subsection{Constrained Argumentation Frameworks}
\label{sec:caf}  
 
Our proposal shares certain similarities with the constrained argumentation frameworks \cite{CosteMarquisDM06}, 
which permits external requirements among unrelated arguments to be imposed in the framework. 
This constraint would represent certain restrictions that 
(for reasons unknown to the abstract system) are considered desirable by, for example, the user, and which are not necessarily 
reflected by the structure of the graph. 
Although this approach has only been analyzed in the context of attack--based graphs, 
certain positive relationships between arguments could potentially be simulated through the use of propositional formulae
representing the external requirements. Nevertheless, this modelling is targeted mainly at two--valued semantics, 
and thus the framework does not deal with fine--grained acceptability.   

\begin{definition}
Let $PROP_S$ be a propositional language defined in the usual inductive way from a set $S$ of propositional symbols, 
boolean constants $\top$, $\bot$ and the connectives $\neg, \land, \lor, \leftrightarrow$ and $\rightarrow$.
A \textbf{constrained argumentation framework} is a tuple $(\graph, \lab, {\sf PC})$ where $(\graph, \lab)$ is a labelled 
graph s.t. $\lab$ assigns only $-$ to all edges, and ${\sf PC}$ is a propositional formula from $PROP_{\nodes(\graph)}$.
\end{definition} 

Semantics of the graph are primarily defined in terms of sets of arguments that, along with meeting 
the classical extension-based semantics \cite{Dung95}, satisfy the external constraint. 
Such classical semantics can be easily retrieved by the epistemic postulates \cite{PolbergHunter17,Baroni:2011}, 
which themselves are straightforwardly generalized by epistemic graphs. The propositional formula ${\sf PC}$ 
can also be straightforwardly mapped to an epistemic constraint. We will therefore consider an example showing 
how constrained argumentation frameworks can be expressed within epistemic graphs.  

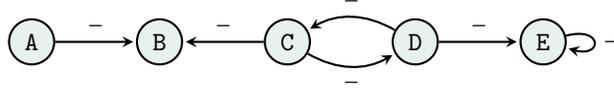
\begin{figure}[!ht]
\centering
  \begin{tikzpicture}
[->,>=stealth,shorten >=1pt,auto,node distance=1.7cm,
  thick,main node/.style={shape=rounded rectangle,fill=darkgreen!10,draw,minimum size = 0.6cm,font=\normalsize\bfseries} 
]

\node[main node] (a) {$\xa$};
\node[main node] (b) [right of=a] {$\xb$};
\node[main node] (c) [right of=b] {$\xc$};
\node[main node] (d) [right of=c] {$\xd$};
\node[main node] (e) [right of=d] {$\xe$};  
 
 \path
	(a) edge node {$-$} (b)
	(c) edge node[swap] {$-$} (b)
	(d) edge node {$-$} (e)
    	(e) edge [loop right] node {$-$} (e)
 	 (c) edge [bend right] node[swap] {$-$} (d)
    (d) edge [bend right] node[swap] {$-$} (c);
\end{tikzpicture}
\caption{A conflict--based argument graph}
\label{fig:cafx}
\end{figure} 

\begin{example} 
\label{ex:cafx}
Consider the graph
depicted in Figure~\ref{fig:cafx} and augmented with the constraint 
${\sf PC} = \neg \xa \lor \xd$. 
The admissible extensions (i.e. extensions in which no arguments attack each other and every argument 
attacking an argument in the set is itself attacked by an element of the set) of the graph (without the constraint) are 
$\{\xa,\xc\}$,$\{\xa, \xd\}$, $\{\xa\}$, $\{\xc\}$, $\{\xd\}$ and $\emptyset$.
Once the constraint is applied, the sets $\{\xa, \xc\}$ and $\{\xa\}$ have to be removed. 
The preferred extensions (i.e. maximal admissible extensions) of the graph are initially $\{\xa, \xc\}$ and $\{\xa,\xd\}$. 
However, if we take the constraint into account, we in fact receive $\{\xc\}$ and $\{\xa, \xd\}$.

Following the method from \cite{PolbergHunter17}, we now create the following set of constraints $\con$:

\begin{itemize}
\item $\pprob(\xa) \geq 0.5$
\item $(\pprob(\xb) >0.5 \rightarrow (\pprob(\xa) < 0.5 \land \pprob(\xc) < 0.5)) \land 
(\pprob(\xb) <0.5 \rightarrow (\pprob(\xa) > 0.5 \lor \pprob(\xc) > 0.5))$

\item $(\pprob(\xc) >0.5 \leftrightarrow \pprob(\xd) < 0.5) \land  
(\pprob(\xc) <0.5 \leftrightarrow \pprob(\xd) > 0.5)$
 
\item $(\pprob(\xe) > 0.5 \rightarrow (\pprob(\xd) < 0.5 \land \pprob(\xe) <0.5)) \land
(\pprob(\xe) <0.5 \rightarrow (\pprob(\xd) > 0.5 \lor \pprob(\xe) >0.5))$ 
\end{itemize} 

In Table \ref{tab:constrlabx} we have listed all the ternary distributions satisfying $\con$. It is easy to see that the sets of believed 
arguments obtained from these distributions coincide with the admissible extensions of our labelled graph. 
The epistemic representation of the ${\sf PC}$ constraint is $\pprob(\xa) \leq 0.5 \lor \pprob(\xd) > 0.5$. 
By adding it to the set $\con$, we obtain the constraint set $\con'$, which excludes distributions 
$\prob_{2}$, $\prob_{3}$, $\prob_{8}$ and $\prob_9$ that corresponded to extensions $\{\xa\}$ and $\{\xa,\xc\}$.  
By enforcing information maximality along with the 
$\con$ and $\con'$ constraints, we obtain either distributions $\prob_{9}$ and $\prob_{13}$ or 
$\prob_4$ and $\prob_{13}$, which are associated with the desired preferred extensions.  
\end{example}  
 
\begin{table}[!ht]
\centering
\begin{tabular}{|c|c|c|c|c|c|c|c|c|c|} 
\cline{9-10}
\multicolumn{8}{c|}{} &\multicolumn{2}{c|}{Max. w.r.t. $\lesssim_I$} \\
\cline{2-10}
\multicolumn{1}{c|}{} & $\prob(\xa)$	& $\prob(\xb)$	& $\prob(\xc)$  & $\prob(\xd)$	& $\prob(\xe)$  
& $\con$ & $\con'$ & $\con$ & $\con'$ \\
\hline
$\prob_1$		& $0.5$		& $0.5$		& $0.5$ & $0.5$		& $0.5$ & \cm	& \cm	& \xm	& \xm	\\
$\prob_2$		& $1$  		& $0.5$		& $0.5$ & $0.5$		& $0.5$	& \cm	& \xm	& \xm	& \xm	\\
$\prob_3$		& $1$  		& $0$		& $0.5$ & $0.5$		& $0.5$	& \cm	& \xm	& \xm	& \xm	\\ 
$\prob_4$		& $0.5$		& $0$		& $1$	& $0$		& $0.5$ & \cm	& \cm	& \xm	& \cm	\\
$\prob_5$		& $0.5$		& $0.5$		& $1$	& $0$		& $0.5$ & \cm	& \cm	& \xm	& \xm	\\
$\prob_6$		& $0.5$		& $0.5$		& $0$ 	& $1$		& $0.5$ & \cm	& \cm	& \xm	& \xm	\\
$\prob_7$		& $0.5$		& $0.5$		& $0$ 	& $1$		& $0$	& \cm	& \cm	& \xm	& \xm	\\
$\prob_8$		& $1$  		& $0.5$		& $1$ 	& $0$		& $0.5$	& \cm	& \xm	& \xm	& \xm	\\
$\prob_9$		& $1$  		& $0$		& $1$ 	& $0$		& $0.5$	& \cm	& \xm	& \cm	& \xm	\\ 
$\prob_{10}$		& $1$  		& $0.5$		& $0$ 	& $1$		& $0.5$	& \cm	& \cm	& \xm	& \xm	\\
$\prob_{11}$		& $1$  		& $0.5$		& $0$ 	& $1$		& $0$	& \cm	& \cm	& \xm	& \xm	\\    
$\prob_{12}$		& $1$  		& $0$		& $0$ 	& $1$		& $0.5$	& \cm	& \cm	& \xm	& \xm	\\
$\prob_{13}$		& $1$  		& $0$		& $0$ 	& $1$		& $0$	& \cm	& \cm	& \cm	& \cm	\\
\hline
\end{tabular}
\caption{Ternary satisfying and information maximizing distributions associated with the sets of 
constraints $\con$ and $\con'$ from Example \ref{ex:cafx}.}
\label{tab:constrlabx}
\end{table}  

\subsection{Other Approaches}

We conclude our comparison with related work by considering two general approaches in knowledge representation and reasoning.

\subsubsection{Constraint Satisfaction Problem} 

Constraint programming \cite{Dechter2003,Rossi:2006,Tsang1993} is a general problem solving paradigm for modelling and solving hard search problems. In essence, our approach comprises of a series of constraints that take both probabilistic and argumentative aspects into account. Thus, the problems we discussed can be captured as constraint satisfaction problems and existing constraint programming solvers can be used to solve these problems. Constraint programming is a general approach and we already noted that there are some general notions in the constraint programming literature that subsume some of our specific notions (such as \emph{eliminating explanations}, cf.\ Definition~\ref{def:coverage}). 
However, our work provides the first proposal for how to turn some aspects of representing and reasoning with beliefs in arguments into a constraint satisfaction problem.
We provide the language for constraints over belief in arguments, the entailment and consequence relations, 
and epistemic semantics.

A concrete formalisation and implementation of our approach using constraint programming technology is part of current work. As a first step in this direction, we have considered performing updates in sub-classes of epistemic graphs as linear optimization problems \cite{HunterPP18}. 

\subsubsection{Bayesian Networks}

A Bayesian network (or a causal probabilistic network) is an acyclic directed graph in which each node denotes a random variable 
(a variable that can be instantiated with an element from some set of events) and each arc denotes causal influence of one 
random variable on another \cite{Pearl2000,Barber2012}. 
Random variables can be used to represent propositions that are either 
\enquote{true} or \enquote{false}. For example, if the random variable is {\tt car-battery-is-flat}, then it 
can be instantiated with the event {\tt car-battery-is-flat}, or the event 
$\neg\mbox{\tt car\mbox{-}battery\mbox{-}is\mbox{-}flat}$.

A key advantage of a Bayesian network is the use of independence assumptions that can be derived from the graph structure. 
These independence assumptions allow for the joint probability distribution for the random variables in a graph to be decomposed 
into a number of smaller joint probability distributions. This makes the acquisition and use of probabilistic information much 
more efficient.

Superficially, there are some similarities between Bayesian networks and epistemic graphs. Both have a graphical 
representation of the influence of one node on another where the nodes can be used to represent statements. Furthermore, 
the influences by one node on another can change the belief in the target node, and this change can be either positive or negative.

However, Bayesian networks and epistemic graphs are significantly different in their underlying representation and in 
the way they work as we clarify here:
\begin{enumerate}
\item A Bayesian network is used with a single probability distribution whereas the constraints associated with an epistemic graph 
allow for multiple probability distributions that satisfy the constraints;

\item A Bayesian network updates a random variable by taking on a specific instantiation, and that there is no longer any 
doubt about that instantiation (e.g. in the case of a random variable being updated by taking on the value \enquote{true}, then there 
is no longer any doubt or uncertainty about the value of the random variable being \enquote{true}), whereas with epistemic graphs, if 
the belief in a node is updated, it can be of any value in the unit interval (e.g. for an argument $\xa$, that is current believed 
to degree $0.7$, we may choose to update it to degree $0.3$), and so this means epistemic graphs can reflect 
uncertainty in updating; and

\item A Bayesian network propagates updates by conditioning, which is a specific kind of constraint (e.g. for a graph with two 
variables $\alpha$ and $\beta$, after updating $\alpha$, the propagated belief in $\beta$ is $\prob(\beta\mid\alpha)$), 
whereas the framework for epistemic graphs provides a rich language for specifying a wide variety of constraints between 
the two variables.
\end{enumerate} 

The motivations for Bayesian networks and epistemic graphs are also different. Bayesian networks are for modelling 
normative reasoning (i.e. they model how we should reason with a set of random variables with given set of influences 
between them). In contrast, epistemic graphs are for modelling non-normative reasoning, and intended to reflect how 
people may choose to reason with the uncertainty concerning arguments. So with epistemic graphs, we may model 
how some people may regard the relative belief in 
a set of arguments, but it does not mean that they are correct in any normative sense, rather it is just a way of 
modelling their perspective or behaviour.

There are also some proposals for capturing aspects of Bayesian networks in argumentation such as qualitative 
probabilistic reasoning \cite{Parsons2004} and \cite{Timmer2017}. As with Bayesian networks, 
they are also concerned with capturing aspects of normative reasoning, and so have a different aim to epistemic graphs.


\section{Discussion}
\label{section:Discussion}


In this paper, we have generalized the epistemic approach to probabilistic argumentation by introducing the notion of 
epistemic graphs which define how arguments influence each other through the use of epistemic constraints. 
We provided an extensive study of properties of graphs and exemplified their potential use in practical applications.  
We have also created a proof theory 
for reasoning with the constraints that is both sound and complete, and 
analyzed various ways in which the constraints can affect arguments and relations 
between them. 
We have also compared our research to other relevant works in argumentation, CSP and Bayesian networks. 
Our proposal meets the requirements postulated in the introduction: 

\begin{description}
\item[Modelling fine--grained acceptability]  
Epistemic graphs can express varying degrees of belief we have in arguments and these beliefs can be 
harnessed and restricted through the use of epistemic constraints, as seen in Section \ref{sec:epistemicgraphs}. 
The beliefs can be easily associated with the traditional notions of acceptance and rejection of arguments \cite{PolbergHunter17} 
and, in contrast 
to more abstract forms of scoring and ranking arguments, provide a clearer meaning of the values associated with 
arguments. 

\item[Modelling positive and negative relations between arguments]  
With epistemic graphs, we can model various types of relations between arguments, including positive, negative or mixed, 
as 
studied in Section \ref{sec:consistentlabel}. 
Furthermore, they can also handle relations marked as group or binary (for example, two attackers need to be believed 
in order for the target to be disbelieved versus at least one attacker needs to be believed for the target to be disbelieved). 
Finally, in our analysis 
of the nature of various relations, we have also discussed how the views on the influence one argument has over 
another change depending on whether local or global perspective is taken into account. 

\item[Modelling context--sensitivity] 
Two structurally similar graphs can be assigned different sets of epistemic constraints. 
An agent is allowed to have different opinions on similar 
graphs and adopt them according to his or her needs, be it caused by the actual content of the arguments, 
agent's preferences or knowledge, 
or the way an agent understands the arguments.
Thus, there is no requirement for the same graphs being evaluated in the same fashion under the 
same epistemic semantics. For example, we can easily create two different sets of constraints for the 
two scenarios considered in Example \ref{ex:context} despite the fact that their formal representations are equivalent. 
Epistemic graphs can also deal with 
restrictions that are not necessarily reflected in the structure of the graph. 

\item[Modelling different perspectives] 
Agents do not need to adhere to a uniform perspective on a given problem. They can perceive arguments and relations 
between them differently, and thus find different arguments believable or not, such as seen in Examples \ref{ex:trains}. 
Furthermore, even arguments sharing some similarities in their views 
can respond differently when put in the same situation. Such behaviour could have been observed in Example \ref{ex:smoking}, 
and it would not be problematic to create constraints that handle rejecting certain arguments differently.   

\item[Modelling imperfect agents] 
The freedom in defining constraints and beliefs in arguments allows agents to express their views freely, independently of 
whether they are deemed rational or not or are strongly affected by cognitive biases. For example, two logically conflicting 
arguments do not need to be accompanied by constraints reflecting this conflict. Furthermore, agents do not necessarily 
need to adhere to various types of semantics \cite{PolbergHunter17}, and epistemic constraints could be used to grasp 
their views more accurately. 

\item[Modelling incomplete graphs] 
An argument graph might not reflect all the knowledge an agent has and that is relevant to a given problem, 
as seen in Example \ref{ex:smoking}.
Consequently, a given argument can be believed or disbelieved 
without any apparent justification, as seen in Section \ref{section:CaseStudy}. It is however not difficult 
to create constraints stating that a given argument should be assigned a particular score. It is also possible to not create 
any constraints at all if it is not known how an agent views the interactions between arguments, and thus provide 
no coverage to arguments or relations, as seen in Section \ref{sec:CoverageAnalysis}. 
\end{description}
 
Although our analysis of epistemic graphs is extensive, there are still various topics to be considered. 
The currently proposed epistemic graph semantics can be further refined in order to take the additional information 
contained in the structure of the graph, but not in the constraints, into account. We could, for example, consider a 
coverage--based family of semantics, where the status assigned to a given argument can depend on the level of coverage 
it possesses. 
%

Another issue we want to explore concerns how the constraints can be obtained. Crowdsourcing opinions on arguments 
is a popular method for obtaining data \cite{Cerutti2014,HunterPolberg17,PolbergHunter17}. 
Such data concerning beliefs in arguments and whether arguments are seen as related
could be analyzed with, for example, machine learning techniques, in order to construct appropriate constraints. 

In the future we would like to explore the use of epistemic graphs for practical applications, in particular for computational 
persuasion. Applying the existing epistemic approach to modelling persuadee's beliefs
in arguments has produced methods for updating beliefs during
a dialogue \cite{Hunter15ijcai,Hunter16sum,HunterPotyka17}, efficient representation and reasoning
with the probabilistic user model \cite{Hadoux16}, modelling uncertainty in belief distributions \cite{Hunter16ecai}, 
for learning belief distributions \cite{HadouxHunter18}, and harnessing
decision rules for optimizing the choice of arguments based
on the user model \cite{Hadoux17}. These methods can be further developed in the context of epistemic 
graphs in order to provide a well understood
theoretical and computationally viable framework
for applications such as behaviour change. For a preliminary investigation on how to update beliefs in epistemic graphs, please see \cite{HunterPP18}.

The epistemic approach is not the only form of probabilistic argumentation. 
Another popular method relies on constellation probabilities 
\cite{Li:2011,Hunter:2013,Fazzinga:2015}
in which we can consider a number of argument graphs, each one having a probability of being the \enquote{real graph}. 
Incorporating constellation probabilities in epistemic graphs would, for example, allow for a more refined handling of agents 
whose argument graphs are not complete but have a chance of containing certain arguments. 
Furthermore, it is also possible to allow epistemic constraints to express beliefs in arguments as well as 
in the relations between them, similarly as done in \cite{PolbergHT17}. Consequently, further developments
of the epistemic graphs are an interesting topic for future work.

Finally, we will also investigate algorithms and implementations aimed at handling epistemic graphs. 
This can be done through devising dedicated solutions as well as by introducing appropriate translations to, for example, 
propositional logic, as indicated by the results in Sections \ref{section:ReasoningConstraints}. 
Further possibilities concern employing SMT solvers or constraint logic programming. 

\section*{Acknowledgments}
 Sylwia Polberg and Anthony Hunter were supported by EPSRC Project EP/N008294/1 \enquote{Framework for Computational Persuasion}.
Matthias Thimm was partially supported by the Deutsche Forschungsgemeinschaft (grant KE 1686/3-1).

\bibliographystyle{abbrv}
\bibliography{epistemicgraph}

\begin{thebibliography}{10}

\bibitem{Amgoud2013}
L.~Amgoud and J.~{Ben-Naim}.
\newblock Ranking-based semantics for argumentation frameworks.
\newblock In W.~Liu, V.~S. Subrahmanian, and J.~Wijsen, editors, {\em
  Proceedings of SUM'13}, volume 8078 of {\em LNCS}, pages 134--147. Springer,
  2013.

\bibitem{AmgoudBenNaim16b}
L.~Amgoud and J.~{Ben-Naim}.
\newblock Axiomatic foundations of acceptability semantics.
\newblock In C.~Baral, J.~Delgrande, and F.~Wolter, editors, {\em Proceedings
  of KR'16}, pages 2--11. AAAI Press, 2016.

\bibitem{AmgoudBenNaim16a}
L.~Amgoud and J.~{Ben-Naim}.
\newblock Evaluation of arguments from support relations: Axioms and semantics.
\newblock In S.~Kambhampati, editor, {\em Proceedings of IJCAI'16}, pages
  900--906. AAAI Press, 2016.

\bibitem{AmgoudBenNaim17}
L.~Amgoud and J.~{Ben-Naim}.
\newblock Evaluation of arguments in weighted bipolar graphs.
\newblock In A.~Antonucci, L.~Cholvy, and O.~Papini, editors, {\em Proceedings
  of ECSQARU'17}, volume 10369 of {\em LNCS}, pages 25--35. Springer, 2017.

\bibitem{Amgoud16kr}
L.~Amgoud, J.~{Ben-Naim}, D.~Doder, and S.~Vesic.
\newblock Ranking arguments with compensation-based semantics.
\newblock In C.~Baral, J.~Delgrande, and F.~Wolter, editors, {\em Proceedings
  of KR'16}, pages 12--21. AAAI Press, 2016.

\bibitem{AmgoudBNDV17}
L.~Amgoud, J.~{Ben-Naim}, D.~Doder, and S.~Vesic.
\newblock Acceptability semantics for weighted argumentation frameworks.
\newblock In C.~Sierra, editor, {\em Proceedings of IJCAI'17}, pages 56--62.
  AAAI Press, 2017.

\bibitem{Barber2012}
D.~Barber.
\newblock {\em Bayesian Reasoning and Machine Learning}.
\newblock Cambridge University Press, 2012.

\bibitem{Baroni:2011}
P.~Baroni, M.~Caminada, and M.~Giacomin.
\newblock An introduction to argumentation semantics.
\newblock {\em The Knowledge Engineering Review}, 26(4):365--410, 2011.

\bibitem{BesnardHunter01}
P.~Besnard and A.~Hunter.
\newblock A logic-based theory of deductive arguments.
\newblock {\em Artificial Intelligence}, 128(1):203 -- 235, 2001.

\bibitem{BesnardHunter2014}
P.~Besnard and A.~Hunter.
\newblock Constructing argument graphs with deductive arguments: A tutorial.
\newblock {\em Argument \& Computation}, 5(1):5--30, 2014.

\bibitem{Black:2012}
E.~Black and A.~Hunter.
\newblock A relevance-theoretic framework for constructing and deconstructing
  enthymemes.
\newblock {\em Journal of Logic and Computation}, 22(1):55--78, 2012.

\bibitem{BoellaGTV10}
G.~Boella, D.~Gabbay, L.~van~der Torre, and S.~Villata.
\newblock Support in abstract argumentation.
\newblock In P.~Baroni, F.~Cerutti, M.~Giacomin, and G.~R. Simari, editors,
  {\em Proceedings of COMMA'10}, volume 216 of {\em FAIA}, pages 111--122. IOS
  Press, 2010.

\bibitem{Bonzon16}
E.~Bonzon, J.~Delobelle, S.~Konieczny, and N.~Maudet.
\newblock A comparative study of ranking-based semantics for abstract
  argumentation.
\newblock In D.~Schuurmans and M.~Wellman, editors, {\em Proceedings of
  AAAI'16}, pages 914--920. AAAI Press, 2016.

\bibitem{BrewkaESWW13}
G.~Brewka, S.~Ellmauthaler, H.~Strass, J.~P. Wallner, and S.~Woltran.
\newblock Abstract dialectical frameworks revisited.
\newblock In F.~Rossi, editor, {\em Proceedings of IJCAI'13}, pages 803--809.
  AAAI Press, 2013.

\bibitem{BrewkaPW14}
G.~Brewka, S.~Polberg, and S.~Woltran.
\newblock Generalizations of {Dung} frameworks and their role in formal
  argumentation.
\newblock {\em IEEE Intelligent Systems}, 29(1):30--38, Jan 2014.

\bibitem{Brewka18}
G.~Brewka, H.~Strass, J.~P. Wallner, and S.~Woltran.
\newblock Weighted abstract dialectical frameworks.
\newblock In S.~A. McIlraith and K.~Q. Weinberger, editors, {\em Proceedings of
  AAAI'18}, pages 1771--1778. AAAI Press, 2018.

\bibitem{BrewkaWoltran10}
G.~Brewka and S.~Woltran.
\newblock Abstract dialectical frameworks.
\newblock In F.~Lin, U.~Sattler, and M.~Truszczynski, editors, {\em Proceedings
  of KR'10}, pages 102--111. AAAI Press, 2010.

\bibitem{CabrioVillata13}
E.~Cabrio and S.~Villata.
\newblock A natural language bipolar argumentation approach to support users in
  online debate interactions.
\newblock {\em Argument \& Computation}, 4(3):209--230, 2013.

\bibitem{Caminada:2009}
M.~Caminada and D.~M. Gabbay.
\newblock A logical account of formal argumentation.
\newblock {\em Studia Logica}, 93:109--145, 2009.

\bibitem{CayrolLS05}
C.~Cayrol and M.~{Lagasquie-Schiex}.
\newblock Gradual valuation for bipolar argumentation frameworks.
\newblock In L.~Godo, editor, {\em Proceedings of ECSQARU'05}, volume 3571 of
  {\em LNCS}, pages 366--377. Springer, 2005.

\bibitem{CayrolLS05b}
C.~Cayrol and M.~{Lagasquie-Schiex}.
\newblock Graduality in argumentation.
\newblock {\em Journal of Artificial Intelligence Research}, 23:245--297, 2005.

\bibitem{CayrolLS13}
C.~Cayrol and M.-C. Lagasquie-Schiex.
\newblock Bipolarity in argumentation graphs: Towards a better understanding.
\newblock {\em International Journal of Approximate Reasoning}, 54(7):876--899,
  2013.

\bibitem{Cerutti2014}
F.~Cerutti, N.~Tintarev, and N.~Oren.
\newblock Formal arguments, preferences, and natural language interfaces to
  humans: an empirical evaluation.
\newblock In T.~Schaub, G.~Friedrich, and B.~O'Sullivan, editors, {\em
  Proceedings of ECAI'14}, volume 263 of {\em FAIA}, pages 1033--1034. {IOS}
  Press, August 2014.

\bibitem{haepilot}
F.~Cerutti, N.~Tintarev, and N.~Oren.
\newblock {Human-Argumentation Experiment Pilot 2013}.
\newblock
  \url{https://www.slideshare.net/fcerutti/humanargumentation-experiment-pilot-2013-technical-material},
  2014.

\bibitem{CosteMarquis:2007}
S.~Coste-Marquis, C.~Devred, S.~Konieczny, M.~{Lagasquie-Schiex}, and
  P.~Marquis.
\newblock On the merging of dung's argumentation systems.
\newblock {\em Artificial Intelligence}, 171(10-15):730--753, 2007.

\bibitem{inproc:careful}
S.~Coste-Marquis, C.~Devred, and P.~Marquis.
\newblock Inference from controversial arguments.
\newblock In G.~Sutcliffe and A.~Voronkov, editors, {\em Proceedings of
  {LPAR}'05}, volume 3835 of {\em LNCS}, pages 606--620. Springer, 2005.

\bibitem{inproc:prudent}
S.~Coste-Marquis, C.~Devred, and P.~Marquis.
\newblock Prudent semantics for argumentation frameworks.
\newblock In A.~Lim, editor, {\em Proceedings of {ICTAI}'05}, pages 568--572.
  IEEE Computer Society, 2005.

\bibitem{CosteMarquisDM06}
S.~{Coste-Marquis}, C.~Devred, and P.~Marquis.
\newblock Constrained argumentation frameworks.
\newblock In P.~Doherty, J.~Mylopoulos, and C.~Welty, editors, {\em Proceedings
  of KR'06}, pages 112--122. AAAI Press, 2006.

\bibitem{Costa-Pereira:2011}
C.~da~Costa~Pereira, A.~G.~B. Tettamanzi, and S.~Villata.
\newblock Changing one's mind: Erase or rewind? possibilistic belief revision
  with fuzzy argumentation based on trust.
\newblock In T.~Walsh, editor, {\em Proceedings of IJCAI'11}, pages 164--171.
  AAAI Press, 2011.

\bibitem{Dechter2003}
R.~Dechter.
\newblock {\em Constraint Processing}.
\newblock Morgan Kaufmann, 2003.

\bibitem{Dung95}
P.~M. Dung.
\newblock On the acceptability of arguments and its fundamental role in
  nonmonotonic reasoning, logic programming and n-person games.
\newblock {\em Artificial Intelligence}, 77:321--357, 1995.

\bibitem{Fazzinga:2015}
B.~Fazzinga, S.~Flesca, and F.~Parisi.
\newblock On the complexity of probabilistic abstract argumentation frameworks.
\newblock {\em ACM Transactions on Computational Logic}, 16(3):22:1--22:39,
  2015.

\bibitem{Hadoux16}
E.~Hadoux and A.~Hunter.
\newblock Computationally viable handling of beliefs in arguments for
  persuasion.
\newblock In N.~Bourbakis, A.~Esposito, A.~Mali, and M.~Alamaniotis, editors,
  {\em Proceedings of ICTAI'16}, pages 319--326. IEEE, 2016.

\bibitem{Hadoux17}
E.~Hadoux and A.~Hunter.
\newblock Strategic sequences of arguments for persuasion using decision trees.
\newblock In S.~Singh and S.~Markovitch, editors, {\em Proceedings of AAAI'17},
  pages 1128--1134. AAAI Press, 2017.

\bibitem{HadouxHunter18}
E.~Hadoux and A.~Hunter.
\newblock Learning and updating user models for subpopulations in persuasive
  argumentation using beta distributions.
\newblock In E.~Andr{\'{e}}, S.~Koenig, M.~Dastani, and G.~Sukthankar, editors,
  {\em Proceedings of {AAMAS}'18}, pages 1141--1149. IFAAMAS, 2018.

\bibitem{Hunter:2013}
A.~Hunter.
\newblock A probabilistic approach to modelling uncertain logical arguments.
\newblock {\em International Journal of Approximate Reasoning}, 54(1):47--81,
  2013.

\bibitem{Hunter15ijcai}
A.~Hunter.
\newblock Modelling the persuadee in asymmetric argumentation dialogues for
  persuasion.
\newblock In Q.~Yang and M.~Wooldridge, editors, {\em Proceedings of IJCAI'15},
  pages 3055--3061. AAAI Press, 2015.

\bibitem{Hunter16comma}
A.~Hunter.
\newblock Computational persuasion with applications in behaviour change.
\newblock In P.~Baroni, T.~F. Gordon, T.~Scheffler, and M.~Stede, editors, {\em
  Proceedings of COMMA'16}, volume 287 of {\em FAIA}, pages 5--18. IOS Press,
  2016.

\bibitem{Hunter16sum}
A.~Hunter.
\newblock Persuasion dialogues via restricted interfaces using probabilistic
  argumentation.
\newblock In S.~Schockaert and P.~Senellart, editors, {\em Proceedings of
  SUM'16}, volume 9858 of {\em LNCS}, pages 184--198. Springer, 2016.

\bibitem{Hunter16ecai}
A.~Hunter.
\newblock Two dimensional uncertainty in persuadee modelling in argumentation.
\newblock In G.~A. Kaminka, M.~Fox, P.~Bouquet, E.~H{\"u}llermeier, V.~Dignum,
  F.~Dignum, and F.~van Harmelen, editors, {\em Proceedings of ECAI'16}, volume
  285 of {\em FAIA}, pages 150--157. IOS Press, 2016.

\bibitem{HunterPP18}
A.~Hunter, S.~Polberg, and N.~Potyka.
\newblock {Updating Belief in Arguments in Epistemic Graphs}.
\newblock In M.~Thielscher, F.~Toni, and F.~Wolter, editors, {\em Proceedings
  of {KR}'18}, pages 138--147. {AAAI} Press, 2018.

\bibitem{HunterPT2018Arxiv}
A.~{Hunter}, S.~{Polberg}, and M.~{Thimm}.
\newblock Epistemic graphs for representing and reasoning with positive and
  negative influences of arguments.
\newblock {\em ArXiv CoRR}, abs/1802.07489, 2018.

\bibitem{HunterPotyka17}
A.~Hunter and N.~Potyka.
\newblock Updating probabilistic epistemic states in persuasion dialogues.
\newblock In A.~Antonucci, L.~Cholvy, and O.~Papini, editors, {\em Proceedings
  of {ECSQARU}'17}, volume 10369 of {\em LNCS}, pages 46--56. Springer, 2017.

\bibitem{Hunter:2014}
A.~Hunter and M.~Thimm.
\newblock Probabilistic argumentation with incomplete information.
\newblock In T.~Schaub, G.~Friedrich, and B.~O'Sullivan, editors, {\em
  Proceedings of ECAI'14}, volume 263 of {\em FAIA}, pages 1033--1034. IOS
  Press, 2014.

\bibitem{KontarinisToni16}
D.~Kontarinis and F.~Toni.
\newblock Identifying malicious behavior in multi-party bipolar argumentation
  debates.
\newblock In M.~Rovatsos, G.~Vouros, and V.~Julian, editors, {\em Proceedings
  of EUMAS'15}, pages 267--278. Springer, 2016.

\bibitem{LeiteMartins11}
J.~Leite and J.~Martins.
\newblock Social abstract argumentation.
\newblock In T.~Walsh, editor, {\em Proceedings of IJCAI'11}, pages 2287--2292.
  AAAI Press, 2011.

\bibitem{Li:2011}
H.~Li, N.~Oren, and T.~J. Norman.
\newblock Probabilistic argumentation frameworks.
\newblock In S.~Modgil, N.~Oren, and F.~Toni, editors, {\em Proceedings of
  TAFA'11}, volume 7132 of {\em LNCS}, pages 1--16. Springer, 2011.

\bibitem{NouiouaRisch11}
F.~Nouioua and V.~Risch.
\newblock Argumentation frameworks with necessities.
\newblock In S.~Benferhat and J.~Grant, editors, {\em Proceedings of SUM'11},
  volume 6929 of {\em LNCS}, pages 163--176. Springer, 2011.

\bibitem{ogden2012health}
J.~Ogden.
\newblock {\em Health Psychology: A Textbook}.
\newblock Open University Press, 5th edition, 2012.

\bibitem{OrenNorman08}
N.~Oren and T.~J. Norman.
\newblock Semantics for evidence-based argumentation.
\newblock In P.~Besnard, S.~Doutre, and A.~Hunter, editors, {\em Proceedings of
  COMMA'08}, volume 172 of {\em FAIA}, pages 276--284. IOS Press, 2008.

\bibitem{Parsons2004}
S.~Parsons.
\newblock On precise and correct qualitative probabilistic reasoning.
\newblock {\em International Journal of Approximate Reasoning}, 35:111--135,
  2004.

\bibitem{Pearl2000}
J.~Pearl.
\newblock {\em Causality: Models, Reasoning, and Inference}.
\newblock Cambridge University Press, 2000.

\bibitem{Polberg16}
S.~Polberg.
\newblock Understanding the {A}bstract {D}ialectical {F}ramework.
\newblock In L.~Michael and A.~Kakas, editors, {\em Proceedings of JELIA'16},
  volume 10021 of {\em LNCS}, pages 430--446. Springer, 2016.

\bibitem{thesis:polberg}
S.~Polberg.
\newblock {\em Developing the Abstract Dialectical Framework}.
\newblock Phd thesis, Technical University of Vienna, Vienna, Austria, 2017.

\bibitem{Polberg17}
S.~Polberg.
\newblock Intertranslatability of abstract argumentation frameworks.
\newblock Technical Report DBAI-TR-2017-104, Institute for Information Systems,
  Technical University of Vienna, 2017.

\bibitem{PolbergDoder14}
S.~Polberg and D.~Doder.
\newblock Probabilistic abstract dialectical frameworks.
\newblock In E.~{Ferm{\'e}} and J.~Leite, editors, {\em Proceedings of
  JELIA'14}, volume 8761 of {\em LNCS}, pages 591--599. Springer, 2014.

\bibitem{HunterPolberg17}
S.~Polberg and A.~Hunter.
\newblock Empirical methods for modelling persuadees in dialogical
  argumentation.
\newblock In J.~Guerrero, editor, {\em Proceedings of {ICTAI'17}}, pages
  382--389. IEEE, 2017.

\bibitem{PolbergHunter17}
S.~Polberg and A.~Hunter.
\newblock Empirical evaluation of abstract argumentation: Supporting the need
  for bipolar and probabilistic approaches.
\newblock {\em International Journal of Approximate Reasoning}, 93:487 -- 543,
  2018.

\bibitem{PolbergHT17}
S.~Polberg, A.~Hunter, and M.~Thimm.
\newblock Belief in attacks in epistemic probabilistic argumentation.
\newblock In S.~Moral, O.~Pivert, D.~S{\'a}nchez, and N.~Mar{\'i}n, editors,
  {\em Proceedings of SUM'17}, volume 10564 of {\em LNCS}, pages 223--236.
  Springer, 2017.

\bibitem{PolbergOren14a}
S.~Polberg and N.~Oren.
\newblock Revisiting support in abstract argumentation systems.
\newblock In S.~Parsons, N.~Oren, C.~Reed, and F.~Cerutti, editors, {\em
  Proceedings of COMMA'14}, volume 266 of {\em FAIA}, pages 369--376. {IOS}
  Press, 2014.

\bibitem{Prakken14}
H.~Prakken.
\newblock On support relations in abstract argumentation as abstractions of
  inferential relations.
\newblock In T.~Schaub, G.~Friedrich, and B.~O'Sullivan, editors, {\em
  Proceedings ECAI'14}, volume 263 of {\em Frontiers in Artificial Intelligence
  and Applications}, pages 735--740. {IOS} Press, 2014.

\bibitem{Pu2014}
F.~Pu, J.~Luo, Y.~Zhang, and G.~Luo.
\newblock Argument ranking with categoriser function.
\newblock In R.~Buchmann, C.~V. Kifor, and J.~Yu, editors, {\em Proceedings of
  {KSEM}'14}, pages 290--301. Springer International Publishing, 2014.

\bibitem{puhrer15}
J.~P{\"u}hrer.
\newblock Realizability of three-valued semantics for abstract dialectical
  frameworks.
\newblock In Q.~Yang and M.~Wooldridge, editors, {\em Proceedings of
  {IJCAI}'15}, pages 3171--3177. {AAAI} Press, 2015.

\bibitem{Rago16}
A.~Rago, F.~Toni, M.~Aurisicchio, and P.~Baroni.
\newblock Discontinuity-free decision support with quantitative argumentation
  debates.
\newblock In C.~Baral, J.~Delgrande, and F.~Wolter, editors, {\em Proceedings
  of KR'16}, pages 63--73. AAAI Press, 2016.

\bibitem{Rahwan2011}
I.~Rahwan, M.~Madakkatel, J.~Bonnefon, R.~Awan, and S.~Abdallah.
\newblock Behavioural experiments for assessing the abstract argumentation
  semantics of reinstatement.
\newblock {\em Cognitive Science}, 34(8):1483--1502, 2010.

\bibitem{Reiter1978}
R.~Reiter.
\newblock On closed world data bases.
\newblock In H.~Gallaire and J.~Minker, editors, {\em Logic and Data Bases},
  pages 55--76. Springer, 1978.

\bibitem{RosenfeldKraus16}
A.~Rosenfeld and S.~Kraus.
\newblock Strategical argumentative agent for human persuasion.
\newblock In G.~A. Kaminka, M.~Fox, P.~Bouquet, E.~H{\"u}llermeier, V.~Dignum,
  F.~Dignum, and F.~van Harmelen, editors, {\em Proceedings of ECAI'16}, volume
  285 of {\em FAIA}, pages 320--328. IOS Press, 2016.

\bibitem{Rossi:2006}
F.~Rossi, P.~van Beek, and T.~Walsh, editors.
\newblock {\em Handbook of Constraint Programming}, volume~2 of {\em
  Foundations of Artificial Intelligence}.
\newblock Elsevier, 2006.

\bibitem{Strass13a}
H.~Strass.
\newblock Approximating operators and semantics for abstract dialectical
  frameworks.
\newblock {\em Artificial Intelligence}, 205:39 -- 70, 2013.

\bibitem{Strass13}
H.~Strass.
\newblock Instantiating knowledge bases in abstract dialectical frameworks.
\newblock In J.~Leite, T.~C. Son, P.~Torroni, L.~van~der Torre, and S.~Woltran,
  editors, {\em Proceedings of CLIMA'13}, volume 8143 of {\em LNCS}, pages
  86--101. Springer, 2013.

\bibitem{StrassWallner15}
H.~Strass and J.~P. Wallner.
\newblock Analyzing the {Computational Complexity} of {Abstract Dialectical
  Frameworks} via {Approximation Fixpoint Theory}.
\newblock {\em Artificial Intelligence}, 226:34--74, 2015.

\bibitem{TanNDNML16}
C.~Tan, V.~Niculae, C.~Danescu-Niculescu-Mizil, and L.~Lee.
\newblock Winning arguments: Interaction dynamics and persuasion strategies in
  good-faith online discussions.
\newblock In J.~Bourdeau, J.~Hendler, R.~Nkambou, I.~Horrocks, and B.~Y. Zhao,
  editors, {\em Proceedings of {WWW}'16}, pages 613--624. {ACM}, 2016.

\bibitem{Thimm:2012}
M.~Thimm.
\newblock A probabilistic semantics for abstract argumentation.
\newblock In L.~D. Raedt, C.~Bessiere, D.~Dubois, P.~Doherty, P.~Frasconi,
  F.~Heintz, and P.~Lucas, editors, {\em Proceedings of ECAI'12}, volume 242 of
  {\em FAIA}, pages 750--755. IOS Press, 2012.

\bibitem{Timmer2017}
S.~Timmer, J.~Meyer, H.~Prakken, S.~Renooij, and B.~Verheij.
\newblock A two-phase method for extracting explanatory arguments from bayesian
  networks.
\newblock {\em International Journal of Approximate Reasoning}, 80:475--494,
  2017.

\bibitem{Tsang1993}
E.~Tsang.
\newblock {\em Foundations of Constraint Satisfaction}.
\newblock Academic Press, 1993.

\bibitem{VanBeek:2006}
P.~{van Beek}.
\newblock {\em Backtracking Search Algorithms}, chapter~4, pages 85 -- 134.
\newblock Volume~2 of Rossi et~al. \cite{Rossi:2006}, 2006.

\bibitem{Zeng2008}
Z.~Zeng.
\newblock Context-based and explainable decision making with argumentation.
\newblock In E.~Andr{\'{e}}, S.~Koenig, M.~Dastani, and G.~Sukthankar, editors,
  {\em Proceedings of AAMAS'18}, pages 1114--1122. IFAAMAS, 2018.

\end{thebibliography}

\newpage
\section{Proof Appendix}

\reasonablerestricted*

\begin{proof}
\begin{itemize}
\item If $\Pi$ is nonempty, then there exists $x \in [0,1]$ s.t. $x \in \Pi$. By the definition of the restricted value set, 
$x-x \in \Pi$. Hence, $0 \in \Pi$. 

\item Let $\nodes(\graph) \neq \emptyset$. Since $\Pi$ is reasonable, then $\dist(\graph, \Pi) \neq \emptyset$. 
Hence, there exists $\prob \in \dist(\graph, \Pi)$ and 
for every $X \subseteq \nodes(\graph)$, there exist a value $y \in [0,1]$ s.t. $\prob(X) = y$. Hence, 
$y \in \Pi$ and $\Pi$ is nonempty. Thus, based on the previous part of this proof, $0 \in \Pi$.  

\item If $\Pi$ is a reasonable restricted value set, then for any nonempty graph, $\dist(\graph, \Pi) \neq \emptyset$. 
Hence, we can find $x_1, \ldots, x_n \in \Pi$ s.t. $\sum_{i=1}^n x_i = 1$. Since $\Pi$ is a restricted value set, 
then $x_1 + x_2 = y_1 \in \Pi$, $y_1 + x_3 = y_2 \in \Pi$, \ldots, $y_{n-2} + x_n = 1 \in \Pi$. 

Let $\Pi$ be a nonempty restricted value set s.t. $1 \in \Pi$. By the previous parts of this proof, $0 \in \Pi$. 
Thus, for any graph $\graph$ s.t. $\nodes(\graph) \neq \emptyset$, we can create a trivial distribution $\prob$ 
s.t. $\prob(\emptyset) = 1$ and $\forall X \subseteq \nodes(\graph)$ s.t $X \neq \emptyset$, $\prob(X) = 0$. 
Consequently, $\dist(\graph, \Pi) \neq \emptyset$ and $\Pi$ is reasonable. 
\end{itemize}
\end{proof}

\restrnonempty*

\begin{proof} 

Let us focus on the first case. It is easy to see that if any of the conditions 
are met, then  $\Pi^x_\ineq = \emptyset$. What remains to be shown is that if 
$\Pi^x_\ineq = \emptyset$, then one of these conditions has to be satisfied. 
Assume it is not the case, i.e. $\Pi^x_\ineq = \emptyset$, but none of the conditions is satisfied. 

Let $\ineq$ be $<$. Then, $\Pi^x_< = \emptyset$ iff $x$ is equal to the minimal value in the set, which given the nature of 
$\Pi$ is $0$. We reach a contradiction (option 3). If $\ineq$ is $>$, we can repeat similar reasoning and 
reach a contradiction with option 2. 
Let $\ineq$ be $\neq$. $\Pi^x_\neq = \emptyset$ iff $|\Pi| = 1$. 
Given the nature of $\Pi$, this is only possible when $\Pi = \{0\}$. We reach a contradiction with option 1. 
It is easy to see that for $\ineq \in \{=,\geq,\leq\}$, $\Pi^x_\ineq \neq \emptyset$. 

This proves that $\Pi^x_\ineq = \emptyset$ if and only if one of the listed conditions is met. 
 
Let us now analyze the combination sets. It is easy to verify that if any of the conditions is met, then the resulting 
combination set is indeed empty. Let us therefore show that if the combination set is empty, then one of the conditions is met. 
Let us assume that it is not the case, i.e. $\Pi^{x,(*_1, \ldots, *_k)}_{\ineq} = \emptyset$ but no condition is satisfied. 

Let $\Pi^{x,(*_1, \ldots, *_k)}_{\ineq} = \emptyset$ and assume that $k = 0$. This means that 
$\Pi^{x,(*_1, \ldots, *_k)}_{\ineq} = \{(v) \mid v \in \Pi^{x}_{\ineq} \}$. Thus, $\Pi^{x,(*_1, \ldots, *_k)}_{\ineq}$
is empty iff $\Pi^{x}_{\ineq}$ is empty. However, this means that we satisfy the first condition and thus reach a contradiction 
with our assumptions. 

Let $\Pi^{x,(*_1, \ldots, *_k)}_{\ineq} = \emptyset$ and assume that $k > 0$. Since $\Pi$ is nonempty, 
then by Lemma \ref{lemma:reasonablerestricted}, $0 \in \Pi$. 
Thus, for any $x$, any $k$ and any sequence $(*_1, \ldots, *_k)$, we can create a trivial tuple $(v_1, \ldots, v_{k+1})$ 
s.t. $v_1 *_1 v_2 \ldots v_k *_k v_{k+1} = x$. This is simply achieved by setting $v_1 = x$ and $v_i = 0$, where $i>1$. 
Hence, for $\ineq \in \{ \geq, \leq, =\}$, clearly $\Pi^{x,(*_1, \ldots, *_k)}_{\ineq} \neq \emptyset$. 
Let us therefore focus on $\ineq \in \{>, <, \neq\}$ and start with $>$. 
Let $y = max(\Pi)$. 
If $x \neq y$, then $y > x$, and we can create a sequence $(v_1, \ldots, v_{k+1})$ s.t. 
$v_1 *_1 v_2 \ldots v_k *_k v_{k+1} > x$ by setting $v_1 = y$ and $v_i = 0$, where $i>1$. 
Hence, we reach a contradiction 
for this case. 
If $x = y$, then if $y \neq 0$ (recall that values from $\Pi$ belong in the unit interval), 
then $y + y > x$.  Consequently, if there is at least one $j$ s.t. $*_j = +$, we can create a sequence 
$(v_1, \ldots, v_{k+1})$ s.t. 
$v_1 *_1 v_2 \ldots v_k *_k v_{k+1} > x$ by setting $v_1 = v_j = y$ and $v_i = 0$, where $i>1$ and $i \neq j$. 
If there is no addition present (i.e. we only have subtractions), 
then it is easy to see that the maximal value we can obtain from our formula is when 
$v_1 = y$ and the remaining values are set to $0$. Thus, in this case, 
$v_1 *_1 v_2 \ldots v_k *_k v_{k+1} = y$ and since $x=y$, then our combination set is empty. However, this scenario 
coincides with one of our conditions that was not supposed to be satisfied, and we reach a contradiction. 
We are therefore left with the case where $x = y = 0$. Since $y = max(\Pi)$, then clearly $\Pi = \{0\}$
and independently of the used arithmetic operators and values, every formula will always amount to $0$. 
Since $0 \not > 0$, our combination set is empty. However, this scenario is again covered by one of our conditions and we 
reach a contradiction. 

Let us now focus on $<$. It is easy to see that since $0 \in \Pi$, then as long as $x \neq 0$, we can observe that  
$v_1 *_1 v_2 \ldots v_k *_k v_{k+1} < x $ for $v_i = 0$. Hence, in such a case, the combination set would never be empty. 
Thus, consider the case where $x = 0$. 
If $\Pi = \{0\}$, then the smallest value obtainable by $v_1 *_1 v_2 \ldots v_k *_k v_{k+1}$ is $0$, and the combination 
set is therefore empty. However, this is already covered by one of our conditions, and we reach a contradiction. 
If $\Pi \neq \{0\}$, then as long as there is at least one $*_j$ s.t. $*_j = -$, we can obtain a formula producing 
a value smaller  than $0$ and the combination set is nonempty. If 
the sequence of operations does not contain any subtractions, then the smallest value obtainable by 
$v_1 *_1 v_2 \ldots v_k *_k v_{k+1}$ is again $0$, and the combination set is empty. However, this again is covered 
by one of our conditions and we reach a contradiction. 

Finally, we can consider $\neq$. Let $y = max(\Pi)$. If $y \neq 0$ and $x \neq y$, then a tuple s.t. the first 
position is $y$ and every other value is $0$ will be in the combination set irrespective of the sequence of operators. 
If $y \neq 0$ and $x = y$, then a tuple of $0$'s 
will be in the combination set irrespective of the sequence of operators.
If $y = 0$ then $\Pi = \{0\}$ and therefore $x = 0$ as well. In this case, independently of the sequence of operators, 
every possible formula will evaluate to $0$ and the combination set will be empty. However, this case is covered by one 
of our conditions, and we reach a contradiction. 

Therefore, we have shown that the combination set is empty iff one of our conditions is met. 
\end{proof}

\satdistrdnf*

\begin{proof}
Let us assume that the arguments in $\graph$ are ordered according to some ordering. 
Let $\argcomplete(\graph) = \{c_1, \ldots, c_j\}$ be the collection of all argument complete propositional 
formulae for $\graph$
and $\varphi^\prob = \pprob(c_1) = x_1 \land \pprob(c_2) = x_2 \land \ldots \land \pprob(c_j) = x_j$, 
where $x_i = \prob(c_i)$, the epistemic formula associated with $\prob$. 

By definition, $\prob$ is an assignment, where the elements of the powerset of arguments are mapped to 
numerical values s.t. these 
values add up to 1. 
Every set of arguments in the powerset can be described with a binary number, where if an i-th digit is 1, then 
the i-th argument is in the set, and it if its 0, then it is not in the set. 
We can observe that every argument complete formula has precisely one model which is trivially constructed -- if 
the i-th argument appears as a positive literal in the formula, then it is in the model, if it appears as a negative literal, 
then it is not in the model. 

We can therefore observe that a given complete formula encodes precisely one set of arguments from the powerset 
and the epistemic atom involving it demands that the value assigned to the formula is the same as the value 
assigned to the corresponding set by the probability distribution. It is therefore easy to see that 
$\varphi^\prob$ is satisfied only by $\prob$. Hence, $\{\prob\} = \sat(\varphi^\prob)$. 
\end{proof}

\satdistrdnfform*

\begin{proof}
Assume $\sat(\psi, \Pi) = \emptyset$. Then, $\varphi \formis \bot$ and $\sat(\bot, \Pi) = \emptyset$. 
Hence, $\sat(\psi, \Pi) = \sat(\varphi, \Pi)$. 

Assume $\sat(\psi, \Pi) \neq \emptyset$. Let $\varphi \formis  \varphi^{\prob_1} \lor \varphi^{\prob_2} \ldots \lor \varphi^{\prob_n}$. 
Then, 
$\sat(\varphi, \Pi) = \sat(\varphi^{\prob_1}, \Pi) \cup \sat(\varphi^{\prob_2}, \Pi) \cup \ldots \cup \sat(\varphi^{\prob_n}, \Pi)$. 
Since $\prob_i \in \sat(\psi, \Pi)$, then $\prob_i \in \dist(\graph, \Pi)$ as well. Thus, based on Proposition \ref{satdistrdnf}, 
$\sat(\varphi^{\prob_i}, \Pi) = \{\prob_i\}$. We can therefore show that 
$\sat(\varphi, \Pi) =\{\prob_1, \ldots, \prob_n\} = \sat(\psi, \Pi)$.  
\end{proof}

\temp*

\begin{proof} 
Assume $\Phi \VDash \psi$.
Therefore, $\sat(\Phi) \subseteq \sat(\psi)$.
Therefore, $\sat(\Phi)\cap \dist(\graph,\Pi) \subseteq \sat(\psi) \cap \dist(\graph,\Pi)$.
Therefore, $\sat(\Phi,\Pi) \subseteq \sat(\psi,\Pi)$.
Therefore, $\Phi \VDash_{\Pi} \psi$. 
\end{proof}

\tempmore*

\begin{proof}
Let $X, Y, W, Z$ be sets of elements s.t. $W \subseteq Z$. 
It is easy to show that if $X \cap Z \subseteq Y \cap Z$, then $X \cap Z \cap W \subseteq Y \cap Z \cap W$. 
Since, $W \subseteq Z$, then $X \cap Z \cap W = X \cap W$ and $Y \cap Z \cap W = Y \cap W$. 
Thus, $X \cap Z \subseteq Y \cap Z$ implies $X \cap W \subseteq Y \cap W$ when $W \subseteq Z$. 

We can show that if $\Pi_1 \subseteq \Pi_2$, then $\dist(\Pi_1) \subseteq \dist(\Pi_2)$. Hence, using the above 
analysis, it is easy to prove that if 
$\sat(\Phi) \cap \dist(\Pi_2) \subseteq \sat(\psi)\cap \dist(\Pi_2)$
then 
$\sat(\Phi) \cap \dist(\Pi_1) \subseteq \sat(\psi)\cap \dist(\Pi_1)$. Thus, 
$\Phi \VDash_{\Pi_2} \psi$ then $\Phi \VDash_{\Pi_1} \psi$. 
\end{proof}

\propformsat*

\begin{proof}
Let $f_1 \formis  \pprob(\alpha_1) *_1 \pprob(\alpha_2) *_2 \ldots  *_{m-1} \pprob(\alpha_m)$, 
and $f_2 \formis \pprob(\beta_1) \star_1 \pprob(\beta_2) \star_2 \ldots  \star_{l-1} \pprob(\beta_l)$, 
where $\alpha_i, \beta_i \in \terms(\graph)$ 
and $*_i, \star_i \in \{+, -\}$, be operational formulae. 
Let $\varphi_1 = f_1 \ineq x$ and $\varphi_2 = f_2 \ineq x$, where $\ineq \in \operators$ and 
$x \in [0,1]$.
\begin{itemize} 

\item Let $\pprob(\alpha_i)$ be the element that became weakened to $\pprob(\alpha'_i)$ (i.e. $\beta_i = \alpha'_i$). Since 
$\{\alpha_i\}\vdash\alpha'_i$, then for any probability distribution $\prob$ it holds that $\prob(\alpha_i) \leq \prob(\alpha'_i)$. 
Therefore, for any probability distribution $\prob$, 
$\prob(\alpha_1) *_1 \prob(\alpha_2) *_2 \ldots + \prob(\alpha_i) *_i \ldots *_{m-1} \prob(\alpha_m)\leq \prob(\alpha_1) *_1 \prob(\alpha_2) \ldots + \prob(\alpha'_i) *_i \ldots *_{m-1} \prob(\alpha_m)$. 
Consequently, 
if
$\prob(\alpha_1) *_1 \prob(\alpha_2) *_2 \ldots + \prob(\alpha_i) *_i\ldots *_{m-1} \prob(\alpha_m) 
\ineq x$, where $\ineq \in \{>,\geq\}$, 
then $\prob(\alpha_1) *_1 \prob(\alpha_2) \ldots + \prob(\alpha'_i) *_i \ldots *_{m-1} \prob(\alpha_m) \ineq x$ as well. 
Hence, for every $\prob' \in \sat(\varphi_1)$,  $\prob' \in \sat(\varphi_2)$ as well, and it holds that 
$\sat(\varphi_1) \subseteq \sat(\varphi_2)$. 
Furthermore, if  
$\prob(\alpha_1) *_1 \prob(\alpha_2) *_2 \ldots + \prob(\alpha'_i) *_i \ldots *_{m-1} \prob(\alpha_m)\ineq x$, where $\ineq \in \{<,\leq\}$, 
then $\prob(\alpha_1) *_1 \prob(\alpha_2) \ldots + \prob(\alpha_i) *_i \ldots *_{m-1} \prob(\alpha_m)\ineq x$ as well. 
Hence, for every $\prob' \in \sat(\varphi_2)$,  $\prob' \in \sat(\varphi_1)$ as well, and it holds that 
$\sat(\varphi_2) \subseteq \sat(\varphi_1)$.  

\item Let $\pprob(\alpha_i)$ be the element that became weakened to $\pprob(\alpha'_i)$ (i.e. $\beta_i = \alpha'_i$). Since 
$\{\alpha_i\}\vdash\alpha'_i$, then for any probability distribution $\prob$ it holds that $\prob(\alpha_i) \leq \prob(\alpha'_i)$. 
Therefore, for any probability distribution $\prob$, 
$\prob(\alpha_1) *_1 \prob(\alpha_2) *_2 \ldots - \prob(\alpha_i) *_i\ldots *_{m-1} \prob(\alpha_m) \geq \prob(\alpha_1) *_1 \prob(\alpha_2) \ldots - \prob(\alpha'_i) *_i \ldots *_{m-1} \prob(\alpha_m)$. 
Therefore, if
$\prob(\alpha_1) *_1 \prob(\alpha_2) *_2 \ldots - \prob(\alpha_i) *_i\ldots *_{m-1} \prob(\alpha_m)\ineq x$, where $\ineq \in \{<,\leq\}$, 
then $\prob(\alpha_1) *_1 \prob(\alpha_2) \ldots - \prob(\alpha'_i) *_i \ldots *_{m-1} \prob(\alpha_m)\ineq x$ as well. 
Hence, for every $\prob' \in \sat(\varphi_1)$,  $\prob' \in \sat(\varphi_2)$ as well, and it holds that 
$\sat(\varphi_1) \subseteq \sat(\varphi_2)$. 
Furthermore, if $\prob(\alpha_1) *_1 \prob(\alpha_2) \ldots - \prob(\alpha'_i) *_i \ldots *_{m-1} \prob(\alpha_m)\ineq x$, where $\ineq \in \{>, \geq\}$
then 
$\prob(\alpha_1) *_1 \prob(\alpha_2) *_2 \ldots - \prob(\alpha_i) *_i\ldots *_{m-1} \prob(\alpha_m)\ineq x$ as well. 
Hence, for every $\prob' \in \sat(\varphi_2)$,  $\prob' \in \sat(\varphi_1)$ as well, and it holds that 
$\sat(\varphi_2) \subseteq \sat(\varphi_1)$. 
\end{itemize} 
\end{proof}


To make certain proofs more readable, we also introduce the following derivable rules (a more extensive set can be found in \cite{HunterPT2018Arxiv}).

\begin{restatable}{proposition}{valuedderivables}  
\label{valuedderivables}
Let $\ineq \in \operators$ be the set of inequality relationships, 
let $\Pi$ be a reasonable restricted value set, 
and let $\Pi^x_{\ineq} = \{ y \in \Pi \mid y \ineq x \}$ be the subset obtained according to the value $x$ and relationship $\ineq$. 
Let $f_1 \formis  \pprob(\alpha_1) *_1 \pprob(\alpha_2) *_2 \ldots *_{k-1} \pprob(\alpha_k)$ 
and $f_2 \formis  \pprob(\beta_1) \star_1 \pprob(\beta_2) \star_2 \ldots \star_{l-1} \pprob(\beta_l)$, 
where $k,l \geq 1$, $\alpha_i,\beta_i \in \terms(\graph)$ and $\star_j,*_i \in \{+ ,-\}$ be  operational formulae.  
The following hold, where $\vdash$ is propositional consequence relation, 
$\Phi \subseteq \eformulae(\graph,\Pi)$, $\phi,\psi \in \eformulae(\graph,\Pi)$ and $x \in \Pi$.  
\begin{enumerate}

\item \label{derivable:1} $\Phi\Vdash_\Pi \pprob(\alpha) \geq 0 $

\item  \label{derivable:2} $\Phi\Vdash_\Pi \pprob(\alpha) \leq 1$

\item  \label{derivable:3} $\Phi\Vdash_\Pi \pprob(\top) = 1 $

\item  \label{derivable:4} $\Phi\Vdash_\Pi \pprob(\bot) = 0 $  
%
%
%
%
%
%
%
%
%
%

\item \label{derivable:15} $\Phi\Vdash_\Pi f_1 > x   \mbox{ iff } \Phi\Vdash_\Pi \neg (f_1 \leq x)$

\item \label{derivable:16} $ \Phi\Vdash_\Pi f_1 < x   \mbox{ iff } \Phi\Vdash_\Pi \neg (f_1 \geq x) $

\item \label{derivable:17} $ \Phi\Vdash_\Pi f_1 \leq x   \mbox{ iff } \Phi\Vdash_\Pi \neg (f_1 > x) $

\item \label{derivable:18} $ \Phi\Vdash_\Pi f_1 \geq x  \mbox{ iff } \Phi\Vdash_\Pi \neg (f_1 < x) $ 

\item \label{derivable:19} $ \Phi\Vdash_\Pi f_1 = x   \mbox{ iff } \Phi\Vdash_\Pi \neg (f_1 \neq x)  $

\item \label{derivable:20} $ \Phi\Vdash_\Pi f_1 \neq x  \mbox{ iff } \Phi\Vdash_\Pi \neg (f_1 = x) $ 

\item \label{derivable:21} $\Phi\Vdash_\Pi f_1 = x \mbox{ and } f_1 \succeq_{su}^+ f_2 \mbox{ implies } \Phi\Vdash_\Pi f_2 \geq x$ 

\item  \label{derivable:22} $\Phi\Vdash_\Pi f_1 = x \mbox{ and } f_1 \succeq_{su}^- f_2 \mbox{ implies } \Phi\Vdash_\Pi f_2 \leq x$
  
\item \label{derivable:23} $\Phi\Vdash_\Pi f_2 = x \mbox{ and } f_1 \succeq_{su}^+ f_2 \mbox{ implies } \Phi\Vdash_\Pi f_1 \leq x$ 

\item \label{derivable:24} $\Phi\Vdash_\Pi f_2 = x \mbox{ and } f_1 \succeq_{su}^- f_2 \mbox{ implies } \Phi\Vdash_\Pi f_1 \geq x$

\item \label{derivable:25} if $f_1 \succeq_{su}^+ f_2$ and $f_2 \succeq_{su}^+ f_1$, then $\Phi\Vdash_\Pi f_1 = x$ iff $\Phi\Vdash_\Pi f_2 = x$ 

\item \label{derivable:26} if $f_1 \succeq_{su}^- f_2$ and $f_2 \succeq_{su}^- f_1$, then $\Phi\Vdash_\Pi f_1 = x$ iff $\Phi\Vdash_\Pi f_2 = x$ 
 
%
%
%
%
%
%
%
  
\item  \label{derivable:35} $\Phi\Vdash_\Pi f_1 = x \land f_1 = y \mbox{ where } x \neq y \mbox { iff } \Phi\Vdash_\Pi \bot$
%
%
%
%
%
%
%
%
%

\item \label{derivable:45} $\Phi \Vdash_\Pi \pprob(\alpha \lor \beta) = x$ iff $\Phi \Vdash_\Pi \pprob(\alpha) + \pprob(\beta) - \pprob(\alpha \land \beta) = x$

\item \label{derivable:46} $\Phi \Vdash_\Pi \pprob(\alpha \land \beta) = x$ iff $\Phi \Vdash_\Pi \pprob(\alpha) + \pprob(\beta) - \pprob(\alpha \lor \beta) = x$

%
\end{enumerate}  
 
\end{restatable}

\begin{proof} 
\begin{enumerate}
%
%
%
\item We prove that  $\Phi\Vdash_\Pi \pprob(\alpha) \geq 0$. 
We can observe that  
$\Phi \Vdash_\Pi \top$ by propositional rule P2. Thus, by combining the propositional rules and the basic rule B1, 
$\Phi \Vdash_\Pi \pprob(\alpha) \geq 0$. 

\item We can prove that $\Phi\Vdash_\Pi \pprob(\alpha) \leq 1$ similarly to previous point. 

\item We can prove that $\Phi\Vdash_\Pi \pprob(\top) = 1$ similarly to previous points. 

\item We can prove that  $\Phi\Vdash_\Pi \pprob(\bot) = 0$ similarly to previous points. 

\item We can prove that $\Phi\Vdash_\Pi  f_1  > x   $ iff $ \Phi\Vdash_\Pi \neg ( f_1  \leq x)$ 
using enumeration rules E1 and E2.  

\item We prove that $ \Phi\Vdash_\Pi  f_1  < x   $ iff $ \Phi\Vdash_\Pi \neg ( f_1  \geq x) $ 
using enumeration rules E1 and E4.

\item We prove that  $ \Phi\Vdash_\Pi  f_1  \leq x   $ iff $ \Phi\Vdash_\Pi \neg ( f_1  > x) $ 
 using enumeration rules E1 and E3. 

\item We prove that $ \Phi\Vdash_\Pi  f_1  \geq x  $ iff $ \Phi\Vdash_\Pi \neg ( f_1  < x) $
 using enumeration rules E1 and E5. 
 
\item We prove that $\Phi\Vdash_\Pi f_1 = x   $ iff $ \Phi\Vdash_\Pi \neg (f_1  \neq x)$. 
Using the propositional rules we can show that $\Phi\Vdash_\Pi f_1 = x$ iff 
$\Phi\Vdash_\Pi  \neg (\neg (f_1 = x))$. 
We can use enumeration and propositional rules to show that $\Phi\Vdash_\Pi  \neg (\neg (f_1 = x))$ iff 
$\Phi\Vdash_\Pi  \neg (\neg (f_1 \geq x \land f_1 \leq x))$. 
This, by the propositional rules and previous parts
of this proof, is equivalent to $\Phi\Vdash_\Pi  \neg (f_1 < x \lor f_1 > x))$, which through the enumeration rule 
is the same as $\Phi\Vdash_\Pi  \neg (f_1 \neq x)$.
 
\item We can prove that  $ \Phi\Vdash_\Pi f_1 \neq x  $ iff $ \Phi\Vdash_\Pi \neg (f_1 = x) $
using the previous point and the propositional rules. 

\item We can prove that 
$\Phi\Vdash_\Pi f_1 = x \mbox{ and } f_1 \succeq_{su}^+ f_2 \mbox{ implies } \Phi\Vdash_\Pi f_2 \geq x$ 
using the previously proved rule $\Phi\Vdash_\Pi f_1 \geq x   $ iff $ \Phi\Vdash_\Pi f_1 > x \lor f_1  = x $, subject rule S2 and the propositional rules.  

\item We can prove that $\Phi\Vdash_\Pi f_1 = x \mbox{ and } f_1 \succeq_{su}^- f_2 \mbox{ implies } \Phi\Vdash_\Pi f_2 \leq x$
using the previously proved rule $\Phi\Vdash_\Pi f_1 \leq x   $ iff $ \Phi\Vdash_\Pi f_1 < x \lor f_1  = x $, subject rule S4 and the propositional rules.   

\item We can prove that 
$\Phi\Vdash_\Pi f_2 = x \mbox{ and } f_1 \succeq_{su}^+ f_2 \mbox{ implies } \Phi\Vdash_\Pi f_1 \leq x$ 
using the previously proved rule $\Phi\Vdash_\Pi f_2 \leq x   $ iff $ \Phi\Vdash_\Pi f_2 < x \lor f_2  = x $, subject rule S6 and the propositional rules.  

\item We can prove that 
$\Phi\Vdash_\Pi f_2 = x \mbox{ and } f_1 \succeq_{su}^- f_2 \mbox{ implies } \Phi\Vdash_\Pi f_1 \geq x$ 
using the previously proved rule $\Phi\Vdash_\Pi f_2 \geq x   $ iff $ \Phi\Vdash_\Pi f_2 > x \lor f_2  = x $, subject rule S8 and the propositional rules.

\item We now prove that if $f_1 \succeq_{su}^+ f_2$ and $f_2 \succeq_{su}^+ f_1$, then $\Phi\Vdash_\Pi f_1 = x$ iff $\Phi\Vdash_\Pi f_2 = x$. Using the previous points of this proof and the propositional rule P1, we can show that
if $f_1 \succeq_{su}^+ f_2$ and $f_2 \succeq_{su}^+ f_1$, then
$\Phi\Vdash_\Pi f_2 \geq x \land f_2 \leq x$. We can use propositional and enumeration rules to show that this is equivalent 
$\Phi\Vdash_\Pi f_2 =x$. The right to left direction can be proved in a similar fashion. 

\item We can show that if $f_1 \succeq_{su}^- f_2$ and $f_2 \succeq_{su}^- f_1$, then $\Phi\Vdash_\Pi f_1 = x$ iff $\Phi\Vdash_\Pi f_2 = x$, similarly as the previous point. 

\item  
We now prove that 
$\Phi\VDash_\Pi f_1= x \land f_1 = y$ where $ x \neq y$ iff $\Phi\VDash_\Pi \bot$. 
Assume that $\Pi = \{z_1, \ldots, z_m\}$ and that $x = z_i$ and $y = z_j$ where $i \neq j$. 
Through repeated use of enumeration rule E1 and propositional rules, 
we can show that if $\Phi\VDash_\Pi f_1 = x \land f_1 = y$ then 
$\Phi\VDash_\Pi f_1 = z_i \land (f_1 = z_1 \lor \ldots \lor f_1 = z_{i-1} \lor f_1 = z_{i+1} \lor \ldots \lor f_1 = z_m)$, 
which is equivalent to $\Phi\VDash_\Pi f_1 = z_i \land f_1 \neq z_i$. This, through derivable rule \ref{derivable:20}, 
is equivalent to $\Phi\VDash_\Pi f_1 = z_i \land \neg f_1 = z_i$ and, through the propositional rule, to $\Phi\VDash_\Pi \bot$. 
The right to left direction can be easily proved from the propositional rules. Therefore, 
$\Phi\VDash_\Pi f_1= x \land f_1 = y$ where $ x \neq y$ iff $\Phi\VDash_\Pi \bot$.

\item We show that $\Phi \Vdash_\Pi \pprob(\alpha \lor \beta) = x$ iff $\Phi \Vdash_\Pi \pprob(\alpha) + \pprob(\beta) - \pprob(\alpha \land \beta) = x$. Let us consider the left to right direction and assume that 
 $\Phi \Vdash_\Pi \pprob(\alpha \lor \beta) = x$. Based on the probabilistic rule PR1, 
$\Phi\Vdash_\Pi \pprob(\alpha \lor \beta)  - \pprob(\alpha) -  \pprob(\beta) +  \pprob(\alpha \land \beta) = 0$, 
which using enumeration rule E1 can be written as 
$\Phi\Vdash_\Pi \bigvee_{(v_1,v_2,v_3,v_4) \in \Pi^{0, (-,-,+)}_{=}} (\pprob(\alpha \lor \beta) = v_1 \land \pprob(\alpha) = v_2 \land  \pprob(\beta)= v_3 \land  \pprob(\alpha \land \beta) =v_4)$
(we observe that based on Proposition \ref{restrnonempty}, $\Pi^{0, (-,-,+)}_{=} \neq \emptyset$). 
Through the use of propositional rule P1, 
 $\Phi \Vdash_\Pi \pprob(\alpha \lor \beta) = x \land \bigvee_{(v_1,v_2,v_3,v_4) \in \Pi^{0, (-,-,+)}_{=}} (\pprob(\alpha \lor \beta) = v_1 \land \pprob(\alpha) = v_2 \land  \pprob(\beta)= v_3 \land  \pprob(\alpha \land \beta) =v_4)$. 
Through the use of propositional rule (in particular, distributive and identity laws) and derivable rule \ref{derivable:35}, 
we can observe that the above formula is equivalent to  
 $\Phi \Vdash_\Pi \bigvee_{(x,v_2,v_3,v_4) \in \Pi^{0, (-,-,+)}_{=}} (\pprob(\alpha \lor \beta) =x \land \pprob(\alpha) = v_2 \land  \pprob(\beta)= v_3 \land  \pprob(\alpha \land \beta) =v_4)$. This can be further shown to be equivalent to 
$\Phi \Vdash_\Pi \pprob(\alpha \lor \beta) =x \land \bigvee_{(v_1,v_2,v_3) \in \Pi^{x, (+,-)}_{=}} (\pprob(\alpha) = v_1 \land  \pprob(\beta)= v_2 \land  \pprob(\alpha \land \beta) =v_3)$ 
and therefore to $\Phi \Vdash_\Pi \pprob(\alpha \lor \beta) =x \land 
(\pprob(\alpha) +  \pprob(\beta) -  \pprob(\alpha \land \beta) =x)$ through the use of enumeration rule E1. 
Hence, by P1, $\Phi \Vdash_\Pi  \pprob(\alpha) +  \pprob(\beta) -  \pprob(\alpha \land \beta) =x$. 
The right to left direction can be proved in a similar fashion. 

\item We can show that $\Phi \Vdash_\Pi \pprob(\alpha \land \beta) = x$ iff $\Phi \Vdash_\Pi \pprob(\alpha) + \pprob(\beta) - \pprob(\alpha \lor \beta) = x$ in the same way as the previous point of this proof.  

\end{enumerate}
\end{proof}

\restrictedvalsound*

\begin{proof}
We can show that each proof rule is sound. We first consider the \textbf{basic rules}:
\begin{itemize}
\item Consider proof rule 1. We need to show that 
$\Phi\VDash_\Pi \pprob(\alpha) \geq 0$ iff $\Phi\VDash_\Pi \top$. 
We can observe that $\sat(\top, \Pi) = \dist(\Pi)$. 
Furthermore, by definition,  $\sat(\pprob(\alpha) \geq 0, \Pi) = 
\{\prob' \in \dist(\graph) \mid \prob'(\alpha) \geq 0\} \cap \dist(\Pi)$. 
It is easy to see that $\{\prob' \in \dist(\graph) \mid \prob'(\alpha) \geq 0\} = \dist(\graph)$ for any $\alpha$. 
Since $\dist(\Pi) \subseteq \dist(\graph)$, $\sat(\pprob(\alpha) \geq 0, \Pi) =\dist(\Pi) = \sat(\top, \Pi)$. 
Thus, we can show that 
$\Phi\VDash_\Pi \pprob(\alpha) \geq 0$ iff $\Phi\VDash_\Pi \top$. 

\item Proof rules 2 to 4 can be proved in a similar fashion.  
\end{itemize} 
 
We now consider the \textbf{enumeration rules}:
\begin{itemize}
\item Consider proof rule 1. We need to show that 
$\Phi\VDash_\Pi f_1 \ineq x$ iff $\Phi\VDash_\Pi \bigvee_{(v_1,\ldots,v_k) \in \Pi^{x, \arop(f_1)}_{\ineq}} (\pprob(\alpha_1) = v_1 \land \pprob(\alpha_2) = v_2 \land \ldots \land \pprob(\alpha_k)= v_k)$
 if $\Pi^{x, \arop(f_1)}_{\ineq} \neq \emptyset$ and $\Phi \VDash_\Pi \bot$ otherwise.
We first consider $\ineq$ being $>$ and start with the case where $\Pi^{x, \arop(f_1)}_{>} = \emptyset$.
Based on Proposition \ref{restrnonempty} and the fact that $\{0,1\} \subseteq \Pi$ (note that $\Pi$ is a reasonable restricted
value set), it therefore holds that $x=1$ and either $\arop(f_1) = ()$ or for no $*_i$, $*_i = +$. 
If $\arop(f_1) = ()$, then $f_1: \pprob(\alpha_1) >1$ and it is easy to see that 
$\sat(\pprob(\alpha_1) >1, \Pi) = \emptyset = \sat(\bot, \Pi)$. If for every $*_i$, $*_i = -$, 
then based on the fact that probabilities belong to the unit interval, 
$\pprob(\alpha_1) - \pprob(\alpha_2) - \ldots - \pprob(\alpha_k) \leq \pprob(\alpha_1)$. 
Hence, if $\sat(\pprob(\alpha_1) >1, \Pi) = \emptyset$, then $\sat(f_1 >1, \Pi) = \emptyset$ as well. 
Thus, if $\Pi^{x, \arop(f_1)}_{>} = \emptyset$, $\Phi\VDash_\Pi f_1 \ineq x$ iff $\Phi \VDash_\Pi \bot$. 

Now consider the case where 
$\Pi^{x, \arop(f_1)}_{\ineq} \neq \emptyset$. 
For every $\prob' \in \sat(f_1 > x,\Pi)$, 
$\prob'(\alpha_1) *_1 \prob'(\alpha_2) *_2 \ldots *_{k-1} \prob'(\alpha_k) >x$. 
Consequently, $(\prob'(\alpha_1), \ldots, \prob'(\alpha_k)) \in \Pi^{x, \arop(f_1)}_{>}$.
We can therefore show that 
$\sat(f_1 >x,\Pi) \subseteq \sat(\bigvee_{(v_1,\ldots,v_k) \in \Pi^{x, \arop(f_1)}_{>}} (\pprob(\alpha_1) = v_1 \land \pprob(\alpha_2) = v_2 \land \ldots \land \pprob(\alpha_k)= v_k), \Pi)$. 
Let now $\prob' \in \sat(\bigvee_{(v_1,\ldots,v_k) \in \Pi^{x, \arop(f_1)}_{>}} (\pprob(\alpha_1) = v_1 \land \pprob(\alpha_2) = v_2 \land \ldots \land \pprob(\alpha_k)= v_k), \Pi)$. Based on the properties of $\sat$, it means that 
there is $(v_1,\ldots,v_k) \in \Pi^{x, \arop(f_1)}_{>}$ s.t. $\prob' \in \sat((\pprob(\alpha_1) = v_1 \land \pprob(\alpha_2) = v_2 \land \ldots \land \pprob(\alpha_k)= v_k), \Pi)$. Since $v_1 *_1 v_2 *_2 \ldots *_{k-1} v_k >x$, then 
$\prob'(\alpha_1) *_1 \prob'(\alpha_2)*_2 \ldots *_{k-1} \prob'(\alpha_k) > x$. Hence, $\prob' \in \sat(f_1>x,\Pi)$, and we 
can show that
$\sat(\bigvee_{(v_1,\ldots,v_k) \in \Pi^{x, \arop(f_1)}_{>}} (\pprob(\alpha_1) = v_1 \land \pprob(\alpha_2) = v_2 \land \ldots \land \pprob(\alpha_k)= v_k), \Pi) \subseteq \sat(f_1 >x,\Pi)$. 
 Given the previous result, $\sat(f_1 >x,\Pi) = \sat(\bigvee_{(v_1,\ldots,v_k) \in \Pi^{x, \arop(f_1)}_{>}} (\pprob(\alpha_1) = v_1 \land \pprob(\alpha_2) = v_2 \land \ldots \land \pprob(\alpha_k)= v_k), \Pi)$ and therefore
$\Phi\VDash_\Pi f_1 > x$ iff $\Phi\VDash_\Pi \bigvee_{(v_1,\ldots,v_k) \in \Pi^{x, \arop(f_1)}_{>}} (\pprob(\alpha_1) = v_1 \land \pprob(\alpha_2) = v_2 \land \ldots \land \pprob(\alpha_k)= v_k)$. 
The results for other operators can be obtained in a similar fashion.

\item Consider proof rule 2. We need to show that
$\Phi\VDash_\Pi f_1 > x$ iff $\Phi\VDash_\Pi \neg(\bigvee_{(v_1,\ldots,v_k) \in \Pi^{x, \arop(f_1)}_{\leq}} (\pprob(\alpha_1) = v_1 \land \pprob(\alpha_2) = v_2 \land \ldots \land \pprob(\alpha_k)= v_k))$. 
For every $\prob' \in \sat(f_1 > x,\Pi)$, it holds that 
$\prob'(\alpha_1) *_1 \prob'(\alpha_2) *_2 \ldots *_{k-1} \prob'(\alpha_k) >x$. 
Consequently, we can observe that $(\prob'(\alpha_1), \ldots, \prob'(\alpha_k)) \in \Pi^{x, \arop(f_1)}_{>}$ 
and $(\prob'(\alpha_1), \ldots, \prob'(\alpha_k)) \notin \Pi^{x, \arop(f_1)}_{\leq}$.
Hence, it holds that
\begin{multline*}
\sat(f_1 > x,\Pi) \subseteq \sat(\top,\Pi)\setminus\sat(\bigvee_{(v_1,\ldots,v_k) \in \Pi^{x, \arop(f_1)}_{\leq}} (\pprob(\alpha_1) = v_1 \land \pprob(\alpha_2) = v_2 \land \ldots \land \pprob(\alpha_k)= v_k),\Pi)\\
= \sat(\neg(\bigvee_{(v_1,\ldots,v_k) \in \Pi^{x, \arop(f_1)}_{\leq}} (\pprob(\alpha_1) = v_1 \land \pprob(\alpha_2) = v_2 \land \ldots \land \pprob(\alpha_k)= v_k)),\Pi)
\end{multline*}
Let now $\prob' \in \sat(\neg(\bigvee_{(v_1,\ldots,v_k) \in \Pi^{x, \arop(f_1)}_{\leq}} (\pprob(\alpha_1) = v_1 \land \pprob(\alpha_2) = v_2 \land \ldots \land \pprob(\alpha_k)= v_k)), \Pi)$.
This means 
that $\prob' \in \sat(\top,\Pi)\setminus\sat(\bigvee_{(v_1,\ldots,v_k) \in \Pi^{x, \arop(f_1)}_{\leq}} (\pprob(\alpha_1) = v_1 \land \pprob(\alpha_2) = v_2 \land \ldots \land \pprob(\alpha_k)= v_k),\Pi)$. Hence, for every $\prob'$, 
$\prob'(\alpha_1) *_1 \prob'(\alpha_2)*_2 \ldots *_{k-1} \prob'(\alpha_k) > x$ and therefore $\prob' \in \sat(f_1 > x, \Pi)$. 
 Given the previous result, this means that 
$\sat(f_1 > x,\Pi)= \sat(\neg(\bigvee_{(v_1,\ldots,v_k) \in \Pi^{x, \arop(f_1)}_{\leq}} (\pprob(\alpha_1) = v_1 \land \pprob(\alpha_2) = v_2 \land \ldots \land \pprob(\alpha_k)= v_k)), \Pi)$. Hence, 
$\Phi\VDash_\Pi f_1 > x$ iff $\Phi\VDash_\Pi \neg(\bigvee_{(v_1,\ldots,v_k) \in \Pi^{x, \arop(f_1)}_{\leq}} (\pprob(\alpha_1) = v_1 \land \pprob(\alpha_2) = v_2 \land \ldots \land \pprob(\alpha_k)= v_k))$. 

\item Soundness for the proof rules 3 to 5 is obtained in the same way as for rule 2. 
For rule 3, we observe that if $\Pi^{x, \arop(f_1)}_{<} = \emptyset$, then 
based on the properties of $\Pi$, $x=0$ and either $\arop(f_1) = ()$ or for no $*_i$, $*_i = -$. 
If $\arop(f_1) = ()$, then $f_1 \formis \pprob(\alpha_1) \geq 0$ is a tautology based on the basic rules 
(and is therefore equivalent to $\neg \bot$). If for every $*_i$, $*_i = +$, 
then we can observe that $\pprob(\alpha_1) \leq \pprob(\alpha_1) + \ldots + \pprob(\alpha_k)$, 
hence $\sat(\pprob(\alpha_1) + \ldots + \pprob(\alpha_k) \geq 0, \Pi) = \sat(\top, \Pi) = \sat(\neg \bot)$.  
Similar observations can be made for rule 5.
\end{itemize}
 
The soundness of the \textbf{subject rules} can be easily proved by using Proposition \ref{prop:formsat}.
 
We now consider the \textbf{probabilistic rule}. 
We can observe that for any probability distribution $\prob$ and terms $\alpha, \beta \in \terms(\graph)$, 
$\prob(\alpha \lor \beta) = \prob(\alpha) + \prob(\beta) - \prob(\alpha \land \beta)$. 
This can be easily checked by analyzing the definition of the probability of a term. 
Thus, it holds that
$\sat(\pprob(\alpha \lor \beta)  - \pprob(\alpha) -  \pprob(\beta) +  \pprob(\alpha \land \beta) = 0, \Pi) = \dist(\graph,\Pi)$. 
Since for any set of probabilistic formulae $\Phi$, $\sat(\Phi, \Pi) \subseteq \dist(\graph,\Pi)$, 
it holds that $\Phi\VDash_\Pi \pprob(\alpha \lor \beta)  - \pprob(\alpha) -  \pprob(\beta) +  \pprob(\alpha \land \beta) = 0$.

 We now consider the \textbf{propositional rules}. 
\begin{itemize}
\item Assume $\Phi\VDash_\Pi \phi_1$ and .... and  $\Phi\VDash_\Pi\phi_n$ and $\{\phi_1,\ldots,\phi_n\}\vdash\psi$.
So for each $i \in \{1,\ldots,n\}$, $\sat(\Phi, \Pi) \subseteq \sat(\phi_i, \Pi)$.
Furthermore, based on the definition of $\sat$, 
if $\{\phi_1,\ldots,\phi_n\}\vdash\psi$, 
then it holds that $\sat(\{\phi_1,\ldots,\phi_n\}) \subseteq \sat(\psi)$. Hence, 
$\sat(\{\phi_1,\ldots,\phi_n\},\Pi) \subseteq \sat(\psi, \Pi)$ as well. 
Therefore, $\sat(\Phi,\Pi) \subseteq \sat(\psi,\Pi)$.
Hence, $\Phi\VDash_\Pi \psi$.

\item Assume $\Phi \vdash \varphi$. Consequently, $\sat(\Phi)\subseteq \sat(\varphi)$. 
Hence, $\sat(\Phi, \Pi) \subseteq \sat(\varphi, \Pi)$ as well. Therefore, $\Phi \VDash_\Pi \varphi$. 
 \end{itemize} 
We have therefore shown that every rule in our system is sound. Hence, the system is sound as well. 
\end{proof}

In order to prove Proposition \ref{valdistributiondnf}, we also consider the following intermediate results: 
\begin{restatable}{proposition}{subvaldistributiondnf}
\label{subvaldistributiondnf}
Let $\ineq \in \operators$ be the set of inequality relationships, 
let $\Pi$ be a reasonable restricted value set, 
and let $\Pi^x_{\ineq} = \{ y \in \Pi \mid y \ineq x \}$ be the subset obtained according to the value $x$ and relationship $\ineq$. 
Let $\argcomplete(\graph) = \{c_1, \ldots, c_j\}$ be the collection of all argument complete propositional formulae for $\graph$ 
and $T^{\Pi}_{v,k} =  \Pi^{v,(+,\ldots,+)}_{=}$ s.t. the length of $(+,\ldots,+)$ is $k-1$  
be the collection of k--tuples of values from $\Pi$ that sum up to $v \in \Pi$. 
The following hold, where $\vdash$ is propositional consequence relation, 
$\Phi \subseteq \eformulae(\graph,\Pi)$, $\phi,\psi \in \eformulae(\graph,\Pi)$ and $x \in \Pi$:

\begin{enumerate}
\item for consistent formulae $\alpha_1, \ldots \alpha_m \in \terms(\graph)$, if for all $1 \leq i,j\leq m$ s.t. $i \neq j$ it holds that 
$\alpha_i \land \alpha_j \vdash \bot$, then 
$\Phi \Vdash_\Pi  \pprob(\alpha_1 \lor \ldots \lor \alpha_m) = x$ iff $\Phi \Vdash_\Pi  \bigvee_{(\tau_1, \ldots, \tau_m) \in T^\Pi_{x,m}} (\pprob(\alpha_1) = \tau_1 \land \ldots \land \pprob(\alpha_m) = \tau_m)$ 

\item $\Phi \Vdash_\Pi  \bigvee_{(\tau_1, \ldots, \tau_j) \in T^\Pi_{1,j}} (\pprob(c_1) =\tau_1 \land \ldots \land \pprob(c_j) = \tau_j)$ 
\end{enumerate}
\end{restatable}

\begin{proof} 

\begin{enumerate} 

\item We prove this property by induction. We show the cases for $m=1$ and, for clarity, $m=2$, and then prove
that if the property holds for $m=k$, then it holds for $m=k+1$ as well. 

\begin{itemize}
\item Let us start with $m=1$. Since $T^\Pi_{v,1} = \Pi^{v,()}_=$, the rule states that $\Phi \Vdash_\Pi \pprob(\alpha_1) = v$ iff 
$\Phi \Vdash_\Pi \pprob(\alpha_1)=v$, which is clearly true.

\item Let $m=2$. We start with the left to right direction. By derivable rule \ref{derivable:45}, it holds that $\Phi \Vdash_\Pi \pprob(\alpha_1 \lor \alpha_2) = x$ iff $\Phi \Vdash_\Pi \pprob(\alpha_1 \land \alpha_2) - \pprob(\alpha_1) - \pprob(\alpha_2) = x$. 
By enumeration rule E1, this is equivalent to 
$\Phi\Vdash_\Pi \bigvee_{(v_1,v_2,v_3) \in \Pi^{x, (+,-)}_{=}} 
(\pprob(\alpha_1) = v_1 \land \pprob(\alpha_2) = v_2 \land \pprob(\alpha_1 \land \alpha_2) = v_3 )$ 
(we note that based on Proposition \ref{restrnonempty} and the properties of $\Pi$, $\Pi^{x, (+,-)}_{=} \neq \emptyset$).  
Since $\alpha_1 \land \alpha_2 \vdash \bot$ by assumption, then based on derivable rule \ref{derivable:4} and P1
we can observe that for any $(v_1,v_2,v_3) \in \Pi^{x,(+,-)}_{=}$ s.t. $v_3\neq 0$, $(\pprob(\alpha_1) = v_1 \land  \pprob(\alpha_2) = v_2 \land \pprob(\alpha_1 \land \alpha_2) = v_3)$ is equivalent to $\bot$. 
Thus, using the propositional rule P1 once more, this is equivalent 
to 
$\Phi\Vdash_\Pi  \bigvee_{(v_1,v_2,0) \in \Pi^{x,(+,-)}_{=}} ( \pprob(\alpha_1) = v_1 \land  \pprob(\alpha_2) = v_2 \land \pprob(\alpha_1 \land \alpha_2) =0)$. 
Since $\pprob(\alpha_1 \land \alpha_2) = 0$ is a tautology, then again by P1, it holds that 
$\Phi \Vdash_\Pi  \bigvee_{(v_1,v_2,0) \in \Pi^{x,(+,-)}_{=}}  (\pprob(\alpha_1) = v_1 \land  \pprob(\alpha_2) =v_2)$, 
which is exactly $\Phi \Vdash_\Pi  \bigvee_{(\tau_1,\tau_2) \in T^\Pi_{x,2}}  (\pprob(\alpha_1) = \tau_1 \land  \pprob(\alpha_2) = \tau_2)$. The right to left direction can be shown in a similar fashion.

\item  
Assume that for $m=k$ our property is true, i.e. if our assumptions hold, then
$\Phi \Vdash_\Pi \pprob(\alpha_1 \lor \ldots \lor \alpha_k) = x$ iff $\Phi \Vdash_\Pi \bigvee_{(\tau_1, \ldots, \tau_k) \in T^\Pi_{x,k}} (\pprob(\alpha_1) = \tau_1 \land \ldots \land \pprob(\alpha_k) = \tau_k)$.
We need to show that for $m=k+1$, if our assumptions hold, then 
$\Phi \Vdash_\Pi \pprob(\alpha_1 \lor \ldots \lor \alpha_{k+1}) = x$ iff $\Phi \Vdash_\Pi \bigvee_{(\tau_1, \ldots, \tau_{k+1}) \in T^\Pi_{x,{k+1}}} (\pprob(\alpha_1) = \tau_1 \land \ldots \land \pprob(\alpha_{k+1}) = \tau_{k+1})$. 
We focus on the left to right direction first. 
By derivable rule \ref{derivable:45}, it holds that $\Phi \Vdash_\Pi \pprob(\alpha_1 \lor \ldots \lor \alpha_{k+1}) = x$ iff 
$\Phi \Vdash_\Pi \pprob(\alpha_1 \lor \ldots \lor \alpha_{k}) + \pprob(\alpha_{k+1}) - \pprob((\alpha_1 \lor \ldots \lor \alpha_{k})  \land \alpha_{k+1}) = x$. 
By enumeration rule E1, this is equivalent to 
$\Phi\Vdash_\Pi \bigvee_{(v_1,v_2,v_3) \in \Pi^{x, (+,-)}_{=}} (\pprob(\alpha_1 \lor \ldots \lor \alpha_{k}) = v_1 
\land \pprob(\alpha_{k+1}) = v_2 \land \pprob((\alpha_1 \lor \ldots \lor \alpha_{k})  \land \alpha_{k+1}) =v_3)$ 
(we note that based on Proposition \ref{restrnonempty} and the properties of $\Pi$, $\Pi^{x, (+,-)}_{=} \neq \emptyset$).

Given that for $i \in \{1, \ldots, k\}$, $\alpha_{k+1} \land \alpha_i \vdash \bot$, 
then based on derivable rule \ref{derivable:4} and propositional rule P1 
we can observe that for any $(v_1,v_2,v_3) \in \Pi^{v,(+,-)}_{=}$ s.t. $v_3\neq 0$, $(\pprob(\alpha_1 \lor \ldots \lor \alpha_k) = v_1 \land  \pprob(\alpha_{k+1}) = v_2 \land  \pprob((\alpha_1\lor \ldots \lor \alpha_k) \land \alpha_{k+1}) = v_3)$ 
is equivalent to $\bot$. Thus, using the propositional rule P1 once more, our formula 
is equivalent to 
$\Phi\Vdash_\Pi  \bigvee_{(v_1,v_2,0) \in \Pi^{v,(+,-)}_{=}}   (\pprob(\alpha_1 \lor \ldots \lor \alpha_k) = v_1 \land  \pprob(\alpha_{k+1}) = v_2 \land  \pprob((\alpha_1\lor \ldots \lor \alpha_k) \land \alpha_{k+1}) = 0)$.
Since $(\alpha_1\lor \ldots \lor \alpha_k) \land \alpha_{k+1})$ is equivalent to $\bot$, then 
$\pprob((\alpha_1\lor \ldots \lor \alpha_k) \land \alpha_{k+1}) = 0$ is a tautology. Thus, again by P1, our formula is equivalent to 
$\Phi \Vdash_\Pi  \bigvee_{(v_1,v_2,0) \in \Pi^{v,(+,-)}_{=}} (\pprob(\alpha_1 \lor \ldots \lor \alpha_k) = v_1 \land  \pprob(\alpha_{k+1}) = v_2)$, 
which is the same as 
$\Phi \Vdash_\Pi  \bigvee_{(\tau_1,\tau_2)  \in T^\Pi_{x,2}}  (\pprob(\alpha_1 \lor \ldots \lor \alpha_k) = \tau_1 \land  \pprob(\alpha_{k+1}) = \tau_2)$.  
Since the proposition holds for $m=k$ by assumption, then our formula is equivalent to: 
$$\Phi \Vdash_\Pi  \bigvee_{(\tau_1,\tau_2)  \in T^\Pi_{x,2}} ((\bigvee_{(\tau'_1,\ldots, \tau'_k) \in T^\Pi_{\tau_1,k}}(\pprob(\alpha_1) = \tau'_1 \land \ldots 
\land \pprob(\alpha_k) = \tau'_k)) \land  \pprob(\alpha_{k+1}) = \tau_2)$$
 
Now, as $\Pi$ is finite, then so is $T^\Pi_{x,2}$.
Let us assume that $T^\Pi_{x,2}$ is of the form 
$\{ (\tau_{1,1}, \tau_{2,1})$, $(\tau_{1,2}, \tau_{2,2})$, $\ldots$, $(\tau_{1,p}, \tau_{2,p})\}$. 
Our formula can therefore be rewritten as
\begin{multline*}
\Phi \Vdash_\Pi  ((\bigvee_{(\tau'_{1,1},\ldots, \tau'_{k,1}) \in T^\Pi_{\tau_{1,1},k}}(\pprob(\alpha_1) = \tau'_{1,1} \land \ldots 
\land \pprob(\alpha_k) = \tau'_{k,1})) \land  \pprob(\alpha_{k+1}) = \tau_{2,1})\\
\lor
\ldots
\lor \\
((\bigvee_{(\tau'_{1,p},\ldots, \tau'_{k,p}) \in T^\Pi_{\tau_{1,p},k}}(\pprob(\alpha_1) = \tau'_{1,p} \land \ldots 
\land \pprob(\alpha_k) = \tau'_{k,p})) \land  \pprob(\alpha_{k+1}) = \tau_{2,p})
\end{multline*}

Now, as $\Pi$ is finite, then so is every $T^\Pi_{\tau_{1,i},k}$ for $i \in \{1,\ldots,p\}$. 
Assume that $|T^\Pi_{\tau_{1,i},k}| = l_i$ and therefore that 
every $T^\Pi_{\tau_{1,i},k}$ is of the form
$\{ ({\tau'}_{1,i}^1, \ldots, {\tau'}_{k,i}^1), \ldots, ({\tau'}_{1,i}^{l_i}, \ldots, {\tau'}_{k,i}^{l_i})\}$. 
Thus, our formula can again be rewritten as:
\begin{multline*}
\Phi \Vdash_\Pi  
(
((\pprob(\alpha_1) = {\tau'}_{1,1}^1 \land \ldots 
\land \pprob(\alpha_k) = {\tau'}_{k,1}^1)
\lor 
\ldots 
\lor
(\pprob(\alpha_1) = {\tau'}_{1,1}^{l_1} \land \ldots 
\land \pprob(\alpha_k) = {\tau'}_{k,1}^{l_1}))\\
 \land  \pprob(\alpha_{k+1}) = \tau_{2,1})\\
\lor
\ldots
\lor \\
(
((\pprob(\alpha_1) = {\tau'}_{1,1}^1 \land \ldots 
\land \pprob(\alpha_k) = {\tau'}_{k,1}^1)
\lor 
\ldots 
\lor
(\pprob(\alpha_1) = {\tau'}_{1,p}^{l_p} \land \ldots 
\land \pprob(\alpha_k) = {\tau'}_{k,p}^{l_p}))\\
 \land  \pprob(\alpha_{k+1}) = \tau_{2,p})
\end{multline*}
which by propositional rule P1 is equivalent to 
\begin{multline*}
\Phi \Vdash_\Pi  
(\pprob(\alpha_1) = {\tau'}_{1,1}^1 \land \ldots 
\land \pprob(\alpha_k) = {\tau'}_{k,1}^1 \land  \pprob(\alpha_{k+1}) = \tau_{2,1})\\
\lor 
\ldots 
\lor\\
(\pprob(\alpha_1) = {\tau'}_{1,1}^{l_1} \land \ldots 
\land \pprob(\alpha_k) = {\tau'}_{k,1}^{l_1}) \land  \pprob(\alpha_{k+1}) = \tau_{2,1}) \\
\lor
\ldots
\lor \\
(\pprob(\alpha_1) = {\tau'}_{1,1}^1 \land \ldots 
\land \pprob(\alpha_k) = {\tau'}_{k,1}^1\land  \pprob(\alpha_{k+1}) = \tau_{2,p} )\\
\lor 
\ldots 
\lor\\
(\pprob(\alpha_1) = {\tau'}_{1,p}^{l_p} \land \ldots 
\land \pprob(\alpha_k) = {\tau'}_{k,p}^{l_p}\land  \pprob(\alpha_{k+1}) = \tau_{2,p} )
\end{multline*}

We can observe that every conjunctive clause we obtain is in fact a combinations of 
values from $T^\Pi_{x,k+1}$. 
Furthermore, if 
$(v_1, \ldots, v_{k+1}) \in T^\Pi_{x,k+1}$ then $(v_1+\ldots+v_k, v_{k+1}) \in T^\Pi_{x,2}$ 
 -- since $v \in [0,1]$ and we are dealing with addition only, then $v_1+\ldots+v_k \in \Pi$ as well. 
Thus, for every combination from $T^\Pi_{x,k+1}$ 
we can also find a conjunctive clause associated with it, and our formula is in fact equivalent to
$\Phi \Vdash_\Pi \bigvee_{(\tau''_1, \ldots, \tau''_{k+1}) \in T^\Pi_{x,k+1}} (\pprob(\alpha_1) = \tau''_1 \land \ldots \land \pprob(\alpha_{k+1}) = \tau''_{k+1})$. 
Since all of our transformation were using syntactical equivalences, the right to left direction of our proof can be shown in a similar fashion.


\end{itemize}


 \item Based on derivable rule \ref{derivable:3}, $\Phi \Vdash_\Pi \pprob(\top)=1$. 
We can observe that $c_1 \lor \ldots  c_j \vdash \top$ and $\top \vdash c_1 \lor \ldots  c_j$. 
Furthermore, every $c_i$ is consistent and for every $i,k$ s.t. $i \neq k$, $c_i \land c_k \vdash \bot$.  
Thus, $\Phi \Vdash_\Pi \pprob(c_1 \lor \ldots  c_j)=1$, and by the previous point of this proof, 
$\Phi \Vdash_\Pi \bigvee_{(\tau_1, \ldots, \tau_j) \in T^\Pi_{1,j}} (\pprob(c_1) =\tau_1 \land$  $\ldots$ $\land \pprob(c_j) = \tau_j)$. 
\end{enumerate}

\end{proof}

\valdistributiondnf*

\begin{proof}

We now show that $\Phi \Vdash_\Pi \psi$ iff $\Phi \Vdash_\Pi \varphi$, where $\varphi$ is the distribution disjunctive 
normal form of $\psi$. 
Let $\argcomplete(\graph) = \{c_1, \ldots, c_j\}$ be the collection of all argument complete terms for $\graph$. 
First of all, we can observe that for any term $\alpha \in \terms(\graph)$, we can find an equivalent formula 
that is either $\bot$ if $\alpha$ is inconsistent, or which is equivalent to $c'_1 \lor \ldots \lor c'_k$, where $C' = \{c'_1,\ldots, 
c'_k\} \subseteq \argcomplete(\graph)$ 
is a nonempty collection of argument complete terms. This form of $\alpha$ is called the full disjunctive normal form, 
and can be easily found by e.g. constructing (propositional) 
DNF using the 
truth table method. Let $\alpha'$ be this form of $\alpha$. 
Through the use of derivable rule \ref{derivable:25}, it is easy to show that $\Phi \Vdash_\Pi \pprob(\alpha) = x$ iff 
 $\Phi \Vdash_\Pi \pprob(\alpha') = x$. It therefore suffices to focus on $\alpha'$ in the remainder of this proof. 

Let us consider $\psi: \pprob(\alpha') = x$. 
%
Since $\{\alpha\} \vdash \alpha'$ and $\{\alpha'\} \vdash \alpha$, 
then by derivable rules \ref{derivable:25} and \ref{derivable:26} 
it holds that $\Phi \Vdash_\Pi  \pprob(\alpha) = x$ iff $\Phi \Vdash_\Pi  \pprob(\alpha') = x$. 
We can now observe that $\alpha'$ is either a propositional contradiction ($\bot$), tautology ($c_1 \lor \ldots \lor c_j$), 
or neither. Under various values $x$, the probabilistic atom $\pprob(\alpha') = x$ can become or remain a contradiction, tautology, 
or neither. We can therefore distinguish the following cases: 
\begin{enumerate}
\item $\alpha' \formis c_1 \lor \ldots \lor c_j$ and $x = 1$
\item $\alpha' \formis c_1 \lor \ldots \lor c_j$ and $x \neq 1$ 
\item $\alpha' \formis \bot$ and $x = 0$ 
\item $\alpha' \formis \bot$ and $x \neq 0$
\item remaining cases, i.e. $\alpha \neq \bot$ and $\alpha \neq c_1 \lor \ldots \lor c_j$
\end{enumerate}

Let us start with the most complicated, last point (i.e. point $e$). 
Since all argument complete formulae are consistent and for each two different formulae $c'_i, c'_l$ it holds that 
$c'_i \land  c'_l\vdash \bot$, then by using (proof of) Proposition \ref{subvaldistributiondnf}, we can show the following:

\[\Phi \Vdash_\Pi  \pprob(\alpha') = x \mbox{ iff }\Phi \Vdash_\Pi  \bigvee_{(\tau'_1, \ldots, \tau'_k) \in T^\Pi_{x,k}} (\pprob(c'_1) =\tau'_1 \land \ldots \land \pprob(c'_k) = \tau'_k)\]
 
Let $D =\{d'_1,\ldots, d'_f\} = \argcomplete(\graph)\setminus C'$ be the collection of argument complete formulae
not appearing in $\alpha'$. 
If $D=\emptyset$, then all of the argument complete terms appear in our formula, which means that 
$\alpha' \formis c_1 \lor \ldots \lor c_j$. Consequently, depending on the value of $x$, please see the analysis 
of points $a$ and $b$ which is explained below.  
If $D \neq \emptyset$, then we can observe that every conjunctive clause in our formula is a partial description of (possibly more 
than one) probability distribution. We can therefore use the second point of this proposition, which in fact enumerates all possible 
restricted probability distribution on $\graph$, to \enquote{complete} our partial descriptions. 
Using the propositional rules and (proof of) Proposition \ref{subvaldistributiondnf}, we can show that 
\begin{align*}
&\Phi \Vdash_\Pi   \bigvee_{(\tau'_1, \ldots, \tau'_k) \in T^\Pi_{x,k}} (\pprob(c'_1) =\tau'_1 \land \ldots \land \pprob(c'_k) = \tau'_k)
\mbox{ iff } \\
&  \Phi \Vdash_\Pi 
(\bigvee_{(\tau'_1, \ldots, \tau'_k) \in T^\Pi_{x,k}} (\pprob(c'_1) =\tau'_1 \land \ldots \land \pprob(c'_k) = \tau'_k)) 
\land
(\bigvee_{(\tau_1, \ldots, \tau_j) \in T^\Pi_{1,j}} (\pprob(c_1) =\tau_1 \land \ldots \land \pprob(c_j) = \tau_j))
\end{align*}

Let $V = \{(\tau_1, \ldots, \tau_j) \in T^\Pi_{1,j} \mid f((\tau_1, \ldots, \tau_j)) \in T^\Pi_{x,k}\}$, where
$f((\tau_1, \ldots, \tau_j))$ returns a sub--tuple $(\tau'_1, \ldots, \tau'_k)$ of $(\tau_1, \ldots, \tau_j)$ 
s.t. if $c'_g \formeq c_h$ then $\tau'_g = \tau_h$, be a collection of tuples in $T^\Pi_{1,j}$ that preserve 
the assignments from $T^\Pi_{x,k}$.
Using the propositional rules, the second formula can be equivalently rewritten as  
\[\Phi \Vdash_\Pi  \bigvee_{(\tau_1, \ldots, \tau_j) \in V} (\pprob(c_1) =\tau_1 \land \ldots \land \pprob(c_k) = \tau_j)\] 
if $V \neq \emptyset$ 
and as 
$\Phi \Vdash_\Pi \bot$ otherwise. 
We have therefore obtained an epistemic formula in (propositional) DNF s.t. it is either $\bot$, or 
every conjunctive clause is built out of atoms of the 
form $\pprob(\beta) = x$ where $\beta$ is an argument complete term and such that 
for every argument complete term $\beta'$ there is a probabilistic atom $\pprob(\beta') = x'$ in the clause. 
Let us refer to this formula as $\psi'$. 
Based on the presented procedure and the fact that our system is sound (see Proposition \ref{prop:restrictedvalsound}), 
we can show that $\sat(\psi, \Pi) = \sat(\psi', \Pi)$. Hence, 
if $\psi' \formis \bot$, then $\sat(\psi, \Pi) = \sat(\psi'\, \Pi) = \emptyset$, and it is easy to show that $\psi'$ is indeed 
the DDNF of $\psi$. Therefore, let us focus on the case where $\psi'$ is not $\bot$, and enumerate 
the conjunctive clauses  as $\psi'_1, \ldots, \psi'_p$. 
Let $\varphi \formis  \varphi_1 \lor \ldots \lor \varphi_r$ 
be a DDNF of $\psi$. We will now show that $\psi'$ and $\varphi$ are equivalent under commutativity.  
By Proposition \ref{satdistrdnfform}, it holds that $\sat(\psi', \Pi) = \sat(\varphi, \Pi)$.
Hence, by the properties of $\sat$, $\sat(\psi'_1,\Pi) \cup  \ldots \cup \sat(\psi'_p,\Pi) = \sat(\varphi_1, \Pi) \cup \ldots \cup 
\sat(\varphi_r, \Pi)$. We can observe that every $\psi'_i \in \{\psi'_1, \ldots, \psi'_p\}$ and 
$\varphi_s \in \{\varphi_1 \lor \ldots \lor \varphi_r\}$ is in fact an epistemic formula 
associated with a single unique probability distribution. Consequently, for every $\psi'_i$ we can find a formula $\varphi_s$ 
s.t. $\sat(\psi'_i,\Pi) = \sat(\varphi_s, \Pi)$ and vice versa. Given the form of these formulae, it therefore has to follow 
that $\psi'_i \vdash \varphi_s$ and $\varphi_s \vdash \psi'_i$. Hence, $\psi'$ is a DDNF 
of $\psi$, and we have shown that for $\psi: \pprob(\alpha) = x$, $\Phi \Vdash_\Pi \psi$ iff $\Phi \Vdash_\Pi \psi'$. 

Let us now consider point $a$, i.e. where $\alpha' \formis c_1 \lor \ldots \lor c_j$ and $x = 1$. 
We can observe that by the construction of $\alpha'$, this means that $\alpha'$ (and $\alpha$) are equivalent to $\top$. 
By repeating the procedures for the last point, we can show that for $\alpha'$, 
$\Phi \Vdash_\Pi  \pprob(\alpha') = 1$ iff $\Phi \Vdash_\Pi  \bigvee_{(\tau_1, \ldots, \tau_j) \in T^\Pi_{1,j}} 
(\pprob(c_1) =\tau_1 \land \ldots \land \pprob(c_j) = \tau_j)$. Given that $\pprob(\top) = 1$ is a tautology 
(see derivable rules), 
we can use derivable rule \ref{derivable:23} to show that $\pprob(\alpha') = 1$ is a tautology as well. Hence, it is easy to check 
that the obtained formula is indeed a DDNF of $\psi$. Please note that the same formula 
can also be obtained through the use of the propositional rules, basic rule B3 and the second point of this proposition. 

Let us now consider point $b$, i.e. where $\alpha' \formis c_1 \lor \ldots \lor c_j$ and $x \neq 1$. 
Based on derivable rule \ref{derivable:25}, $\Phi \Vdash_\Pi  \pprob(\alpha') = x$ iff $\Phi \Vdash_\Pi  \pprob(\top) = x$. 
Using derivable rules \ref{derivable:3} and 
\ref{derivable:35}, we can show this to be a contradiction. Hence, $\bot$ is indeed the DDNF of $\psi$. 

Let us now consider point $c$, i.e. $\alpha' \formis \bot$ and $x = 0$. By combining basic rules 
B3 and B4 we can observe that $\Phi \Vdash_\Pi  \pprob(\bot) = 0$ iff $\Phi \Vdash_\Pi  \pprob(\top) = 1$. 
We can therefore show that 
$\Phi \Vdash_\Pi  \pprob(\bot) = 0$ iff $\Phi \Vdash_\Pi  \bigvee_{(\tau_1, \ldots, \tau_j) \in T^\Pi_{1,j}} (\pprob(c_1) =\tau_1 \land \ldots \land \pprob(c_j) = \tau_j)$. Given that $\sat(\pprob(\bot) = 0, \Pi) = \dist(\graph, \Pi) = 
\sat(\bigvee_{(\tau_1, \ldots, \tau_j) \in T^\Pi_{1,j}} (\pprob(c_1) =\tau_1 \land \ldots \land \pprob(c_j) = \tau_j))$, 
it is easy to verify that the obtained formula is indeed the DDNF of $\psi$. 

Finally, consider point $d$, i.e. where $\alpha' \formis \bot$ and $x \neq 0$. Based on the derivable rules \ref{derivable:4} and 
\ref{derivable:35}, 
we can observe that $\Phi \Vdash_\Pi \pprob(\bot) = 0$
and that for $x \neq 0$, 
$\Phi \Vdash_\Pi \pprob(\bot) = 0 \land \pprob(\bot) = x$ iff $\Phi \Vdash_\Pi \bot$. 
We therefore observe that $\Phi \Vdash_\Pi \pprob(\alpha) = x$ iff $\Phi \Vdash_\Pi \bot$ 
in this case. Given that $\sat(\pprob(\bot) = x) = \emptyset$ for $x\neq 0$, it is easy to see 
that $\bot$ is indeed the DDNF for $\psi$. 

We can therefore conclude that for all $\psi$ of the form $\pprob(\alpha) = x$, 
$\Phi \Vdash_\Pi \psi$ iff $\Phi \Vdash_\Pi \varphi$, where $\varphi$ is the DDNF of $\psi$.

We now also consider $\psi \formis  \top$ and $\psi \formis  \bot$. 
If $\psi \formis  \top$, then we can, for example, use basic rule B3 
along with the previous parts of the proof to show that 
$\bigvee_{(\tau_1, \ldots, \tau_j) \in T^\Pi_{1,j}} (\pprob(c_1) =\tau_1 \land \ldots \land \pprob(c_j) = \tau_j)$ 
is the DDNF of $\psi$. If $\psi \formis  \bot$, then we can observe that it is already in DDNF.  

Let us therefore consider more complex epistemic formulae $\psi$.
We first bring $\psi$ to its (propositional) negation normal form. We then replace every negated 
epistemic atom with a positive one using the derivable rules 
\ref{derivable:15} to \ref{derivable:20}. 
Based on enumeration rule E1, every epistemic atom using $\ineq \in \{>, <, \geq, \leq, \neq\}$ can 
be equivalently expressed a disjunction of atoms using only equality or $\bot$. Hence, $\psi$ can be transformed into $\psi_=$
that uses only positive equality atoms or $\bot$ and s.t. $\Phi \Vdash_\Pi \psi$ iff $\Phi \Vdash_\Pi \psi_=$. 
Every term in an epistemic atom can be transformed into $\bot$ or a disjunction of certain argument complete formulae. 
Furthermore, every epistemic atom containing a term equivalent to $\bot$ can, depending on the value $x$ contained 
in the atom, be replaced by an epistemic atom $\bot$ or $\top$ through the use of the propositional and basic rules. 
In a similar fashion, every epistemic atom containing a term equivalent to $\top$ (i.e. one using all possible argument 
complete formulae) can, depending on the value of $x$, be replaced by an epistemic atom $\bot$ or $\top$. 
The resulting epistemic formula can be transformed into a minimal (propositional) conjunctive normal form 
and if required, various additional derivable rules and the propositional identity and domination laws\footnote{Recall that the identity and 
domination logical equivalence laws 
state that $q \land \top \equiv q$, $q \lor \bot\equiv q$, $q \lor \top \equiv \top$ and $q \land \bot\equiv \bot$.} 
can be used to further simplify it. 
We therefore obtain a formula $\psi'$ s.t. $\psi' \formis \top$, or $\psi' \formis \bot$, or 
$\psi' \formis \psi'_1 \land \ldots \land \psi'_a$ 
where $a \leq 1$ and every $\psi'_i \formis \pprob(\alpha_{i_1}) = x_{i_1} \lor \ldots \lor \pprob(\alpha_{i_n}) = x_{i,n}$ s.t. 
every $\alpha_{i_k}$ is a disjunction of certain argument complete formulae and is not equivalent to neither $\bot$ nor $\top$.   
Given the used syntactical rules, 
we can observe that $\Phi \Vdash_\Pi \psi$ iff $\Phi \Vdash_\Pi \psi'$. If $\psi' = \top$ or $\psi'=\bot$, then we refer the reader 
to the previous part of this proof concerning how $\psi'$ can be transformed to appropriate DDNF. 
We therefore focus on the case where $\psi' = \psi'_1 \land \ldots \land \psi'_a$. 

For every $\psi'_i = \pprob(\alpha_{i_1}) = x_{i_1} \lor \ldots \lor \pprob(\alpha_{i_n}) = x_{i_n}$, 
every $\pprob(\alpha_{i_k}) = x_{i_k}$ can be brought to DDNF formula $\varphi_{i_k}$ 
using the previous parts of this proof. 
Through the use of associativity, every $\psi'_i$ can be equivalently 
written down as a disjunction of epistemic formulae  $\varphi_{i_1} \lor \ldots \lor \varphi_{i_{n_i}}$ 
where each formula $\varphi_{i_k}$ is associated with a probability distribution and 
s.t. $\Phi \Vdash_\Pi \psi'_i$ iff $\Phi \Vdash_\Pi \varphi_{i_1} \lor \ldots \lor \varphi_{i_{n_i}}$. 
Without loss of generality, we can assume that every two $\varphi_{i_k}$ and $\varphi_{i_l}$ formulae
are different - otherwise the idempotent law can be used to remove duplicate formulae. 
We now therefore have that $\Phi \Vdash_\Pi \psi$ iff $\Phi \Vdash_\Pi (\varphi_{1_1} \lor \ldots \lor \varphi_{1_{n_1}}) \land 
\ldots \land (\varphi_{a_1} \lor \ldots \lor \varphi_{a_{n_a}})$. 
Using the propositional rules and the derivable rule \ref{derivable:35}, we can show that for any two formulae $\varphi_{i_k}, \varphi_{m_l} 
\in \{\varphi_{1_1},\ldots, \varphi_{1_{n_1}},\ldots ,\varphi_{a_1} ,\ldots,\varphi_{a_{n_a}}\}$ 
associated with probability distributions,  
if $\varphi_{i_k} \nformeq \varphi_{m_l}$ then $\Phi \Vdash_\Pi  \varphi_{i_k} \land \varphi_{m_l}$ iff
$\Phi \Vdash_\Pi \bot$.   
Using this and the distribution laws, we can bring 
$(\varphi_{1_1} \lor \ldots \lor \varphi_{1_{n_1}}) \land 
\ldots \land (\varphi_{a_1} \lor \ldots \lor \varphi_{a_{n_a}})$ into minimal (propositional) disjunctive normal form $\gamma$
s.t. it is either $\bot$ or every conjunctive clause is an epistemic formula associated with a distribution. 
Due to the nature of the syntactic rules we have used, we can observe that 
$\Phi \Vdash_\Pi \psi$ iff $\Phi \Vdash_\Pi \gamma$. Since our system is sound, it holds that 
$\Phi \VDash_\Pi \psi$ iff $\Phi \VDash_\Pi \gamma$. By using the same argument as in the case of 
$\psi: \pprob(\alpha) = x$ in the previous parts of this proof, we can show that $\gamma$ is indeed the 
DDNF of $\psi$.

We can therefore conclude that for all $\psi \in \eformulae(\graph, \Pi)$, 
$\Phi \Vdash_\Pi \psi$ iff $\Phi \Vdash_\Pi \varphi$, where $\varphi$ is the DDNF of $\psi$. 

\end{proof}

\restrictedvalsoundcomp*

\begin{proof} 
We have shown that our system is sound in Proposition \ref{prop:restrictedvalsound}. 
We now need to show that the system is also complete, 
i.e. that if $\Phi\VDash_\Pi \psi$ then $\Phi \Vdash_\Pi \psi$.

Let $\Phi = \{\phi_1, \ldots, \phi_m\}$ and $\phi \formis \phi_1 \land \ldots \land \phi_m$. 
By using the propositional rules we can easily show that $\Phi \Vdash_\Pi \phi$ and for every formula 
$\gamma \in  \Phi$, $\{\phi\} \Vdash_\Pi \gamma$.
Furthermore, it clearly follows from the definition of $\sat$ that $\sat(\Phi, \Pi) = \sat(\phi, \Pi)$. 
Consequently, for the purpose of this proof, it suffices to show that 
if $\{\phi\} \VDash_\Pi \psi$ then $\{\phi\} \Vdash_\Pi \psi$. 
If $\Phi = \emptyset$, then it is easy to see that we can set $\phi$ to $\top$. 

Let $\sat(\phi, \Pi) = \{\prob_1, \ldots, \prob_k\}$ and $\sat(\psi, \Pi) = \{\prob'_1, \ldots, \prob'_l\}$. 
Let $\varphi^\phi$ and $\varphi^\psi$ be the DDNFs of $\phi$ and $\psi$. 
Based on Proposition \ref{satdistrdnfform}, $\sat(\psi, \Pi) = \sat(\varphi^\psi, \Pi)$ and 
$\sat(\phi,\Pi) = \sat(\varphi^\phi,\Pi)$. 
Therefore, $\{\phi\} \VDash_\Pi \psi$ iff $\{\varphi^\phi\} \VDash_\Pi \varphi^\psi$. 
By definition, $\{\varphi^\phi\} \VDash_\Pi \varphi^\psi$ iff $\sat(\varphi^\phi, \Pi) \subseteq \sat(\varphi^\psi, \Pi)$. 
If $k=0$, then the DDNF of $\phi$ is $\bot$ and therefore through the propositional rule P2, we can show that 
$\{\phi\} \Vdash_\Pi \psi$ for any $\psi$.
If $k \neq 0$ and $l = 0$, then it cannot be the case that $\sat(\varphi^\phi, \Pi) \subseteq \sat(\varphi^\psi, \Pi)$. 
Therefore, it cannot be the case that $\{\phi\} \VDash_\Pi \psi$ either and we reach a contradiction. 
Hence, let $k \neq 0$ and $l \neq 0$. 
Let $\varphi^\phi$ and $\varphi^\psi$ be of the forms $\varphi^\phi \formis \varphi^{\prob_1} \lor \varphi^{\prob_2} \ldots \lor \varphi^{\prob_k}$ 
and $\varphi^\psi \formis \varphi^{\prob'_1} \lor \varphi^{\prob'_2} \ldots \lor \varphi^{\prob'_l}$. 
Since $\sat(\varphi^\phi, \Pi) \subseteq \sat(\varphi^\psi, \Pi)$, then 
for every $\prob_i$ there exists $\prob'_j$ s.t. $\prob_i = \prob'_j$ and 
therefore for every $\varphi^{\prob_i}$ there exists an equivalent $\varphi^{\prob'_j}$.
Consequently, by using the propositional proof rule P1
it is easy to show that if $\{\varphi^\phi\} \VDash_\Pi \varphi^\psi$ then $\{\varphi^\phi\} \Vdash_\Pi \varphi^\psi$. 
We can now use Proposition \ref{valdistributiondnf} to show that $\{\phi\} \Vdash_\Pi \psi$. 
We can therefore conclude that if $\Phi\VDash_\Pi \psi$ then $\Phi\Vdash_\Pi \psi$ and our system is complete.    
\end{proof}

\negval*

\begin{proof} 
For the purpose of this proof, we introduce a shorthand $\varphi \formeq \varphi'$ stating 
that the two formulae have the same syntactical features under commutativity and associativity 
(so, for example, $(\xa \lor \xb) \land \xc \formeq (\xa \lor \xb) \land \xc$, $(\xa \lor \xb) \land \xc  \formeq 
(\xb \lor \xa) \land \xc$, but $(\xa \lor \xb) \land \xc \nformeq (\xa \land \xc) \lor (\xb \land \xc)$).  

We want to show that $\Phi\Vdash_\Pi \psi$ iff $\Phi \cup \{\neg \psi\} \Vdash_{\Pi}\bot$. 
Let $\Phi = \{\phi_1, \ldots, \phi_m\}$ and $\phi \formis \phi_1 \land \ldots \land \phi_m$. 
By using the propositional rules we can easily show that $\Phi \Vdash_\Pi \phi$ and for every formula 
$\gamma \in  \Phi$, $\{\phi\} \Vdash_\Pi \gamma$.
Furthermore, it clearly follows from the definition of $\sat$ that $\sat(\Phi, \Pi) = \sat(\phi, \Pi)$. 
If $\Phi = \emptyset$, then it is easy to see that we can set $\phi$ to $\top$. 
Let now $\varphi^\phi$ and $\varphi^\psi$ be DDNFs of $\phi$ and $\psi$. Based on Propositions \ref{prop:restrictedvalsound} 
and \ref{valdistributiondnf} and the definition of DDNF, we can observe that for the purpose of this proof, it suffices to show that 
$\{\varphi^\phi\} \Vdash_\Pi \varphi^\psi$ iff $\{\varphi^\phi \land \neg \varphi^\psi\} \Vdash_\Pi \bot$. 
Without loss of generality, let us assume that we have an ordering over arguments, and
if a given DDNF is not $\bot$, then the atoms in every subformula associated with a probability distribution 
are ordered according to the set of argument complete terms $\argcomplete(\graph) = \{c_1, \ldots, c_p\}$ of $\graph$
and that the terms themselves are also ordered. We can use this assumption due to 
derivable rule \ref{derivable:25} (i.e. two probabilistic atoms with same constants and argument complete terms equivalent under commutativity 
are themselves equivalent) and propositional rules (commutativity law). 

Let us focus on the left to right direction first. Assume $\varphi^\phi \formis \bot$. Then it is easy to see that 
$\{\varphi^\phi \land \neg \varphi^\psi\} \Vdash_\Pi \bot$ independently of the nature of $\neg \varphi^\psi$. 
We can also observe that if $\varphi^\psi \formis \bot$, then as $\{\varphi^\phi\} \Vdash_\Pi \varphi^\psi$, 
$\varphi^\phi \formis \bot$ as well. Thus, based on previous remarks, the property still holds. 

Let us therefore consider the situation in which none of the DDNFs are $\bot$ and let
$\varphi^\phi$ and $\varphi^\psi$ be of the forms
$\varphi^\phi \formis \varphi^\phi_1 \lor \ldots \lor \varphi^\phi_m$ and 
$\varphi^\psi \formis \varphi^\psi_1 \lor \ldots \lor \varphi^\psi_n$ for $m,n \geq 1$. 
We can observe that due to the nature of DDNF, every 
$\varphi^\phi_i$ and $\varphi^\psi_j$ is an epistemic formula associated with a (single) probability distribution. 
Therefore, if $\{\varphi^\phi\} \Vdash_\Pi \varphi^\psi$, 
then based on the ordering assumption we have made, it has to be the case that
$\{ \varphi^\phi_1, \ldots, \varphi^\phi_m\} \subseteq \{\varphi^\psi_1, \ldots ,\varphi^\psi_n\}$. 
Thus, for $\varphi^\phi_i$ there exists a $\varphi^\psi_j$ s.t. $\varphi^\phi_i \formeq \varphi^\psi_j$. 
Using the propositional rule, we can show that for such formulae, 
$\{\varphi^\phi_i \land \neg \varphi^\psi_j\} \Vdash_\Pi \bot$. 
Using this, we can now prove the following. 
Using propositional rules, 
 $\{\varphi^\phi \land \neg \varphi^\psi\} \Vdash_\Pi \varphi^\phi \land \neg \varphi^\psi$
which is 
 $\{\varphi^\phi \land \neg \varphi^\psi\} \Vdash_\Pi (\varphi^\phi_1 \lor \ldots \lor \varphi^\phi_m) 
\land 
\neg (\varphi^\psi_1 \lor \ldots \lor \varphi^\psi_n)$, 
which in turns is equivalent to 
 $\{\varphi^\phi \land \neg \varphi^\psi\} \Vdash_\Pi (\varphi^\phi_1 \lor \ldots \lor \varphi^\phi_m) 
\land  (\neg \varphi^\psi_1 \land \ldots \land \neg \varphi^\psi_n)$
and to 
 $\{\varphi^\phi \land \neg \varphi^\psi\} \Vdash_\Pi (\varphi^\phi_1 \land \neg \varphi^\psi_1 \land \ldots \land \neg \varphi^\psi_n) \lor \ldots \lor (\varphi^\phi_m \land \neg \varphi^\psi_1 \land \ldots \land \neg \varphi^\psi_n)$. 
Based on the previous explanations and propositional rules, 
for every $\varphi^\phi_i$, 
$\{\varphi^\phi_i \land \neg \varphi^\psi_1 \land \ldots \land \neg \varphi^\psi_n\} \vdash \bot$. 
Hence, using propositional rules once more, 
 $\{\varphi^\phi \land \neg \varphi^\psi\} \Vdash_\Pi \bot \lor \ldots \lor \bot$ 
and therefore 
 $\{\varphi^\phi \land \neg \varphi^\psi\} \Vdash_\Pi \bot$. 
This concludes the left to right direction of our proof. 

Let us now focus on the right to left direction of our proof, i.e. 
that if  $\{\varphi^\phi \land \neg \varphi^\psi\} \Vdash_\Pi \bot$ then $\{\varphi^\phi\} \Vdash_\Pi \varphi^\psi$. 
Assume $\varphi^\phi \formis \bot$. Then clearly, $\{\varphi^\phi\} \Vdash_\Pi \varphi^\psi$ based on propositional rules. 
Now assume $\varphi^\psi \formis \bot$. Then $\neg \varphi^{\psi}$ is equivalent to $\top$ through the use 
of propositional rules and our property holds easily 
and it is easy to argue that if 
$\{\varphi^\phi \land \top\} \Vdash_\Pi \bot$ then $\{\varphi^\phi\} \Vdash_\Pi \bot$.

Let us therefore consider the situation in which none of the DDNFs are $\bot$ and let
$\varphi^\phi$ and $\varphi^\psi$ be of the forms
$\varphi^\phi \formis \varphi^\phi_1 \lor \ldots \lor \varphi^\phi_m$ and 
$\varphi^\psi \formis \varphi^\psi_1 \lor \ldots \lor \varphi^\psi_n$ for $m,n \geq 1$. 
Assume that it is not the case that $\{\varphi^\phi\} \Vdash_\Pi \varphi^\psi$. Therefore, 
using the previous analysis and the properties of DDNFs, we can show that there must exist 
$\varphi^\phi_i \in \{\varphi^\phi_1 , \ldots ,\varphi^\phi_m\}$ s.t. 
$\varphi^\phi_i \notin \{\varphi^\psi_1, \ldots ,\varphi^\psi_n\}$ (i.e. there exists a probability distribution satisfying 
$\varphi^\phi$ and not satisfying $\varphi^\psi$, and therefore a formula associated with that probability 
present in $\varphi^\phi$ and not present in $\varphi^\psi$). 
Therefore, we can reuse the previous reasoning to show that it cannot be the case that
$\{\varphi^\phi_1 \land \neg \varphi^\psi_1 \land \ldots \land \neg \varphi^\psi_n) \lor \ldots \lor (\varphi^\phi_m \land \neg \varphi^\psi_1 \land \ldots \land \neg \varphi^\psi_n\} \Vdash_\Pi \bot$
and therefore it cannot be the case that $\{\varphi^\phi \land \top\} \Vdash_\Pi \bot$. We reach a contradiction with our 
assumptions. Hence, we conclude that if 
$\{\varphi^\phi \land \neg \varphi^\psi\} \Vdash_\Pi \bot$ then $\{\varphi^\phi\} \Vdash_\Pi \varphi^\psi$. 
Given the previous results, we have therefore shown that 
$\{\varphi^\phi \land \neg \varphi^\psi\} \Vdash_\Pi \bot$ iff $\{\varphi^\phi\} \Vdash_\Pi \varphi^\psi$. 
\end{proof}

\lemmavformula*

\begin{proof}  
For every $\prob \in \dist(\graph, \Pi)$ we create 
the epistemic formula $\varphi^\prob$ associated with it. For every such formula, we create 
a proposition $d^\prob$. To every formula $\psi \in \eformulae(\graph, \Pi)$, we assign its DDNF, i.e. 
if $\sat(\psi, \Pi) = \emptyset$, then $f(\psi) = \bot$, and if $\sat(\psi, \Pi) = \{\prob_1, \ldots, \prob_m\}$ for $m \geq 1$, 
then
$f(\psi) = d^{\prob_1} \lor d^{\prob_2} \lor \ldots \lor d^{\prob_m}$. 
We can observe that for two different $\prob, \prob' \in \dist(\graph, \Pi)$, $\varphi^\prob \land \varphi^{\prob'}$ is 
equivalent to $\bot$. Thus, we add rule $d^{\prob_i} \land d^{\prob_j} \leftrightarrow \bot$ for $i \neq j$ to $\Lambda$. 
Assuming that $\dist(\graph, \Pi) = \{\prob_1, \ldots, \prob_k\}$, 
we also add rule $d^{\prob_1} \lor d^{\prob_2} \lor \ldots \lor d^{\prob_k} \leftrightarrow \top$ to $\Lambda$, 
following the result in Proposition \ref{subvaldistributiondnf}.

Let us consider the left to right direction. 
Assume that for a formula $\phi_i$, its DDNF $\varphi^{\phi_i}$ is $\varphi^{\prob}_{1,i} \lor \ldots \lor \varphi^{\prob}_{m_i,i}$ 
(or $\bot$ if $m_i = 0$)
and that for a formula $\psi$, its DDNF $\varphi^{\psi}$ is $\varphi^{\prob}_{1} \lor \ldots \lor \varphi^{\prob}_{l}$ 
(or $\bot$ if $l = 0$).
We observe that under the restricted epistemic consequence relation, 
$\{\phi_1,\ldots,\phi_n\} \Vdash_\Pi \psi$ iff 
$\{\varphi^{\phi_1},\ldots,\varphi^{\phi_n}\} \Vdash_\Pi \varphi^{\psi}$. 
Hence, $\{\phi_1,\ldots,\phi_n\} \Vdash_\Pi \psi$ iff 
$\{\varphi^{\prob}_{1,1} \lor \ldots \lor \varphi^{\prob}_{m_1,1}, \ldots, \varphi^{\prob}_{1,n} \lor \ldots \lor \varphi^{\prob}_{m_n,n} \Vdash_\Pi \varphi^{\prob}_{1} \lor \ldots \lor \varphi^{\prob}_{l}$. 
Due to the nature of DDNFs and $\Vdash_\Pi$, we observe that the long formula holds 
iff every distribution formula (and conjunctive clause at the same time) $\varphi^{\prob}$ that is present in all $\varphi^{\phi_i}$, 
is present in $\varphi^{\psi}$ or, if $n=0$ (i.e. our set of assumptions is empty), $\varphi^{\psi}$ is the DDNF of $\top$. 
If it is the first case, then 
$\{f(\phi_1), \ldots, f(\phi_n)\} \cup \Lambda \vdash f(\psi)$ easily from properties of propositional logic 
or the $\bot$ rule we have added to $\Lambda$. If it is the other, then it holds easily through the second rule we have 
added to $\Lambda$.  

The right to left direction can be proved in a similar fashion based on (proof of) Proposition \ref{valdistributiondnf} 
and the presence of propositional rules in the proof system.  
\end{proof}

\restrictedequivalencetwo*

\begin{proof} 

We assume that $\Pi = \{0,1\}$.
First, we provide the following translation, and then show the equivalence.  
For each propositional formula $\alpha$, let $g(\alpha)$ be defined as follows:
\begin{itemize}
\item if $\alpha$ is an atom, then $g(\alpha)$ is $\pprob(\alpha)=1$
\item if $\alpha$ is negation of the form $\neg \alpha'$, 
then $g(\neg \alpha') = \neg g(\alpha')$.
\item if $\alpha$ is conjunction of the form $\alpha_1\land\alpha_2$, 
then $g(\alpha_1\land\alpha_2) = g(\alpha_1) \land g(\alpha_2)$.
\item if $\alpha$ is disjunction of the form $\alpha_1\lor\alpha_2$, 
then $g(\alpha_1\lor\alpha_2) = g(\alpha_1) \lor g(\alpha_2)$.
\end{itemize}
We can observe that there is a correspondence between the propositional proof rules P1 and P2 and the proof rules for classical propositional logic under this translation.
Hence, it is straightforward to show by induction on the structure of the propositional formulae in 
$\alpha_1,\ldots,\alpha_n$ and $\beta$ 
that $\{\alpha_1,\ldots,\alpha_n\} \vdash \beta$ iff $\{g(\alpha_1),\ldots,g(\alpha_n)\} \Vdash_\Pi g(\beta)$.  
\end{proof}

\propclassequivalenttwo*

\begin{proof}
Follows easily from Lemmas \ref{lemma:vformula} and \ref{lemma:2}.
\end{proof}

\conscoverage*

\begin{proof}
The properties hold straightforwardly from the definitions of default, full and partial coverage. For counterexamples 
showing that the relations hold only one way, please consult Examples \ref{ex:coverage1}, \ref{ex:coveragecont} 
and \ref{ex:counterexfullpart}.
\end{proof}
 
\conscoverageequiv*

\begin{proof}
We can observe that if $\sat(\con) = \sat(\con')$, 
then $\nodes(\graph) = \nodes(\graph')$. 
Furthermore, from this it also holds that $\con \VDash \varphi$ iff $\con'  \VDash \varphi$ for a given 
formula $\varphi$. 
Hence, if $\con \VDash \pprob(\xa) \neq x$ for an argument $\xa \in \nodes(\graph)$ and value $x \in [0,1]$, 
then $\con' \VDash \pprob(\xa) \neq x$ and vice versa. Therefore, an argument $\xa$ is default covered in $X$ 
iff it is default covered in $X'$. Remaining forms of coverage can be shown in a similar fashion. 

\end{proof}

\effectiveness*
\begin{proof}
\begin{itemize}
\item Follows straightforwardly from the definition. To see that the relation is one way, consult Example \ref{ex:effectivenesscounterex2}.
\item Follows straightforwardly form the definition. To see that the relation is one way, consult Example 
\ref{ex:semieffectiveness}. 
\item Follows straightforwardly from the definition.
\item Follows straightforwardly form the definition.
\item Consult Example \ref{ex:semieffectiveness}. 
\item Consider the set of constraints $\con = \{\pprob(\xa) > 0.5, \pprob(\xa) > 0.5 \rightarrow \pprob(\xb) < 0.5\}$
and assume that the $(\xa, \xb)$ relation is present in the graph. It is easy to see 
that it is neither effective nor strongly effective w.r.t. $F = \{\xa\}$. 
However, given $Z = \{\pprob(\xa) > 0.5 \rightarrow \pprob(\xb) < 0.5\}$, 
we can show that $(\xa, \xb)$ is strongly semi--effective w.r.t. $(Z, F)$.
\end{itemize}

\end{proof}

\end{document}